\documentclass[11pt]{article}
\usepackage{setspace}
\usepackage{cprotect}
\usepackage{amsmath,amssymb}
\usepackage{amsthm}
\usepackage[noend]{algorithmic}
\usepackage[ruled,vlined]{algorithm2e}
\usepackage{url}
\usepackage{fullpage}
\usepackage{makeidx}
\usepackage{enumerate}
\usepackage[top=1in, bottom=1.25in, left=1in, right=1in]{geometry}
\usepackage{graphicx,float,psfrag,epsfig,caption,subcaption}
\usepackage{epstopdf}
\usepackage{color}
\usepackage{lmodern}

\usepackage[mathscr]{euscript}

\DeclareSymbolFont{rsfs}{U}{rsfs}{m}{n}
\DeclareSymbolFontAlphabet{\mathscrsfs}{rsfs}

\numberwithin{equation}{section}

\newtheoremstyle{myexample} 
    {\topsep}                    
    {\topsep}                    
    {\rm }                   
    {}                           
    {\bf }                   
    {.}                          
    {.5em}                       
    {}  

\newtheoremstyle{myremark} 
    {\topsep}                    
    {\topsep}                    
    {\rm}                        
    {}                           
    {\bf}                        
    {.}                          
    {.5em}                       
    {}  

\newtheorem{claim}{Claim}[section]
\newtheorem{lemma}[claim]{Lemma}

\newtheorem{conjecture}[claim]{Conjecture}
\newtheorem{theorem}{Theorem}
\newtheorem{proposition}[claim]{Proposition}
\newtheorem{corollary}[claim]{Corollary}

\theoremstyle{myremark}
\newtheorem{remark}{Remark}[section]

\theoremstyle{myremark}

\theoremstyle{myexample}

\def\tgamma{\tilde{\gamma}}
\def\Corr{{\sf C}}
\def\Bind{{\sf B}}
\def\Norm{{\sf V}}
\def\Sym{{\sf S}}
\def\Info{{\rm I}}
\def\sT{{\sf T}}
\def\<{\langle}
\def\>{\rangle}

\def\prob{{\mathbb P}}
\def\hprob{\widehat{\mathbb P}}
\def\integers{{\mathbb Z}}

\def\E{{\mathbb E}} 
\def\hE{\widehat{\mathbb E}} 

\def\bq{\boldsymbol{q}}

\def\de{{\rm d}}
\def\ty{{\tilde{y}}}
\def\NMF{{\mathcal M}}

\def\bfeta{{\boldsymbol \eta}}
\def\tbfeta{\tilde{\boldsymbol \eta}}
\def\teta{\tilde{\eta}}

\def\lbar{{\overline{\lambda}}}

\def\tr{\tilde{r}}
\def\tm{\tilde{m}}
\def\supp{{\rm supp}}
\def\bPu{{\boldsymbol P}_{{\boldsymbol u}}}
\def\bPup{{\boldsymbol P}^{\perp}_{{\boldsymbol u}}}

\newcommand\T{\rule{0pt}{2.6ex}}

\def\diag{{\mbox{\rm diag}}}
\def\op{\mbox{\tiny\rm op}}
\def\inst{\mbox{\tiny\rm inst}}
\def\sp{\mbox{\tiny\rm spect}}

\def\sBayes{\mbox{\tiny\rm Bayes}}
\def\sTV{\mbox{\tiny\rm TV}}

\def\sTAP{\mbox{\tiny\rm TAP}}
\def\sBethe{\mbox{\tiny\rm Bethe}}
\def\reals{\mathbb{R}}

\def\cP{{\mathcal{P}}}

\def\cO{{\mathcal{O}}}
\def\cD{{\mathcal D}}
\def\cL{{\mathcal L}}

\def\bbQ{{\mathbb Q}}
\def\bbS{{\mathbb S}}
\def\Hess{{\boldsymbol \Omega}}
\def\oHess{\bar{\boldsymbol \Omega}}
\def\bg{{\boldsymbol g}}
\def\tbg{\tilde{\boldsymbol g}}
\def\bsigma{{\boldsymbol \sigma}}
\def\hbsigma{\hat{\boldsymbol \sigma}}
\def\hm{\hat{m}}

\def\bC{{\boldsymbol C}}
\def\hq{\hat{q}}
\def\cF{{\cal F}}
\def\tcF{\tilde{\cal F}}
\def\cG{{\cal G}}
\def\ELBO{{\rm ELBO}}
\def\tbr{\tilde{\boldsymbol r}}
\def\tbh{\tilde{\boldsymbol h}}
\def\br{{\boldsymbol r}}

\def\bdelta{{\boldsymbol \delta}}
\def\tbdelta{\tilde{\boldsymbol \delta}}
\def\tdelta{\tilde{\delta}}
\def\bDelta{{\boldsymbol \Delta}}
\def\tbDelta{\widetilde{\boldsymbol \Delta}}
\def\bOmega{{\boldsymbol \Omega}}
\def\tbOmega{\tilde{\boldsymbol \Omega}}

\def\bM{{\boldsymbol M}}

\def\by{{\boldsymbol y}}

\def\tbq{\tilde{\boldsymbol q}}
\def\tfb{\tilde{\boldsymbol f}}
\def\fb{{\boldsymbol f}}

\def\tby{\tilde{\boldsymbol y}}
\def\hby{\tilde{\boldsymbol y}}
\def\tbm{\tilde{\boldsymbol m}}
\def\tphi{{\tilde{\phi}}}
\def\tpsi{{\tilde{\psi}}}

\def\tbq{\widetilde{\boldsymbol q}}
\def\tbQ{\widetilde{\boldsymbol Q}}

\def\KL{{\rm KL}}

\def\cuQ{\mathscrsfs{Q}}
\def\RS{{\sf RS}}

\def\tbM{\widetilde{\boldsymbol M}}

\def\bm{{\boldsymbol m}}
\def\bQ{{\boldsymbol Q}}
\def\bS{{\boldsymbol S}}
\def\bT{{\boldsymbol T}}
\def\tbT{\tilde{\boldsymbol T}}

\def\hR{\widehat{R}}

\def\allone{{\mathbf J}}

\def\Tr{{\sf {Tr}}}

\def\sE{{\sf E}}

\def\bG{{\boldsymbol G}}
\def\bH{{\boldsymbol H}}
\def\hbh{\widehat{\boldsymbol h}}
\def\hbw{\widehat{\boldsymbol w}}
\def\hbQ{\widehat{\boldsymbol Q}}
\def\hbF{\widehat{\boldsymbol F}}
\def\hbH{\widehat{\boldsymbol H}}
\def\hbW{\widehat{\boldsymbol W}}

\def\bJ{{\boldsymbol J}}
\def\bPi{{\boldsymbol \Pi}}
\def\be{{\boldsymbol e}}

\def\obx{\overline{\boldsymbol x}}
\def\oby{\overline{\boldsymbol y}}
\def\bx{{\boldsymbol x}}
\def\bz{{\boldsymbol z}}
\def\ba{{\boldsymbol a}}
\def\bzero{{\boldsymbol 0}}

\def\bL{{\boldsymbol L}}
\def\bD{{\boldsymbol D}}

\def\bT{{\boldsymbol T}}
\def\bX{{\boldsymbol X}}

\def\obX{\bar{\boldsymbol X}}

\def\bX{\boldsymbol X}

\def\bZ{{\boldsymbol Z}}
\def\tbZ{\widetilde{\boldsymbol Z}}
\def\bP{{\boldsymbol P}}
\def\bPp{{\boldsymbol P}_{\perp}}
\def\bv{{\boldsymbol v}}
\def\tbv{\tilde{\boldsymbol v}}
\def\bu{{\boldsymbol u}}
\def\bw{{\boldsymbol w}}
\def\hbW{\widehat{\boldsymbol W}}
\def\hbH{{\widehat{\boldsymbol H}}}
\def\bW{{\boldsymbol W}}

\def\bK{{\boldsymbol K}}

\def\hbx{\hat{\boldsymbol x}}

\def\bOmega{{\boldsymbol \Omega}}

\def\bb{{\boldsymbol b}}

\def\bR{{\boldsymbol R}}
\def\bA{{\boldsymbol A}}

\def\bC{{\boldsymbol C}}

\def\bh{{\boldsymbol h}}
\def\hbm{\hat{\boldsymbol m}}

\def\Dir{{\rm Dir}}

\def\normal{{\sf N}}

\def\eps{{\varepsilon}}

\def\id{{\boldsymbol I}}

\def\sP{{\sf P}}
\def\sF{{\sf F}}

\def\sG{{\sf G}}
\def\tsF{\widetilde{\sf F}}
\def\tsG{\widetilde{\sf G}}

\def\tq{\tilde{q}}

\def\tZ{\tilde{Z}}

\def\oq{\overline{q}}
\def\bfone{{\boldsymbol 1}}

\def\sB{{\sf B}}
\def\sC{{\sf C}}
\def\mmse{{\sf mmse}}
\def\entro{{\sf h}}
\def\GOE{{\sf GOE}}

\title{An Instability in Variational Inference for Topic Models}
\author{Behrooz Ghorbani,\;\; Hamid Javadi\thanks{Department of Electrical Engineering, Stanford University}, \;\;
Andrea Montanari\thanks{Department of Electrical Engineering and Department of Statistics, Stanford University}}

\usepackage[colorinlistoftodos]{todonotes}
\usepackage{float}
\begin{document}

\maketitle

\begin{abstract}
Topic models are Bayesian models that are frequently used to capture the latent structure of certain 
corpora of documents or images. Each data element in such a corpus (for instance each item in a collection of scientific articles)
is regarded as a convex combination of a small number of vectors corresponding to `topics' or `components'.
The weights are assumed to have a Dirichlet prior distribution.
The standard approach towards approximating the posterior is to use variational inference algorithms, and in particular a 
mean field approximation. 

We show that this approach suffers from an instability that can produce misleading conclusions. 
Namely, for certain regimes of the model parameters, variational inference outputs a non-trivial decomposition 
into topics. However --for the same parameter values--  the data  contain no actual information about the 
true decomposition, and hence the output of the algorithm is
uncorrelated with the true topic decomposition. Among other
consequences, the estimated posterior mean is significantly wrong, and
estimated Bayesian credible regions do not achieve the nominal coverage.
We discuss how this instability is remedied by more accurate mean field approximations. 
\end{abstract}

\section{Introduction}
\label{sec:Intro}

Topic modeling \cite{blei2012probabilistic} aims at extracting the latent structure from a corpus of  documents (either images or texts),
that are represented as vectors $\bx_1,\bx_2,\dots,\bx_n\in\reals^d$. The key assumption is that the $n$ documents are (approximately) convex combinations of a small number $k$ of topics
$\tbh_1,\dots,\tbh_k\in\reals^d$. Conditional on the topics, documents are generated 
independently by letting
\begin{align}
\bx_a =  \frac{\sqrt{\beta}}{d}\sum_{\ell=1}^kw_{a,\ell}\tbh_{\ell}+\bz_a\, ,
\end{align}
where the weights $\bw_a= (w_{a,\ell})_{1\le \ell\le k}$ and noise vectors $\bz_a$ are i.i.d. across $a\in\{1,\dots,n\}$. The scaling factor 
$\sqrt{\beta}/d$ is introduced for mathematical convenience (an equivalent parametrization would have been to scale $\bZ$ by a noise-level parameter $\sigma$), and $\beta>0$ can be interpreted as a 
signal-to-noise ratio.
It is also useful to introduce the matrix $\bX\in \reals^{n\times d}$ whose $i$-th row is $\bx_i$, and therefore
\begin{align}
\bX = \frac{\sqrt{\beta}}{d}\,\bW\bH^{\sT} +\bZ\, ,\label{eq:LDAModel}
\end{align}
where $\bW\in\reals^{n\times k}$ and $\bH\in\reals^{d\times k}$. The $a$-th row of $\bW$, is the vector of weights
$\bw_a$, while the rows of $\bH$ will be denoted by $\bh_i\in\reals^k$. 

Note that $\bw_a$ belongs to the simplex $\sP_1(k) = \{\bw\in\reals^k_{\ge 0}\; :\;\;\<\bw,\bfone_k\> =1\}$. 
It is common to assume that its prior is Dirichlet: this class of models is known as \emph{Latent Dirichlet Allocations}, or LDA \cite{blei2003latent}. Here we will take a particularly simple example of this type,
and assume that the prior  is Dirichlet in $k$ dimensions with all parameters equal to $\nu$ (which we will denote by $\Dir(\nu;k)$).
As for the topics $\bH$, their prior distribution depends on the specific application. For instance, when applied to text corpora, the $\tbh_i$ are typically non-negative and represent normalized 
word count vectors. Here we will
make the simplifying assumption that they are standard Gaussian $(\tbh_{i})_{i\le d}\sim_{iid}\normal(0,\id_k)$.
 Finally, $\bZ$ will be a noise matrix with entries $(Z_{ij})_{i\in
   [n], j\in [d]}\sim_{iid}\normal(0,1/d)$.

In  fully Bayesian topic models, the parameters of the Dirichlet distribution, as well as the topic distributions are themselves unknown 
and to be learned from data. Here we will work in an idealized setting in which they are known. We will also assume that
data are in fact distributed according to the postulated generative model. Since we are interested
in studying some limitations of current approaches, our main point is only reinforced by assuming this idealized scenario.

As is common with Bayesian approaches, computing the posterior distribution of the factors $\bH$, $\bW$ given the data $\bX$ is computationally challenging. 
Since the seminal work of Blei, Ng and Jordan \cite{blei2003latent},
variational inference is the method of choice for addressing this problem within topic models. The term `variational inference' refers to a broad class of methods that 
aim at approximating the posterior computation by solving an optimization problem, see \cite{jordan1999introduction,wainwright2008graphical,blei2017variational} for
background. A popular starting point is the Gibbs variational principle, namely the fact that the posterior solves the following convex optimization problem:
\begin{align}
p_{\bW,\bH|\bX}(\,\cdot\,,\cdot,|\bX) & = \arg\min_{q\in\cP_{n,d,k}} \KL(q\| p_{\bW,\bH|\bX})  \label{eq:Gibbs}\\
&= \arg\min_{q\in\cP_{n,d,k}}\Big\{-\E_{q}\log p_{\bX|\bW,\bH}(\bX|\bH,\bW) +\KL(q\| p_{\bW}\times p_{\bH})\Big\}\,, \label{eq:Gibbs2}
\end{align}
where $\KL(\, \cdot\,\|\,\cdot\, )$ denotes the Kullback-Leibler divergence. The variational expression in Eq.~(\ref{eq:Gibbs2})
is also known as the Gibbs free energy.
Optimization is within the space $\cP_{n,d,k}$ of probability measures
on $\bH,\bW$. To be precise, we always assume that a dominating measure $\nu_0$ over $\reals^{n\times k}\times \reals^{d\times k}$ is given for $\bW,\bH$, and
both $p_{\bW,\bH|\bX}$ and $q$ have densities with respect to $\nu_0$: we hence identify the measure with its density. 
Throughout the paper (with the exception of the example in Section \ref{sec:Toy}) $\nu_0$ can be taken to be the  Lebesgue measure. 

Even for $\bW,\bH$ discrete, the Gibbs principle  has exponentially many decision variables.
Variational methods differ in the way the problem  (\ref{eq:Gibbs}) is approximated. The main approach within topic modeling is \emph{naive mean field}, 
which restricts the optimization problem to the space of probability measures that factorize over the rows of $\bW,\bH$:
\begin{align}
\hq\left(\bW,\bH\right) = q\left(\bH\right)\tilde q\left(\bW\right) = \prod_{i=1}^d q_i\left(\bh_i\right)\prod_{a=1}^n\tq_a\left(\bw_a\right)\, . \label{eq:ProductForm}
\end{align}
By a suitable parametrization of the marginals $q_i$, $\tq_a$, this leads to an optimization problem of dimension $O((n+d)k)$, cf. Section \ref{sec:Inst}. Despite being non-convex,
this problem is separately convex in the $(q_i)_{i\le d}$ and $(\tq_a)_{a\le n}$, which naturally suggests the use of
an alternating minimization algorithm which has been successfully deployed in a broad range of applications ranging from computer vision to genetics
\cite{fei2005bayesian,wang2011collaborative,raj2014faststructure}. We will refer to this as to the \emph{naive mean field iteration}.
Following a common use in the topics models literature,  we will use the terms `variational inference' and `naive mean field' interchangeably.

\vspace{0.25cm}

The main result of this paper is that naive mean field presents an instability for learning Latent Dirichlet Allocations.
We will focus on the limit $n,d\to\infty$ with $n/d=\delta$ fixed. Hence, an LDA distribution is determined by the
parameters $(k,\delta,\nu,\beta)$. We will show that there are regions in this parameter space such that the following two findings hold simultaneously:
\begin{description}
\item[No non-trivial estimator.] Any estimator $\hbH$, $\hbW$ of the topic or weight matrices is asymptotically uncorrelated with the real model parameters $\bH, \bW$. 
In other words, the data do not contain enough signal to perform any strong inference.
\item[Variational inference is randomly biased.] Given the above, one would hope the Bayesian posterior to be centered on an unbiased estimate. In particular, 
$p(\bw_a|\bX)$ (the posterior distribution over weights of document $a$) should be centered around the uniform distribution $\bw_a= (1/k,\dots,1/k)$. 
In contrast, we will show that the posterior produced by naive mean field is centered around a random distribution that is uncorrelated with the actual weights. 
Similarly, the posterior over topic vectors is centered around random vectors uncorrelated with the true topics.
\end{description}
One key argument in support of Bayesian methods is the hope that they  provide a measure of uncertainty of the estimated variables. In 
view of this, the failure just described is particularly dangerous because it suggests some measure of certainty, although the estimates are essentially random.

Is there a way to eliminate this instability by using a better mean field approximation? We show that a promising approach 
is provided by a classical idea in statistical physics, the Thouless-Anderson-Palmer (TAP) free energy \cite{thouless1977solution,opper2001adaptive}. This suggests a variational principle that is analogous
in form to naive mean field, but provides a more accurate approximation of the Gibbs principle:
\begin{description}
\item[Variational inference via the TAP free energy.] We show that the instability of naive mean field is  remedied by using the TAP free energy instead
of the naive mean field free energy. The latter can be optimized using an iterative scheme that is analogous to the naive mean field iteration and is known as 
approximate message passing (AMP).  
\end{description}
While the TAP approach is promising --at least for synthetic data-- we believe that further work is needed to develop a reliable inference scheme. 

The rest of the paper is organized as follows. Section  \ref{sec:Toy} discusses a simpler example, $\integers_2$-synchronization, which shares important features with latent Dirichlet allocations. 
Since calculations are fairly straightforward, this example allows to explain the main mathematical points in a simple context. 
We then present our main results about instability of naive mean field in Section \ref{sec:Inst}, and discuss the use of TAP free energy to overcome the instability in 
Section \ref{sec:Fixing}.

\subsection{Related literature}

Over the last fifteen years, topic models have been generalized to cover an impressive number of applications.
A short list includes mixed membership models \cite{erosheva2004mixed,airoldi2008mixed}, dynamic topic models \cite{blei2006dynamic}, correlated topic models
\cite{lafferty2006correlated,blei2007correlated}, spatial LDA \cite{wang2008spatial}, relational topic models \cite{chang2009relational}, Bayesian tensor models \cite{zhou2015bayesian}.
While other approaches have been used (e.g. Gibbs sampling),
variational algorithms are among the most popular methods for Bayesian inference
in these models. Variational methods provide a fairly complete and interpretable description of the posterior, while allowing to leverage advances in optimization algorithms 
and architectures towards this goal (see \cite{hoffman2010online,broderick2013streaming}).

Despite this broad empirical success, little is rigorously known about the accuracy of variational inference in 
concrete statistical problems. Wang and Titterington \cite{wang2004convergence,wang2006convergence} prove local convergence of 
naive mean field estimate to the true parameters for exponential families with missing data and Gaussian mixture models. In the context of
Gaussian mixtures, the same authors prove that the covariance of the variational posterior is asymptotically smaller (in the positive semidefinite
order) than the inverse of the Fisher information matrix \cite{wang2005inadequacy}. All of these results are established in the classical large sample asymptotics
$n\to\infty$ with $d$ fixed. In the present paper we focus instead on the high-dimensional limit $n = \Theta(d)$ and prove that 
also the mode  (or mean) of the variational posterior is incorrect.
Notice that the high-dimensional regime is particularly relevant for the analysis of Bayesian methods. Indeed, in  the classical
low-dimensional asymptotics Bayesian approaches do not outperform maximum likelihood.

 In order to correct for the underestimation of covariances, \cite{wang2005inadequacy}
suggest replacing its variational estimate by the inverse Fisher information matrix. A different approach is developed in 
\cite{giordano2015linear}, building on linear response theory. 

Naive mean field variational inference was used in \cite{celisse2012consistency,bickel2013asymptotic}
to estimate the parameters of the stochastic block model. These works establish consistency and
asymptotic normality of the variational estimates in a large signal-to-noise ratio regime. 
Our work focuses on estimating the latent factors: it would be interesting to consider implications on parameter 
estimation as well.

The recent paper \cite{zhang2017theoretical} also studies variational inference in the context of the stochastic block model,
but focuses on reconstructing the latent vertex labels. The authors prove that naive mean field achieves minimax optimal statistical rates. 
Let us emphasize that this problem is closely related  to topic models:
both are  models for approximately low-rank matrices, with a probabilistic prior on the factors. The results of \cite{zhang2017theoretical} are complementary to ours, 
in the sense that  \cite{zhang2017theoretical} establishes positive
results at large signal-to-noise ratio (albeit for a different model), 
while we prove inconsistency at low signal-to-noise ratio. 
General conditions for consistency of variational Bayes methods are proposed in \cite{pati2017statistical}.

Our work also builds on recent theoretical advances in high-dimensional low-rank models, that were mainly driven 
by techniques from mathematical statistical physics (more specifically, spin glass theory). An incomplete list of relevant references
includes \cite{korada2009exact,deshpande2014information,deshpande2017asymptotic,krzakala2016mutual,barbier2016mutual,lelarge2016fundamental,miolane2017fundamental,lesieur2017constrained,alaoui2018estimation}.
These papers prove asymptotically exact characterizations of the Bayes optimal estimation error in low-rank models, to an increasing degree of
generality, under the high-dimensional scaling $n,d\to\infty$ with $n/d\to\delta\in (0,\infty)$. 

Related ideas also suggest an iterative algorithm for Bayesian estimation, namely Bayes Approximate Message Passing \cite{DMM09,DMM_ITW_I}. As mentioned
above, Bayes AMP can be regarded as minimizing a different variational approximation known as the TAP free energy. An important
advantage over naive mean field is that AMP can be rigorously analyzed using a method known  as state evolution 
\cite{BM-MPCS-2011,javanmard2013state,berthier2017state}.

Let us finally mention that a parallel line of work develops polynomial-time algorithms to construct
non-negative matrix factorizations under certain structural assumptions on the data matrix $\bX$, such as separability
\cite{arora2012learning,arora2012computing,recht2012factoring}. It should be emphasized that the objective of these
algorithms is different from the one of Bayesian methods: they return a factorization that is guaranteed to be unique under separability.
In contrast, variational methods attempt to approximate the posterior distribution, when the data are generated according to the LDA model.

\subsection{Notations}

We denote by $\id_m$ the identity matrix, and by $\bJ_m$ the all-ones
matrix in $m$ dimensions  (subscripts will be dropped when the number of dimensions is clear from the 
context). We use $\bfone_k\in\reals^k$ for the all-ones vector. 

We will use $\otimes$ for the tensor (outer) product. In particular, given vectors expressed in
the canonical basis as $\bu = \sum_{i=1}^{d_1} u_i\be_i\in\reals^{d_1}$ and  $\bv = \sum_{i=j}^{d_2} v_j\be_j\in\reals^{d_2}$,
$\bu\otimes\bv\in\reals^{d_1}\otimes \reals^{d_2}$ is the tensor with coordinates $(\bu\otimes\bv)_{ij} = u_iv_j$ in the basis $\be_i\otimes \be_j$.
We will identify the space of matrices $\reals^{d_1\times d_2}$ with the tensor product $\reals^{d_1}\otimes \reals^{d_2}$.
In particular, for $\bu\in\reals^{d_1}$, $\bv\in\reals^{d_2}$, we identify $\bu\otimes \bv$ with the matrix $\bu\bv^{\sT}$.

Given a symmetric matrix $\bM\in\reals^{n\times n}$, we denote by $\lambda_1(\bM)\ge \lambda_2(\bM)\ge \dots\ge \lambda_n(\bM)$
its eigenvalues in decreasing order. 
For a matrix (or vector) $\bA \in \reals^{d\times n}$ we denote the orthogonal
projector operator onto the subspace spanned by the columns of $\bA$ by
$\bP_\bA\in\reals^{d\times d}$, and its orthogonal complement by $\bP_\bA^{\perp} = \id_d-\bP_{\bA}$. When the subscript is omitted, this is understood to be the
projector onto the space spanned by the all-ones vector: $\bP=\bfone_d\bfone_d/d$ and  $\bP_{\perp}=\id_d-\bP$.

\section{A simple example: $\integers_2$-synchronization}
\label{sec:Toy}

In $\integers_2$ synchronization we are interested in estimating a vector $\bsigma\in\{+1,-1\}^n$ from observations $\bX\in\reals^{n\times n}$,
generated according to 
\begin{align}
\bX = \frac{\lambda}{n}\bsigma\bsigma^{\sT}+\bZ\, ,\label{eq:Z2-synch}
\end{align}
where $\bZ=\bZ^{\sT}\in\reals^{n\times n}$ is distributed according to the Gaussian Orthogonal Ensemble $\GOE(n)$, namely $(Z_{ij})_{i<j\le n}\sim_{iid}\normal(0,1/n)$
are independent of $(Z_{ii})_{i\le n}\sim_{iid}\normal(0,2/n)$. The parameter $\lambda\ge 0$ corresponds to the signal-to-noise ratio.

It is known that for $\lambda\le 1$ no algorithm can estimate $\bsigma$ from data $\bX$ with positive correlation in the limit $n\to\infty$.
The following is an immediate consequence of \cite{korada2009exact,deshpande2017asymptotic}, see Appendix \ref{app:LemmaITZ2}.
\begin{lemma}\label{lemma:IT-Threshold-Z2}
Under model (\ref{eq:Z2-synch}), for $\lambda\le 1$ and any estimator $\hbsigma:\reals^{n\times n}\to\reals^n\setminus \{\bzero\}$, 
the following limit holds in probability:
\begin{align}
\lim\sup_{n\to\infty}\frac{|\<\hbsigma(\bX),\bsigma\>|}{\|\hbsigma(\bX)\|_2\|\bsigma\|_2} = 0\, .
\end{align}
\end{lemma}

How does variational inference perform on this problem? Any product probability distribution $\hq(\bsigma) = \prod_{i=1}^n q_i(\sigma_i)$
can be parametrized by the means $m_i= \sum_{\sigma_i\in\{+1,-1\}} q_i(\sigma_i)\,\sigma_i$, and it is immediate to get
\begin{align}
\KL(\hq\|p_{\bsigma|\bX}) &= \cF(\bm) + {\rm const.}\, ,\\
\cF(\bm) & \equiv -\frac{\lambda}{2}\<\bm,\bX_0\bm\> -\sum_{i=1}^n\entro(m_i)\, .\label{eq:Z2_FreeEnergy}
\end{align}
Here $\bX_0$ is obtained from $\bX$ by setting the diagonal entries to $0$, and $\entro(x) = -\frac{(1+x)}{2}\log\frac{(1+x)}{2} -\frac{(1-x)}{2}\log\frac{(1-x)}{2}$ is the binary entropy function. In view of Lemma \ref{lemma:IT-Threshold-Z2},
the correct posterior distribution should be essentially uniform, resulting in $\bm$ vanishing. Indeed, $\bm_* = 0$ is a stationary point of the 
mean field free energy $\cF(\bm)$: $\left.\nabla\cF(\bm) \right|_{\bm = \bm_*}=0$.  
We refer to this as the `uninformative fixed point'.

\emph{Is $\bm_*$ a local minimum?} Computing the Hessian at the uninformative fixed point yields
\begin{align}
\left.\nabla^2\cF(\bm) \right|_{\bm = \bm_*} =  -\lambda\bX_0 +\id\, .
\end{align}
The matrix $\bX_0$ is a rank-one deformation of a Wigner matrix and its spectrum is well understood \cite{baik2005phase,feral2007largest,benaych2011eigenvalues}. For $\lambda\le 1$, its eigenvalues are contained with high probability in
the interval $[-2,2]$, with $\lambda_{\min}(\bX)\to -2$, $\lambda_{\max}(\bX)\to 2$ as $n\to\infty$. For $\lambda>1$, $\lambda_{\max}(\bX)\to \lambda+\lambda^{-1}$,
while the other eigenvalues are contained in $[-2,2]$. This implies
\begin{align}
\lim_{n\to\infty}\lambda_{\rm min}(\left.\nabla^2\cF\right|_{\bm_*}) = \begin{cases}
1-2\lambda& \;\; \mbox{if $\lambda\le 1$,}\\
-\lambda^2 & \;\; \mbox{if $\lambda > 1$.}\\
\end{cases}
\end{align}
In other words, $\bm_* = 0$ is a local minimum for $\lambda<1/2$, but becomes a saddle point for $\lambda>1/2$. In particular, for $\lambda\in (1/2,1)$, variational
inference will produce an estimate $\hbm\neq 0$, although the posterior should be essentially uniform. In fact, it is possible to make this conclusion more quantitative.
\begin{proposition}\label{propo:Toy}
Let $\hbm \in [-1,1]^n$ be any local minimum of the mean field free energy $\cF(\bm)$, under the $\integers_2$-synchronization
model (\ref{eq:Z2-synch}).  Then there exists a numerical constant $c_0>0$ such that, with high probability, for $\lambda>1/2$,
\begin{align}
\frac{1}{n}\|\hbm\|_2^2 \ge c_0\, \min\big((2\lambda-1)^2,1\big)\, .
\end{align}
\end{proposition}
In other words, although no estimator is positively correlated with the true signal $\bsigma$, variational inference 
outputs biases $\hm_i$ that are non-zero (and indeed of  order one, for a positive fraction of them). 

The last statement immediately implies that naive mean field leads to
incorrect inferential statements for $\lambda\in (1/2,1)$. In order to formalize this point, given any estimators $\{\hq_i(\,\cdot\, )\}_{i\le n}$  of the posterior
marginals, we define the per-coordinate expected coverage as
\begin{align}
\cuQ(\hq) = \frac{1}{n}\sum_{i=1}^n \prob\big(\sigma_i=\arg\max_{\tau_i\in\{+1,-1\}} \hq_i(\tau_i)\big)  \, .
\end{align}
This is the expected fraction of coordinates that are estimated correctly by choosing $\bsigma$ according to the estimated posterior. Since the prior is assumed to be correct, it 
can be interpreted either as the expectation (with respect to the parameters) of the frequentist coverage, or as the expectation (with respect to the data) of the Bayesian coverage.
On the other hand, if the $\hq_i$ were accurate, Bayesian theory would suggest claiming the coverage
\begin{align}
\widehat{\cuQ}(\hq) \equiv \frac{1}{n}\sum_{i\le n}\max_{\tau_i}\hq_i(\tau_i)\, .
\end{align}
The following corollary is a direct consequence of Proposition  \ref{propo:Toy}, and formalizes the claim that naive mean field leads to incorrect
inferential statements. More precisely, it overestimates the coverage achieved.
\begin{corollary}
Let $\hbm \in [-1,1]^n$ be any local minimum of the mean field free energy $\cF(\bm)$, under the $\integers_2$-synchronization
model (\ref{eq:Z2-synch}), and consider the corresponding posterior marginal estimates $\hq_i(\sigma_i) = (1+\hm_i\sigma_i)/2$. Then,
there exists a numerical constant $c_0>0$ such that, with high probability, for $\lambda\in (1/2,1)$,
\begin{align}
\cuQ(\hq)\le \frac{1}{2}+o_n(1)\, ,\;\;\;\;\;\;
\widehat{\cuQ}(\hq) \ge \frac{1}{2}+c_0\, \min\big((2\lambda-1),1\big)\, .
\end{align}
\end{corollary}
While similar formal coverage statements can be obtained also for the more complex case of topic models, we will not
make them explicit, since they are relatively straightforward consequences of our analysis.

\section{Instability of variational inference for topic models}
\label{sec:Inst}

\subsection{Information-theoretic limit}
\label{sec:IT-main}

As in the case of $\integers_2$ synchronization discussed in Section \ref{sec:Toy}, we expect it to be impossible to estimate the factors $\bW,\bH$
with strictly positive correlation for small enough signal-to-noise ratio $\beta$ (or small enough sample size $\delta$).
The exact threshold was characterized recently in
\cite{miolane2017fundamental}  (but see also 
\cite{deshpande2014information,barbier2016mutual,lelarge2016fundamental,lesieur2017constrained}
for closely related results). The characterization in \cite{miolane2017fundamental} is given in terms of a variational principle over $k\times k$ matrices.
\begin{theorem}[Special case of \cite{miolane2017fundamental}]\label{thm:IT_Limit}
Let $\Info_n(\bX;\bW,\bH)$ denote the mutual information between the data $\bX$ and the factors $\bH,\bW$ under the 
LDA model (\ref{eq:LDAModel}). Then, the following limit holds almost surely
\begin{align}
\lim_{n,d\to\infty}\frac{1}{d}\Info_n(\bX;\bW,\bH) = \inf_{\bM\in\bbS_k} \RS(\bM;k,\delta,\nu)\, ,\label{eq:InfimumFreeEnergy}
\end{align}
where $\bbS_k$ is the cone of $k\times k$ positive semidefinite matrices and  $\RS(\,\cdots\,) $ is a function given explicitly in Appendix \ref{app:ProofBayes}.
\end{theorem}
It is also shown in Appendix  \ref{app:ProofBayes} that $\bM^* = (\delta\beta/k^2)\bJ_k$ is a stationary point of the free energy $\RS(\bM;k,\delta,\nu)$. We shall refer to $\bM^*$ as the uninformative point. 
Let $\beta_{\sBayes} = \beta_{\sBayes}(k,\delta,\nu)$ be the supremum value of $\beta$ such that the infimum in  Eq.~(\ref{eq:InfimumFreeEnergy}) 
is uniquely achieved at $\bM^*$:
\begin{align}
\beta_{\sBayes}(k,\delta,\nu) = \sup\Big\{\beta\ge 0:\;\;   \RS(\bM;k,\delta,\nu)>\RS(\bM_*;k,\delta,\nu) \mbox{\;\; for all\;\;\;} \bM\neq\bM_*\Big\}\, .
\end{align}
As formalized below, for  $\beta<\beta_{\sBayes}$ the data $\bX$ do not contain sufficient information for estimating $\bH$, $\bW$
in a non-trivial manner.
\begin{proposition}\label{propo:Bayes}
Let $\bM_* = \delta\beta\bJ_k/k^2$. Then $\bM^*$ is a stationary point of the function $\bM\mapsto  \RS(\bM;\beta,k,\delta,\nu)$. Further, it is a local minimum provided
$\beta<\beta_{\sp}(k,\delta,\nu)$ where the spectral threshold is given by
\begin{align}
\beta_{\sp} \equiv \frac{k(k\nu+1)}{\sqrt{\delta}}.
\end{align}
Finally, if $\beta<\beta_{\sBayes}(k,\delta,\nu)$, for any estimator $\bX\mapsto \hbF_n(\bX)$, we have 
\begin{align}
\lim\inf_{n\to \infty}\E\big\{\left\|\bW\bH^{\sT}-\hbF_n(\bX)\right\|_F^2\big\} \ge \lim_{n\to\infty}\E\left\{\left\|\bW\bH^{\sT}-
c\bfone_n(\bX^{\sT}\bfone_n)^{\sT}\right\|_F^2\right\} \,,\label{eq:TrivialEst}
\end{align}
for $c\equiv\sqrt{\beta}/(k+\beta\delta)$ a constant.
\end{proposition}
We refer to  Appendix  \ref{app:IT} for a proof of this statement.

Note that Eq.~(\ref{eq:TrivialEst}) compares the mean square error of an arbitrary estimator $\hbF_n$, 
to the mean square error of the trivial estimator that replaces each column of $\bX$ by its average. This is equivalent to estimating
all the weights $\bw_i$ by the uniform distribution $\bfone_k/k$.
Of course, $\beta_{\sBayes}\le \beta_{\sp}$. However, this upper bound appears to be tight for small $k$.
\begin{remark}
Solving numerically the $k(k+1)/2$-dimensional problem (\ref{eq:InfimumFreeEnergy}) indicates that 
$\beta_{\sBayes}(k,\nu,\delta) = \beta_{\sp}(k,\nu,\delta)$ for $k\in\{2,3\}$ and
$\nu=1$. 
\end{remark}

\subsection{Naive mean field free energy}

We consider a trial joint distribution that factorizes according to rows of $\bW$ and $\bH$ according to Eq.~(\ref{eq:ProductForm}).
It turns out (see Appendix \ref{app:NMF_ansatz}) that, for any stationary point of $\KL(\hq\|p_{\bH,\bW|\bX})$ over such product distributions, the marginals take the form
\begin{align}
\label{eq:densityforms_main}
\begin{split}
&q_i(\bh)  = \exp\left\{\left\langle\bm_i,\bh\right\rangle-\frac{1}{2}\left\langle\bh, \bQ_i\bh\right\rangle-\phi(\bm_i,\bQ_i)\right\}q_0\left(\bh\right)\, ,\\
&\tq_a(\bw) = \exp\left\{\left\langle\tbm_a,\bw\right\rangle-\frac{1}{2}\left\langle\bw, \tbQ_a\bw\right\rangle-\tphi(\tbm_a,\tbQ_a)\right\}\tq_0\left(\bw\right)\, ,
\end{split}
\end{align}
where $q_0(\,\cdot\,)$ is the density of $\normal(0,\id_k)$, and
$\tq_0(\,\cdot\,)$ is the density of $\Dir(\nu;k)$, and $\phi,\tphi:\reals^k\times \reals^{k\times k}\to \reals$ are defined 
implicitly by the normalization condition $\int q_i(\de\bh_i) = \int \tq_a(\de\bw_a) = 1$.
In the following we let  $\bm = (\bm_i)_{i\le d}$, $\tbm = (\tbm_a)_{a\le n}$ denote the set of parameters in these distributions; these can also be
viewed as  matrices $\bm\in\reals^{d\times k}$ and $\tbm\in\reals^{d\times k}$  whose $i$-th row is 
$\bm_i$ (in the former case) or $\tbm_i$ (in the latter).

It is useful to define the functions  $\sF, \tsF :\reals^k\times \reals^{k\times k}\to\reals^k$ and $\sG,\tsG :\reals^k\times \reals^{k\times k}\to\reals^{k\times k}$ 
as (proportional to) expectations with respect to the 
approximate posteriors (\ref{eq:densityforms_main})
\begin{align}
\sF(\bm_i; \bQ) &\equiv\sqrt{\beta}\, \int \bh\,\, q_i(\de \bh) \, ,\;\;\;\;\;
\tsF(\tbm_a; \tbQ) \equiv \sqrt{\beta}\, \int \bw \, \, \tq_a(\de \bw)\, ,\label{eq:defF_main}\\
\sG(\bm_i; \bQ) &\equiv \beta\, \int \bh^{\otimes 2} \,\,q_i(\de \bh) \, ,\;\;\;\;\;
\tsG(\tbm_a; \tbQ) \equiv \beta\, \int \bw^{\otimes 2} \, \, \tq_a(\de \bw)\, .
\end{align}
For $\bm\in\reals^{d\times k}$, we overload the notation and denote by $\sF(\bm;\bQ)\in\reals^{d\times k}$ the matrix whose $i$-th row is $\sF(\bm_i;\bQ)$
(and similarly for $\tsF(\tbm;\tbQ)$).

When restricted to a product-form ansatz with parametrization (\ref{eq:densityforms_main}), the mean field free energy takes the form (see Appendix \ref{app:NMF_Free_Energy})
\begin{align}
\KL(\hq\|p_{\bW,\bH|\bX}) = \cF(\br,\tbr,\bOmega,\tbOmega) +\frac{d}{2}\|\bX\|_{F}^2+\log p_{\bX}(\bX)\, ,
\end{align}
where 
\begin{align}
\label{eq:FreeEnergy_main} 
\cF(\br,\tbr,\bOmega,\tbOmega) = & 
\sum_{i=1}^d\psi_*(\br_i,\bOmega_i)+\sum_{a=1}^n\tpsi_*(\tbr_a,\tbOmega) -\sqrt{\beta}\Tr\left(\bX\br\tbr^{\sT}\right)+
\frac{\beta}{2d}\sum_{i=1}^d \sum_{a=1}^n\<\bOmega_i,\tbOmega_a\> \, ,\\
\psi_*(\br,\bOmega)  \equiv \sup_{\bm, \bQ}&\left\{\< \br, \bm\> -\frac{1}{2}\<\bOmega,\bQ\>- \phi(\bm, \bQ)\right\} \, ,\;\;\;\;
\tpsi_*(\tbr,\tbOmega)  \equiv \sup_{\tbm,\tbQ}\left\{\< \tbr, \tbm\> -\frac{1}{2}\<\tbOmega,\tbQ\>- \tphi(\tbm, \tbQ)\right\} \, ,\label{eq:LegendrePhi}
\end{align}
Note that  Eq.~(\ref{eq:LegendrePhi})  implies the following convex duality relation between $(\br,\tbr,\bOmega,\tbOmega)$ and $(\bm,\tbm,\bQ,\tbQ)$
\begin{align}
\br_i &\equiv \frac{1}{\sqrt{\beta}}\sF(\bm_i;\bQ)\,,\;\;\;\;\;\;\;\;\tbr_a \equiv \frac{1}{\sqrt{\beta}}\tsF(\tbm_a;\tbQ)\,,\label{eq:r_def}\\
\bOmega_i &\equiv \frac{1}{\beta} \sG(\bm_i;\bQ)\,,\;\;\;\;\;\;\;\;\tbOmega_a \equiv \frac{1}{\beta}\tsG(\tbm_a;\tbQ)\, .\label{eq:Omega_def}
\end{align}
By strict convexity of $\phi(\bm,\bQ)$, $\tphi(\tbm,\tbQ)$ (the latter is strongly convex on the hyperplane $\<\bfone,\tbm\>=0$, 
$\<\bfone,\tbQ\bfone\>=0$) we can view $\cF(\cdots )$
as a function of $(\br,\tbr,\bOmega,\tbOmega)$ or $(\bm,\tbm,\bQ,\tbQ)$. With an abuse of notation,
we will write  $\cF(\br,\tbr,\bOmega,\tbOmega)$ or $\cF(\bm,\tbm,\bQ,\tbQ)$ interchangeably.

A critical (stationary) point of the free energy (\ref{eq:FreeEnergy_main}) is a point at which $\nabla\cF(\bm,\tbm,\bQ,\tbQ) =\bzero$. 
It turns out that the mean field free energy always admits a point that does not distinguish between the $k$ latent factors,
and in particular $\bm = \bv\bfone_k^{\sT}$, $\tbm = \tbv\bfone_k^{\sT}$, as stated in detail below.  We will refer to this as the \emph{uninformative critical point}
(or \emph{uninformative fixed point}).
\begin{lemma}\label{lemma:Uninf}
Define $\sE(q;\nu) \equiv (\int w_1^2 e^{-q\|\bw\|_2^2}\, \tq_0(\de\bw))/(\int e^{-q\|\bw\|_2^2}\, \tq_0(\de\bw))$ and let $q_1^*$ be any solution of the following equation in $[0,\infty)$
\begin{align}
q_1^* = \frac{k\beta\delta}{k-1}\, \left\{\sE\left(\frac{\beta}{1+q_1^*};\nu\right) - \frac{1}{k^2}\right\}\, . \label{eq:qs_1_main}
\end{align}
(Such a solution always exists.) Further define 
\begin{align}
q_2^* &= \frac{\beta\delta-kq_1^*}{k^2}\, ,\;\;\;\;\;\tq_1^* = \frac{\beta}{1+q_1^*}\, ,\label{eq:qs_2_main}\\
\tq_2^* &= \beta\left(\frac{\|\bX^{\sT}\bfone_n\|_2^2}{d(1+q_1^*+kq_2^*)^2} - \frac{q_2^*}{(1+q_1^*)(1+q_1^*+kq_2^*)}\right)\, . \label{eq:qs_3_main}
\end{align}
Then the naive mean field free energy of Eq.~(\ref{eq:FreeEnergy_main})  admits a stationary point whereby, for all $i\in [d]$, $a\in [n]$,
\begin{align}
\bm_i^* &= \frac{\sqrt{\beta}}{k}\, (\bX^{\sT}\bfone_n)_i \, \bfone_k\, ,\\
\tbm_a^* &= \frac{\beta}{k(1+q_1^*+kq_2^*)}\, (\bX\bX^{\sT}\bfone_n)_a\, \bfone_k\, ,\\
\bQ_i^* &= q_1^*\id_k + q_2^*\bJ_k\, ,\;\;\;\; \tbQ_a^* = \tq_1^*\id_k + \tq_2^*\bJ_k\,.
\end{align}
\end{lemma}
The proof of this lemma is deferred to Appendix \ref{app:Uninformative}. We note that Eq.~(\ref{eq:qs_1_main}) appears to always have a unique solution.
Although we do not have a proof of uniqueness, in Appendix \ref{app:Uniqueness} we prove that the solution is unique conditional on a 
certain inequality that can be easily checked numerically. 

\subsection{Naive mean field iteration}

As mentioned in the introduction, the variational approximation of the free energy is often minimized by alternating  minimization over the marginals
$(q_i)_{i\le d}$, $(\tq_a)_{a\le n}$ of Eq.~(\ref{eq:ProductForm}). 
Using the parametrization (\ref{eq:densityforms_main}), we obtain the following naive mean field
iteration for $\bm^t, \tbm^t, \bQ^t,\tbQ^t$ (see Appendix \ref{app:NMF_ansatz}):
\begin{align}
\bm^{t+1}&= \bX^{\sT}\,\tsF(\tbm^t;\tbQ^t)\, ,\;\;\;\;\;\bQ^{t+1} = \frac{1}{d}\sum_{a=1}^n \tsG(\tbm^t_{a};\tbQ^t)\, ,\label{eq:NMF1_Main}\\
\tbm^t &= \bX\,\sF(\bm^t;\bQ^t)\, , \;\;\;\;\;\tbQ^{t} = \frac{1}{d}\sum_{i=1}^d \sG(\bm^t_{i};\bQ^t)\, .\label{eq:NMF2_Main}
\end{align}
Note that, while the free energy naturally depends on the $(\bQ_i)_{i\le d}$, $(\tbQ_a)_{a\le n}$, the iteration 
sets $\bQ^t_i = \bQ^t$, $\tbQ^t_a = \tbQ^t$, independent of the indices $i,a$. In fact, any stationary point of
$\cF(\bm,\tbm,\bQ,\tbQ)$ can be shown to be of this form.

The state of the iteration in Eqs.~(\ref{eq:NMF1_Main}), (\ref{eq:NMF2_Main}) is given by the pair $(\bm^t,\bQ^t)\in\reals^{d\times k}\times\reals^{k\times k}$, and $(\tbm^t,\tbQ^t)$ can be viewed as derived variables.
The iteration hence defines a mapping $\NMF_{\bX}:\reals^{d\times k}\times\reals^{k\times k}\to \reals^{d\times k}\times\reals^{k\times k}$, and 
we can write it in the form
\begin{align}
(\bm^{t+1},\bQ^{t+1}) = \NMF_{\bX}(\bm^{t},\bQ^{t})\, .
\end{align} 
Any critical point of the free energy (\ref{eq:FreeEnergy_main}) is a fixed point of the naive mean field iteration and vice-versa, as follows from 
Appendix \ref{app:NMF_Free_Energy}. 
In particular,
the uninformative critical point $(\bm^*,\tbm^*,\bQ^*,\tbQ^*)$ is a fixed point of the naive mean field iteration.

\subsection{Instability}

In view of Section \ref{sec:IT-main}, for $\beta<\beta_{\sBayes}(k,\delta,\nu)$,  the real posterior should be centered around a point symmetric under permutations
of the topics. In particular, the posterior $\tq(\bw_a)$ over the weights of document $a$ should be centered around the symmetric distribution $\bw_a = (1/k,\dots,1/k)$.
In other words, the uninformative fixed point should be a good approximation of the posterior for $\beta \leq \beta_{\sBayes}$. 

A minimum consistency condition for variational inference  is  that
the uninformative stationary point is a local minimum of the posterior for $\beta<\beta_{\sBayes}$. The next theorem provides a necessary condition for stability of the uninformative
point, which we expect to be tight. As discussed below,  it implies that this point is a saddle in an interval of $\beta$ below $\beta_{\sBayes}$. We 
recall that the index of a smooth function $f$ at stationary point $\bx_*$ is the number of the negative eigenvalues of the Hessian $\nabla^2f(\bx_*)$.
\begin{theorem}\label{thm:Main}
Define $q_1^*$, $q_2^*$ as in Eqs.~(\ref{eq:qs_1_main}), (\ref{eq:qs_2_main}), and let
\begin{align}
L(\beta,k, \delta,\nu) \equiv \frac{\beta(1+\sqrt{\delta})^2 }{1+q_1^*}\left(\frac{q_1^*}{\delta\beta} + k\left[\frac{q_2^*}{1+q_1^*+kq_2^*}\left(\frac{1}{\delta\beta}+\frac{1}{k}\right)-\frac{1}{k^2}\right]_+\right)\, .
\end{align}
If  $L(\beta,k,\delta,\nu)>1$, then there exists $\eps_1,\eps_2>0$ such that  the uninformative critical point of Lemma \ref{lemma:Uninf}, $(\bm^*,\tbm^*,\bQ^*,\tbQ^*)$ is, with high probability, a saddle point,
with index at least  $n\eps_1$ and $\lambda_{\min}(\cF|_{\bm^*,\tbm^*,\bQ^*,\tbQ^*})\le -\eps_2$.

Correspondingly $(\bm^*,\bQ^*)$ is an unstable critical point of the mapping $\NMF_{\bX}$ in the sense that the Jacobian $\bD\NMF_{\bX}$ has spectral
radius larger than one at $(\bm^*,\bQ^*)$. 
\end{theorem}
In the following, we will say that a fixed point $(\bm^*,\bQ^*)$ is stable if the linearization of $\NMF_{\bX}(\, \cdot\,)$ at
$(\bm^*,\bQ^*)$  (i.e. the Jacobian matrix $\bD \NMF_{\bX}(\bm^*,\bQ^*)$) has spectral radius smaller than one.
By the Hartman-Grobman linearization theorem \cite{perko2013differential},  this implies that $(\bm^*,\bQ^*)$ is an attractive fixed point. Namely, there exists a neighborhood $\cO$ of $(\bm^*,\bQ^*)$ 
such that, initializing the naive mean field iteration within that neighborhood,
results in $(\bm^t,\bQ^t)\to (\bm^*,\bQ^*)$ as $t\to\infty$. 
Vice-versa, we say that $(\bm^*,\bQ^*)$ is unstable if the Jacobian $\bD \NMF_{\bX}(\bm^*,\bQ^*)$ has spectral radius larger than one. In this case, for any neighborhood of
$(\bm^*,\bQ^*)$, and a generic initialization in that neighborhood, $(\bm^t,\bQ^t)$ does not converge to the fixed point.

Motivated by Theorem \ref{thm:Main}, we define the instability
threshold $\beta_{\inst} = \beta_{\inst}(k,\delta,\nu)$ by
\begin{align}
\beta_{\inst}(k,\delta,\nu) \equiv \inf\Big\{\beta\ge 0\, :\;\;   L(\beta,k,\delta,\nu)>1\, \Big\}\, .
\end{align}
Let us emphasize that, while we discuss the consequences of the instability at $\beta_{\inst}$ on the naive
mean field iteration, this is a problem of the variational free energy (\ref{eq:FreeEnergy_main}) and not of the
specific optimization algorithm.

\subsection{Numerical results for naive mean field}
\label{sec:NMF_numerical}

In order to investigate the impact of the instability described above,
we carried out extensive numerical simulations with the variational algorithm (\ref{eq:NMF1_Main}), (\ref{eq:NMF2_Main}).
After any number of iterations $t$, estimates of the factors $\bH$, $\bW$ are obtained by computing expectations with
respect to the marginals (\ref{eq:densityforms_main}). This results in
\begin{align}
\hbH^t = \br^t= \frac{1}{\sqrt{\beta}}\sF(\bm^t;\bQ_t)\, ,\;\;\;\;\;\;\hbW^t = \tbr^t=\frac{1}{\sqrt{\beta}}\tsF(\tbm^t;\tbQ_t)\, .\label{eq:Estimates}
\end{align}
Note that $(\hbH^t, \hbQ^t)$ can be used as the state of the naive mean-field iteration instead of $(\bm^t,\bQ^t)$.

We select a two-dimensional grid of $(\delta, \beta)$'s and generate $400$ different instances according to the LDA model for each grid point. 
We report various statistics of the estimates aggregated over the $400$ instances. We have performed the simulations for
 $\nu  =1 $ and $k\in\{2,3\}$. For space considerations, we focus here on the case $\nu = 1$, $k=2$, and
discuss other results in Appendix \ref{app:Numerical_MF}. (Simulations for other values of $\nu$ also yield similar results.)

We initialize both the naive mean field iteration near the uninformative fixed-point as follows:
\begin{align}
\hbH^0 &=(1 - \epsilon) \,\bH_* + \epsilon \frac{\bG }{\|\bG\|_F}\|\bH_* \|_F, \\
\bQ_0 &=  \bQ_*\,.
\end{align}
Here $\bG$  has entries $(G_{ij})_{i\le d,j\le k}\sim_{iid}\normal(0,1)$ and $\epsilon=0.01$ and $\bH_* = \sF(\bm_*,\bQ_*)/\sqrt{\beta}$
 is the estimate at the uninformative fixed point.
We run a maximum of $300$ and a minimum of $40$ iterations, and assess  convergence at iteration $t$ by evaluating 
\begin{equation} \label{eq:convergence_criteria_main}
\Delta_t = \min_{\bPi\in \Sym_k} \big\| \hbW^{t-1}\bPi - \hbW^{t} \big\|_\infty\, ,
\end{equation}
where the minimum is over the set $\Sym_k$ of $k\times k$ permutation matrices. We declare convergence when $\Delta_t<0.005$. 
We denote by $\hbH$, $\hbW$ the estimates obtained at convergence.
\begin{figure}[t!]
\phantom{A}
\vspace{-1cm}

\centering
\includegraphics[height=5.5in]{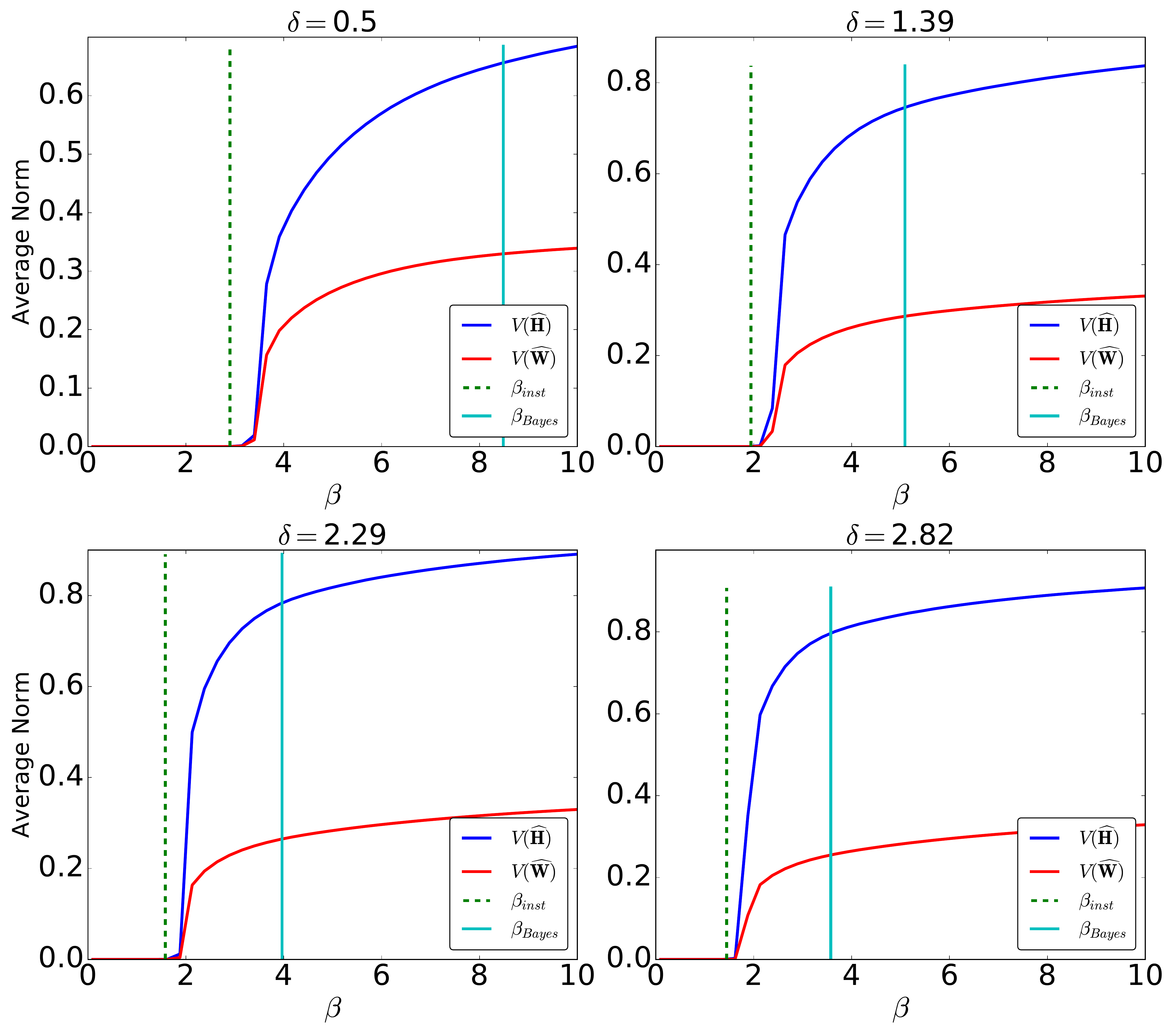}
\caption{Normalized distances $\Norm(\hbH)$, $\Norm(\hbW)$  of the
  naive mean field estimates from the uninformative fixed point.  Here $k=2$, 
 $d = 1000$ and $n= d\delta$: each data point corresponds to an average over $400$ random realizations.} 
\label{fig:H_norm_k_2}
\end{figure}

\begin{figure}[t!]
\phantom{A}\hspace{-1.85cm}\includegraphics[height=2.66in]{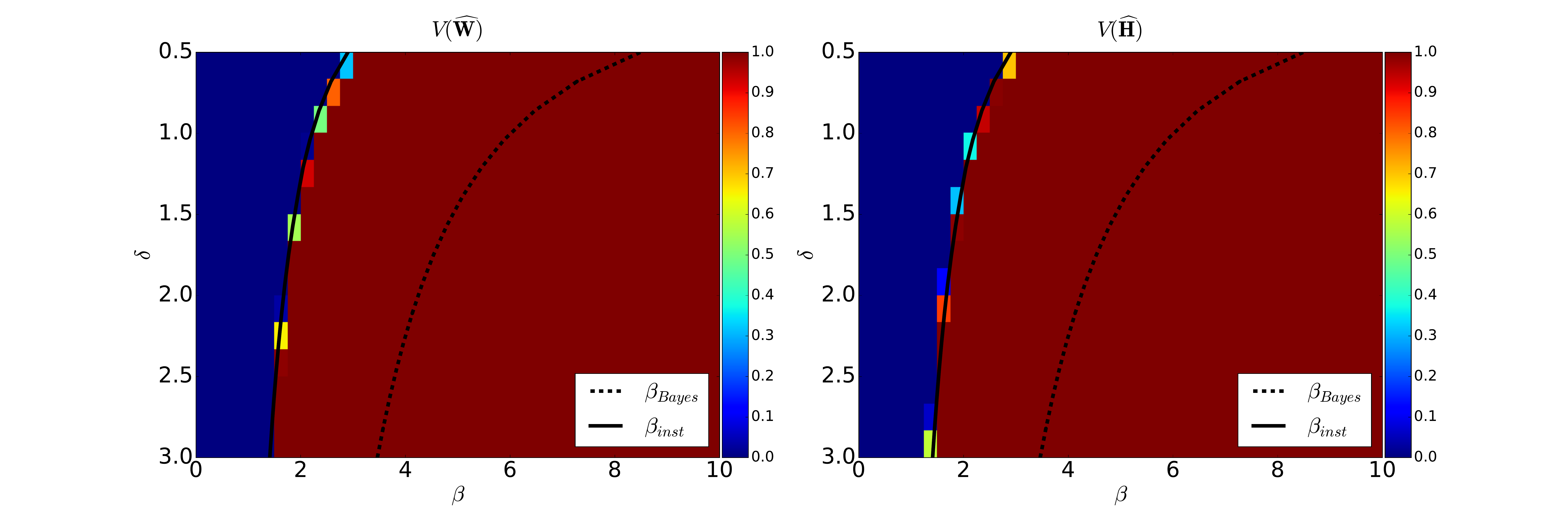}
\caption{Empirical fraction of instances such that  $\Norm(\hbW)\ge \eps_0=10^{-4}$ (left frame) or  $\Norm(\hbH)\ge \eps_0$ (right frame), where $\hbW,\hbH$ are the naive mean field
estimate. Here $k=2$, $d=1000$ and, for each $(\delta,\beta)$ point on a grid, we used $400$ random realizations to estimate the probability of $\Norm(\hbW)\ge \eps_0$.}
\label{fig:H_norm_k_2_HM}
\end{figure}
Recall the definition $\bPp=\id_k-\bfone_k\bfone_k^{\sT}/k$. In order
to investigate the instability of Theorem \ref{thm:Main}, we define the quantities
\begin{align}
\Norm(\hbW)\equiv \frac{1}{\sqrt{n}}\,\|\hbW\bPp\|_F\, ,\;\;\;\;\;\;\Norm(\hbH)\equiv \frac{1}{\sqrt{d}}\,\|\hbH\bPp\|_F
\end{align}
In Figure \ref{fig:H_norm_k_2} we plot empirical results for the average $\Norm(\hbW)$, $\Norm(\hbH)$ for $k=2$, $\nu=1$ and
four values of $\delta$. 
In Figure \ref{fig:H_norm_k_2_HM}, we plot the empirical probability that variational inference does not converge to 
the uninformative fixed point or, more precisely, $\hprob(\Norm(\hbW)\ge \eps_0)$ with $\eps_0= 10^{-4}$, evaluated on a grid of $(\beta,\delta)$
values. We also plot the Bayes threshold $\beta_{\sBayes}$ (which we find numerically that it coincides with the spectral threshold $\beta_{\sp}$)
and the instability threshold $\beta_{\inst}$. 

It is clear from Figures \ref{fig:H_norm_k_2}, \ref{fig:H_norm_k_2_HM}, that variational inference stops converging to
the uninformative fixed point (although we initialize close to it) when $\beta$ is still significantly smaller than the Bayes threshold $\beta_{\sBayes}$
(i.e. in a regime in which the uninformative fixed point would a reasonable output). The data are consistent with the hypothesis that
variational inference becomes unstable at $\beta_{\inst}$, as predicted by Theorem \ref{thm:Main}.

\begin{figure}[ht!]
\phantom{A}
\vspace{-1cm}

\includegraphics[height=5.5in]{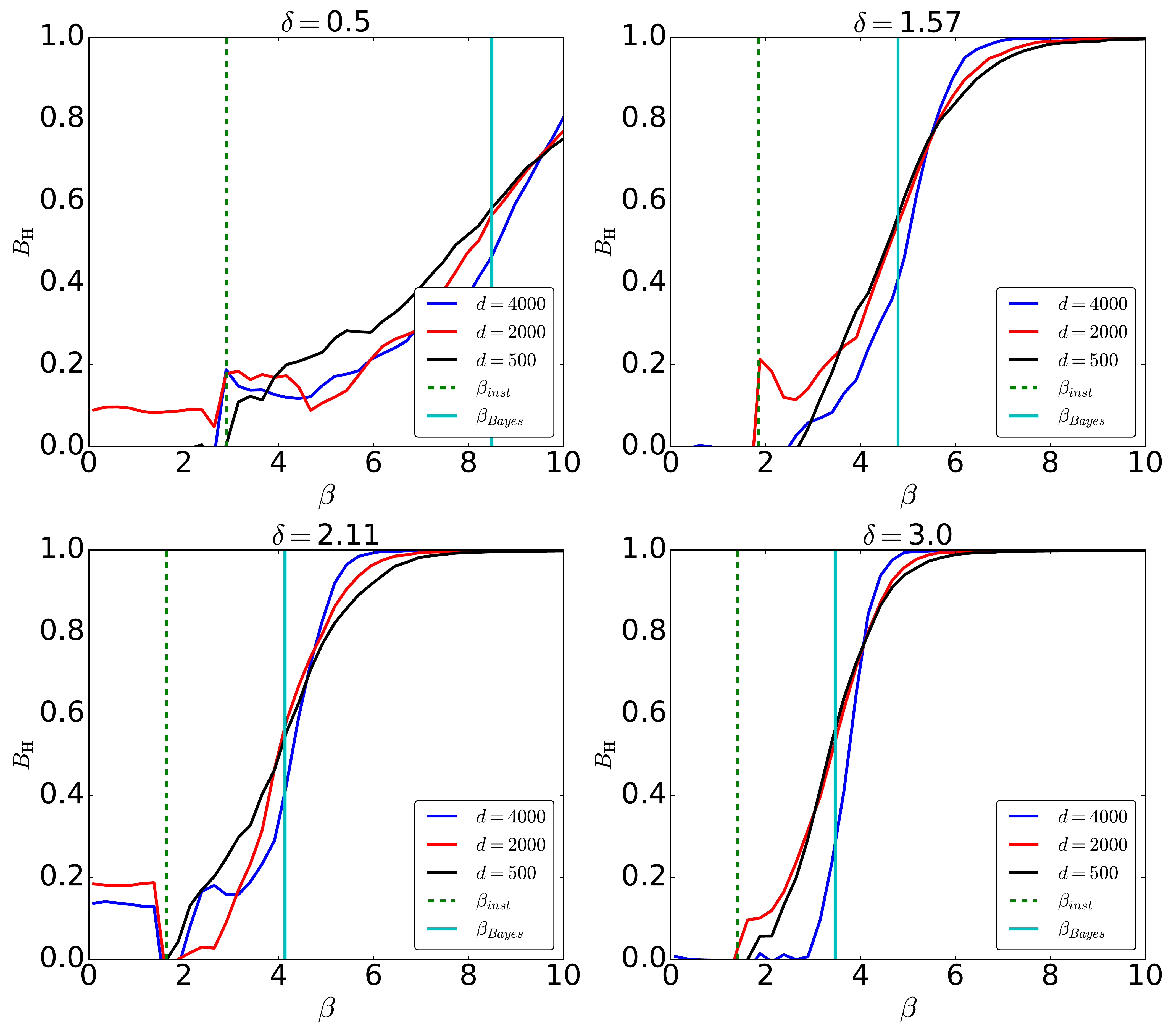}

\caption{Binder cumulant for the correlation between the naive mean field estimates $\hbH$ and the true topics $\bH$, see
Eq.~(\ref{eq:Binder_Def}). Here we report results for $k=2$, $d\in \{500,2000,4000\}$ and $n=d\delta$, obtained by averaging over $400$
realizations. Note that for $\beta<\beta_{\sBayes}(k,\nu,\delta)$,
 $\Bind_{\bH}$ decreases with increasing dimensions, suggesting asymptotically vanishing correlations.}
\label{fig:Binder_k_2}
\end{figure}

\begin{figure}[ht!]
\phantom{A}\hspace{-1.85cm}\includegraphics[height=2.66in]{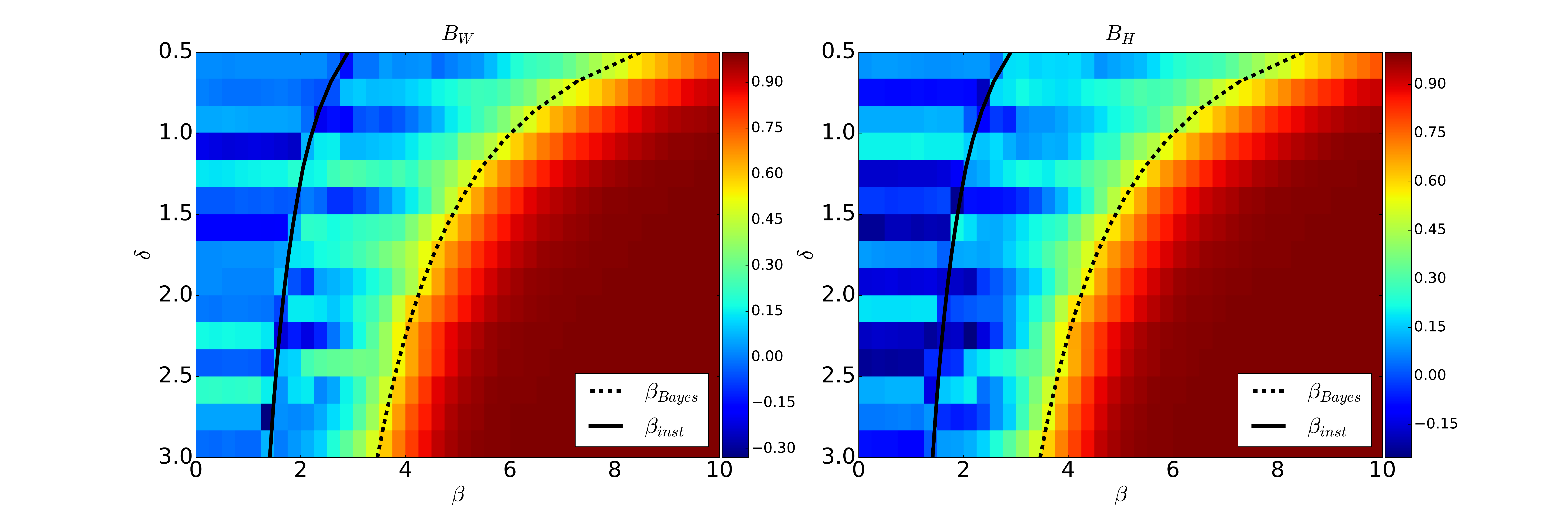}
\caption{Binder cumulant for the correlation between the naive mean field estimates $\hbW$, $\hbH$ and the true weights and topics
$\bW$, $\bH$. Here $k=2$, $d=1000$ and $n=d\delta$, and we averaged over $400$ realizations.}
\label{fig:Binder_k_2_HM}
\end{figure}
Because of Proposition \ref{propo:Bayes}, we expect the estimates $\hbH,\hbW$ produced by variational inference
to be asymptotically uncorrelated with the true factors for $\beta_{\inst}<\beta<\beta_{\sBayes}$. In order to test this
hypothesis, we borrow a technique that has been developed in the study of phase transitions in statistical physics, and
is known as the Binder cumulant \cite{binder1981finite}. For the sake of simplicity, we focus here --again-- on the case $k=2$, deferring the general case
to Appendix \ref{app:Numerical_MF}. Since in this case $\hbH,\bH\in \reals^{d\times 2}$, $\hbW,\bW\in \reals^{n\times 2}$, we can 
encode the informative component of these matrices by taking the difference between their columns.
For instance, we define $\hbh_{\perp} \equiv \hbH(\be_1-\be_2)$, and analogously $\bh_{\perp}$, $\hbw_{\perp}$, $\bw_{\perp}$.
We then define
\begin{align}
\Corr_{\eta}(\bH,\hbH) &\equiv \<\hbh_{\perp}+\eta \bg,\bh_{\perp}\>\, ,\;\;\;\;\;\;\;\; \Bind_{\bH} \equiv\frac{3}{2}-\frac{\hE\{\Corr_{\eta}(\bH,\hbH)^4\}}{2
\hE\{\Corr_{\eta}(\bH,\hbH)^2\}^2} \, .
\label{eq:Binder_Def}
\end{align}
Here $\hE$ denotes empirical average with respect to the sample,  $\bg\sim\normal(0,\id_d)$, and we set $\eta=10^{-4}$.
An analogous definition holds for $\Corr_{\eta}(\hbW)$, $\Bind_{\eta}(\hbW)$.

\begin{figure}[ht!]
\centering
\includegraphics[height=3.in]{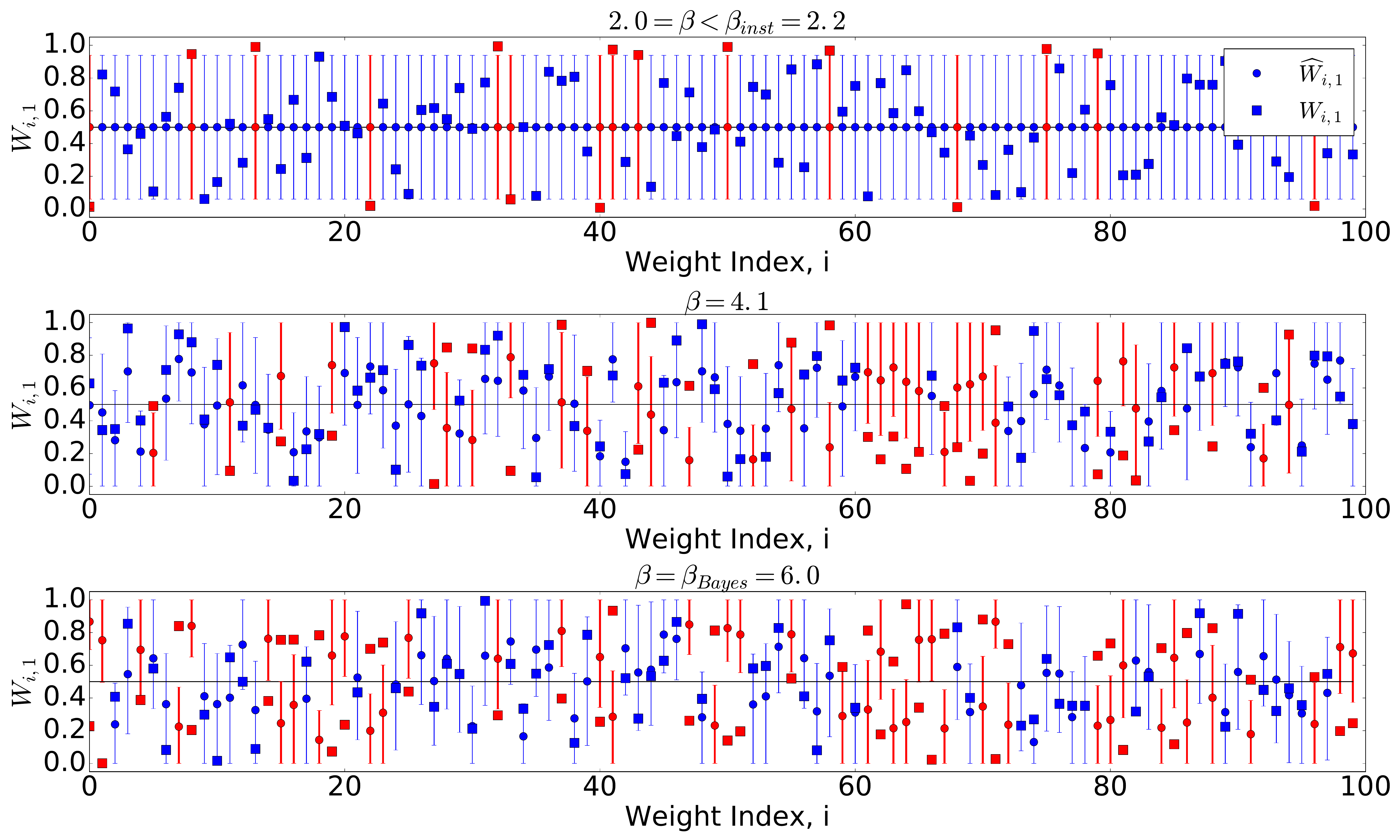}
\caption{Bayesian credible intervals as computed by variational inference at nominal coverage level $1-\alpha= 0.9$.
Here $k=2$, $n=d=5000$, and we consider three values of $\beta$: $\beta\in\{2,4.1,6\}$ (for reference $\beta_{\inst}\approx 2.2, \beta_{\sBayes}=6$).
Circles correspond to the posterior mean, and squares to the actual weights. We use red for the coordinates on which the 
credible interval does not cover the actual value of $w_{i,1}$.}
\label{fig:Uncertainty_delta1}
\end{figure}
The rationale for definition (\ref{eq:Binder_Def}) is easy to explain. At small signal-to-noise ratio $\beta$, we expect $\hbh_{\perp}$
to be essentially uncorrelated from $\bh_{\perp}$ and hence the correlation $\Corr_{\eta}(\bH,\hbH)$  to be roughly 
normal with mean zero and variance $\sigma^2_{\bH}$. 
In particular $\E\{\Corr_{\eta}(\bH,\hbH)^4\}\approx 3 \E\{\Corr_{\eta}(\bH,\hbH)^4\}$ and
therefore $\Bind_{\bH}\approx 0$.  (Note that the term $\eta\bg$ is added to avoid that empirical correlation vanishes, and hence $\Bind_{\bH}$ is not defined.)

In contrast, for large $\beta$, we expect $\hbh_{\perp}$ to be positively correlated with $\bh_{\perp}$, and 
$\Corr_{\eta}(\bH,\hbH)$ should concentrate around a non-random positive value.  As a consequence, $\Bind_{\bH}\approx 1$. 

In Figures \ref{fig:Binder_k_2} we report our empirical results for $\Bind_{\bH}$ and $\Bind_{\bW}$ for four different values of $\delta$,
and several values of $d$. 
As expected, these quantities grow from $0$ to $1$ as $\beta$ grows, and the transition is centered around $\beta_{\sBayes}$.
Figure \ref{fig:Binder_k_2_HM} reports the results on a grid of $(\beta,\delta)$ values. Again, the transition is
well predicted by the analytical curve $\beta_{\sBayes}$.
These data support our claim that, for $\beta_{\inst}<\beta<\beta_{\sBayes}$, the output of variational inference is
non-uniform but uncorrelated with the true signal.

Finally, in Figure \ref{fig:Uncertainty_delta1} we plot the estimates obtained for $100$ entries of the weights vector $w_{i,1}$ for 
three instances with $n=d=5000$ and $\beta=2<\beta_{\inst}$, $\beta= 4.1\in(\beta_{\inst},\beta_{\sBayes})$ and $\beta=6=\beta_{\sBayes}$.
The interval for $w_{a,1}$ is the form $\{w_{a,1}\in [0,1]: \tq_a(w_{a,1})\ge t_{a}(\alpha)\}$ and are constructed to achieve nominal coverage level $1-\alpha=0.9$. 
It is visually clear that the claimed coverage level is not verified in these simulations for $\beta>\beta_{\inst}$, confirming our  analytical results.
Indeed, for the three simulations in  Figure \ref{fig:Uncertainty_delta1} we achieve coverage $0.87$ 
(for $\beta=2<\beta_{\inst}$), $0.65$ (for $\beta= 4.1\in(\beta_{\inst},\beta_{\sBayes})$), and $0.51$ (for $\beta=6=\beta_{\sBayes}$).
Further results of this type are reported in Appendix \ref{app:Numerical_MF}.

\section{Fixing the instability}
\label{sec:Fixing}

The fact that naive mean field is not accurate for certain classes of random high-dimensional 
probability distributions is well understood within statistical physics. In particular, in the context of mean field spin glasses
\cite{SpinGlass}, naive mean field is known to lead to an asymptotically incorrect expression for 
the free energy. We expect the same mechanism to be relevant in the context of topic models.

Namely, the product-form expression (\ref{eq:ProductForm}) only holds asymptotically in the sense of 
finite-dimensional marginals. However, when computing the term $\E_{\hq}\log p_{\bX|\bW,\bH}(\bX|\bH,\bW)$ in the
KL divergence (\ref{eq:Gibbs2}), the error due to the product form approximation is non-negligible.
Keeping track of this error leads to the so-called TAP free energy.

\subsection{Revisiting $\integers_2$-synchronization}

It is instructive to briefly discuss the $\integers_2$-synchronization example of Section \ref{sec:Toy},
as the basic concepts can be explained more easily in this example. For this problem, the TAP approximation replaces the free energy 
(\ref{eq:Z2_FreeEnergy}) with 
\begin{align}
\cF_{\sTAP}(\bm) & \equiv -\frac{\lambda}{2}\<\bm,\bX_0\bm\> -\sum_{i=1}^n\entro(m_i)
-\frac{n\lambda^2}{4}\big(1-Q(\bm)\big)^2\,,
\end{align}
where $Q(\bm) \equiv \|\bm\|_2^2/n$. 

We can now repeat the analysis of Section \ref{sec:Toy} with this new free energy approximation. It is easy to see
that $\bm_*=\bzero$ is again a stationary point. However, the Hessian is now
\begin{align}
\left. \nabla^2\cF(\bm) \right|_{\bm = \bm_*} =  -\lambda\bX_0 +\left(1+\lambda^2\right)\id\, .
\end{align}
In particular, for $\lambda<1$, $\lambda_{\rm min}(\left.\nabla^2\cF\right|_{\bm = \bm_*})$  converges to  $(1-\lambda)^2>0$:
the uninformative stationary point is (with high probability) a local minimum.

The stationarity condition for the TAP free energy are known as TAP equations, and 
the algorithm that corresponds to the naive mean field iteration is  Bayesian approximate message passing (AMP).
For the $\integers_2$ synchronization problem, Bayes AMP is known to achieve the Bayes optimal estimation error 
\cite{deshpande2017asymptotic,montanari2017estimation}.

\subsection{TAP free energy for topic models}
\label{sec:TAP-Topic}

We now turn to topic models. The TAP approach replaces the free energy
(\ref{eq:FreeEnergy_main}) with the following (see Appendix \ref{app:TAP_Derivation} for a derivation)
\begin{align}
\label{eq:FreeEnergy_TAP_TM} 
\cF_{\sTAP}(\br,\tbr) = & \sum_{i=1}^d\psi\left(\br_i, \frac{\beta}{d}\sum_{a=1}^n\tbr_a^{\otimes 2}\right)+
\sum_{a=1}^n \tpsi\left(\tbr_a,  \frac{\beta}{d}\sum_{i=1}^d\br_i^{\otimes 2}\right) -\sqrt{\beta}\Tr\left(\bX\br\tbr^{\sT}\right)  -
\frac{\beta}{2d}\sum_{i=1}^d\sum_{a=1}^n\<\br_i,\tbr_a\>^2\, ,
\end{align}
where $\tbr\bfone_k = \bfone_n$, and we defined the partial Legendre transforms
\begin{align}
\psi(\br,\bQ)  \equiv \sup_{\bm}\left\{\< \br, \bm\> - \phi(\bm, \bQ)\right\} \, ,\;\;\;\;
\tpsi(\tbr,\tbQ) \equiv \sup_{\tbm}\left\{\< \tbr, \tbm\>- \tphi(\tbm, \tbQ)\right\} \, .\label{eq:LegendrePhiPartial}
\end{align}
Notice that $\tpsi(\tbr,\tbQ)$ is finite only if $\<\bfone_k,\tbr\>=1$.

When substituting in Eq.~(\ref{eq:FreeEnergy_TAP_TM}), the supremum of Eq.~(\ref{eq:LegendrePhiPartial}) is achieved at
\begin{align}
\br& = \frac{1}{\sqrt{\beta}}\sF(\bm;\bQ)\, ,\;\;\;\;\;\tbr = \frac{1}{\sqrt{\beta}}\tsF(\tbm;\tbQ)\, ,\label{eq:MapM-R}\\
\bQ&=\frac{\beta}{d}\sum_{a=1}^n\tbr_a^{\otimes 2},\;\;\;\;\;\;\tbQ\equiv \frac{\beta}{d}\sum_{i=1}^d\br_i^{\otimes 2}\, .
\end{align}

Calculus shows that stationary points of this free energy are in one-to-one correspondence (via Eq.~(\ref{eq:MapM-R}))
with the fixed points of the following iteration:
\begin{align}
\bm^{t+1}&= \bX^{\sT}\,\tsF(\tbm^t;\tbQ^t)-\sF(\bm^t;\bQ^t) \tbOmega_t\, ,\label{eq:AMP1}\\
\tbm^t &= \bX\,\sF(\bm^t;\bQ^t)-\tsF(\tbm^{t-1};\tbQ^{t-1}) \bOmega_t\, ,\label{eq:AMP2}\\
\bQ^{t+1} &= \frac{1}{d}\sum_{a=1}^n \tsF(\tbm^t_a;\tbQ^t)^{\otimes 2}\, ,\;\;\;\; \tbQ^t = \frac{1}{d}\sum_{i=1}^d \sF(\bm^t_i;\bQ^t)^{\otimes 2}\, .
\label{eq:Calibration}
\end{align}
where $\bOmega_t$, $\tbOmega_t$ are defined as
\begin{align}
\bOmega_t& =\frac{1}{d\sqrt{\beta}}\sum_{i=1}^d [\sG(\bm^t_i,\bQ^t)-\sF(\bm^t_i;\bQ^t)^{\otimes 2}]=   \frac{1}{d}\sum_{i=1}^d\frac{\partial\sF}{\partial \bm_i}(\bm^t_i;\bQ^t)\, ,\label{eq:OmegaTAP1}\\
\tbOmega_t& =\frac{1}{d\sqrt{\beta}}\sum_{a=1}^n[\tsG(\tbm^t_a,\tbQ)- \tsF(\tbm^t_a;\tbQ^t)^{\otimes 2}]=   \frac{1}{d}\sum_{a=1}^n\frac{\partial\tsF}{\partial \tbm_a}(\tbm^t_a;\tbQ^t)\, .\label{eq:OmegaTAP2}
\end{align}

The stationarity conditions for the TAP free energy (\ref{eq:FreeEnergy_TAP_TM}) are known as TAP equations, 
and the corresponding iterative algorithm (\ref{eq:AMP1}), (\ref{eq:AMP2})
is a special case of approximate message passing (AMP), with Bayesian updates. Note that the specific choice of time indices in Eqs. ~(\ref{eq:AMP1}), (\ref{eq:AMP2}) is instrumental
for the analysis in the next section to hold.
We also note that the general AMP analysis of \cite{BM-MPCS-2011,javanmard2013state} allows for quite general choices of the sequence of matrices
 $\bQ_t, \tbQ_t$. 
However, stationarity of the TAP free energy (\ref{eq:FreeEnergy_TAP_TM}) requires that at convergence the 
condition (\ref{eq:Calibration}) holds at the fixed point

Estimates of the factors $\bW$, $\bH$ are computed following the same
recipe as for naive mean field, cf. Eq.~(\ref{eq:Estimates}), namely
$\hbH^t = \br^t = \sF(\bm^t;\bQ_t)/\sqrt{\beta}$, $\hbW^t = \tbr^t=\tsF(\tbm^t;\tbQ_t)/\sqrt{\beta}$.

It is not hard to see that the AMP iteration admits an uninformative fixed point, which is a stationary point of the TAP free energy, see proof in 
Appendix \ref{app:UninformativeTAP}.
\begin{lemma}\label{lemma:Uninf_TAP}
Define $q_0^* = \beta\delta/k^2$ and $\tq_0^* = \beta^2\|\bX^{\sT}\bfone_n\|_2^2/(dk^2(1+kq_0)^2)$. Then,
AMP iteration admits the following fixed point
\begin{align}
\bm^* & = \frac{\sqrt{\beta}}{k}(\bX^{\sT}\bfone_n)\otimes \bfone_k\, ,\label{eq:UninfTAP_1}\\
\tbm^* & = \frac{\beta}{k(1+kq_0)} (\bX\bX^{\sT}\bfone_n)\otimes\bfone_k - \frac{\beta}{k+\delta\beta} \, \bfone_n\otimes\bfone_k\, ,\label{eq:UninfTAP_2}\\
\bQ^*& = q_0^*\, \bJ_k\, ,\;\;\;\;\;\;\;\; \tbQ^* = \tq_0^* \, \bJ_k\, .
\end{align}
This corresponds to a stationary  point of the TAP free energy (\ref{eq:FreeEnergy_TAP_TM}),
via  Eq.~(\ref{eq:MapM-R}):
\begin{align}
\br_* = \frac{\sqrt{\beta}}{k(1+kq_0^*)} (\bX^{\sT}\bfone_n)\otimes \bfone_k\,,\;\;\;\;\;\;\; \tbr_* = \frac{1}{k}\bfone_n\otimes\bfone_k\, .
\end{align}
Further, this is the only stationary point that is unchanged under permutations of the topics.
\end{lemma}
\subsection{State evolution analysis}
\label{sec:StateEvol}

State evolution is a recursion over matrices $\bM_t$, $\tbM_t\in\reals^{k\times k}$, defined by
\begin{align}
\bM_{t+1} & = \delta\,  \E\Big\{\tsF(\tbM_t\bw+\tbM_t^{1/2}\bz;\tbM_t)^{\otimes 2}\Big\}\, ,\label{eq:FirstSE}\\
\tbM_{t} & = \E\Big\{\sF(\bM_t\bh+\bM_t^{1/2}\bz;\bM_t)^{\otimes 2}\Big\}\, , \label{eq:SecondSE}
\end{align}
where expectation is with respect to $\bh\sim q_0(\,\cdot\,)$, $\bw\sim \tq_0(\,\cdot\,)$ and $\bz\sim \normal(0,\id_k)$ independent.
Note that $\bM_t, \tbM_t$ are positive semidefinite symmetric matrices. Also, Eq.~(\ref{eq:SecondSE}) can be written explicitly as
\begin{align}
\tbM_t = \beta(\id_k+\bM_t)^{-1}\bM_t\, .
\end{align}
State evolution provides an asymptotically exact characterization of the behavior of AMP, as formalized by the 
next theorem (which is a direct application of \cite{javanmard2013state}). 
\begin{theorem}\label{thm:SE}
Consider the AMP algorithm of Eqs.~(\ref{eq:AMP1}), with deterministic  initialization $\bm^0,\bQ^0$.
Assume $\bG\in\reals^{d\times k}$ to be independent of data $\bX$, with 
entries $(G_{ij})_{i\le d,j\le k}\sim_{iid}\normal(0,1)$, and let
$\bm^0 = \bH\bM_0+\bZ\bM_0^{1/2}$ for $\bM_0\in \reals^{k\times k}$ non-random, $\bM_0\succeq 0$. Let $\{\bM_t,\tbM_t\}_{t\ge 1}$
be defined by the state evolution recursion (\ref{eq:FirstSE}), (\ref{eq:SecondSE}).  Then, for any pseudo-Lipschitz function $g:\reals^k\times\reals^k\to\reals$, we have, almost surely,
\begin{align}
\lim_{n\to\infty}\frac{1}{d}\sum_{i=1}^d g(\bh_i,\bm^t_i) & =\E\Big\{g(\bh,\bM_t\bh+\bM_t^{1/2}\bz)\Big\} \, ,\\
\lim_{n\to\infty}\frac{1}{n}\sum_{a=1}^n g(\bw_a,\tbm^t_a) & =\E\Big\{g(\bw,\tbM_t\bw+\tbM_t^{1/2}\bz)\Big\} \, ,
\end{align}
where it is understood that $n,d\to\infty$ with $n/d\to\delta$.  In particular
\begin{align}
\lim_{n\to\infty}\frac{1}{d}\bH^{\sT}\hbH^t & =\frac{1}{\sqrt{\beta}}\tbM_t\, ,\\
\lim_{n\to\infty}\frac{1}{n}\bW^{\sT}\hbW^t& =\frac{1}{\sqrt{\beta}}\bM_{t+1} \, .
\end{align}
Further $\lim_{n\to \infty}\bQ^t = \bM_t$, $\lim_{n\to\infty}\tbQ^t = \tbM_t$.
\end{theorem}

Using state evolution, we can establish a stability result for AMP. First of all, notice that the state evolution iteration
(\ref{eq:FirstSE}), (\ref{eq:SecondSE}) admits a fixed point of the form $\bM^* = (\delta\beta/k^2)\bJ_k$, $\tbM^* = \rho_0\bJ_k$, 
for $\rho_0 = \delta\beta^2/(k\delta\beta + k^2)$, see Appendix \ref{app:SE_FP}. This is an uninformative fixed point,
in the sense that the $k$ topics are asymptotically identical. The
next theorem is proved in Appendix \ref{sec:StabilitySE}.
\begin{theorem}\label{thm:StateEvolStable}
If $\beta<\beta_{\sp}(k,\nu,\delta)$, then the uninformative fixed point is stable under the state evolution iteration 
(\ref{eq:FirstSE}), (\ref{eq:SecondSE}). 

In particular, for $\beta<\beta_{\sp}(k,\nu,\delta)$, there exists $c_0=c_0(\beta,k\nu,\delta)$ such that,
if we initialize AMP as in Theorem \ref{thm:SE} with $\|\bM_0-\bM^*\|_F\le c_0$, then (recalling $\bPp = \id_k-\bfone_k\bfone_k/k$)
\begin{align}
\lim_{t\to\infty}\lim_{n\to\infty}\frac{1}{n}\big\|\bm^t\bPp\|_F^2 = 0\, ,\;\;\;\;\;\;\lim_{t\to\infty}\lim_{n\to\infty}\frac{1}{n}\big\|\bm^t\bPp\|_F^2 = 0\, .
\end{align}
\end{theorem}

\subsection{Stability of the uninformative fixed point}

The next theorem establishes that the uninformative fixed point of the TAP free energy
is a local minimum for all $\beta$ below the spectral threshold $\beta_{\sp}(k,\nu,\delta)$. Since $\beta_{\sBayes}(k,\nu,\delta)\le \beta_{\sp}(k,\nu,\delta)$,
this shows that the instability we discovered in the case of naive mean field is corrected by the TAP free energy.
\begin{theorem}\label{thm:StabilityTAP}
Let $(\br_*,\tbr_*)$ be the uninformative stationary point of the TAP free energy, cf. Lemma \ref{lemma:Uninf_TAP}.
If $\beta<\beta_{\sp}(k,\nu,\delta)$, then there exists $\eps>0$ such that, with high probability
\begin{align}
\lambda_{\min}\left(\left.\nabla^2\cF_{\sTAP}\right|_{(\br_*,\tbr_*)}\right)\ge \eps\, .
\end{align}
\end{theorem}

\begin{remark}
Let us emphasize that this result is not implied by the state evolution result of Theorem \ref{thm:StateEvolStable},
which only establishes stability in a certain asymptotic sense.
Vice-versa, Theorem \ref{thm:StabilityTAP} does not directly imply Theorem \ref{thm:StateEvolStable}.
\end{remark}

\subsection{Numerical results for TAP free energy}
\label{sec:TAP_numerical}

\begin{figure}[t!]
\phantom{A}
\vspace{-1cm}

\centering
\includegraphics[height=5.5in]{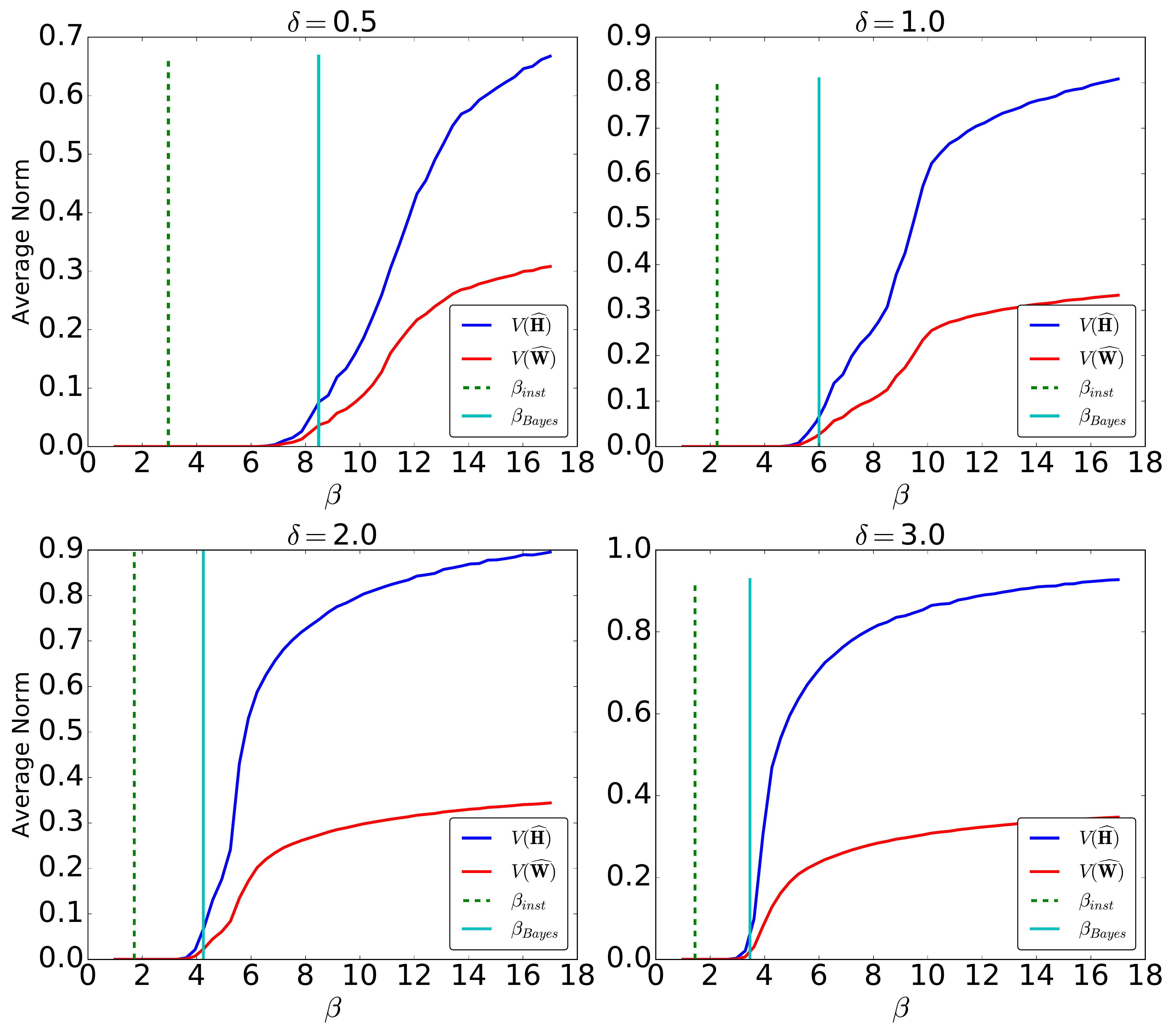}
\caption{Normalized distances $\Norm(\hbH)$, $\Norm(\hbW)$  of the AMP estimates from the uninformative fixed point.  Here,
$k=2$, $d = 1000$ and $n= d\delta$: each data point corresponds to an average over $400$ random realizations.} 
\label{fig:AMP_norm_k_2}
\end{figure}

\begin{figure}[t!]
\phantom{A}\hspace{-1.85cm}\includegraphics[height=2.66in]{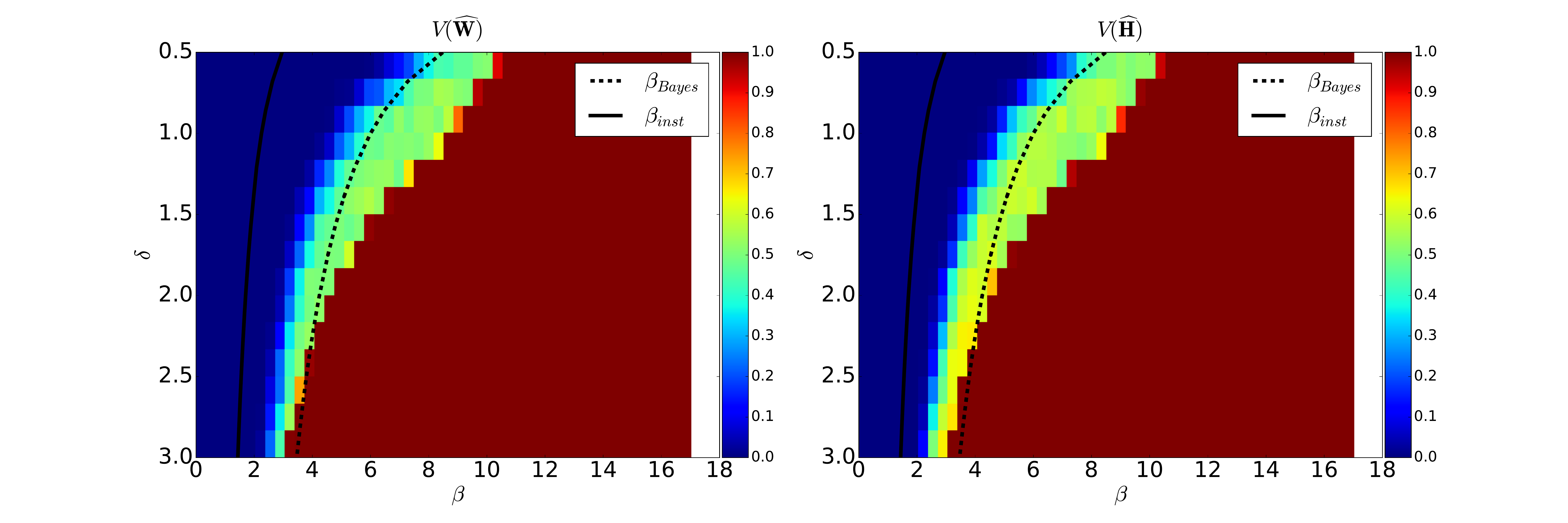}
\caption{Empirical fraction of instances such that  $\Norm(\hbW)\ge
  \eps_0=5\cdot 10^{-3}$, where $\hbW$ is the AMP estimate. Here $k=2$, $d=1000$,   
and for each $(\delta,\beta)$ point on the grid we ran AMP on  $400$ random realizations.}
\label{fig:AMP_norm_k_2_HM}
\end{figure}

\begin{figure}[ht!]
\phantom{A}
\vspace{-1cm}

\centering
\includegraphics[height=5.5in]{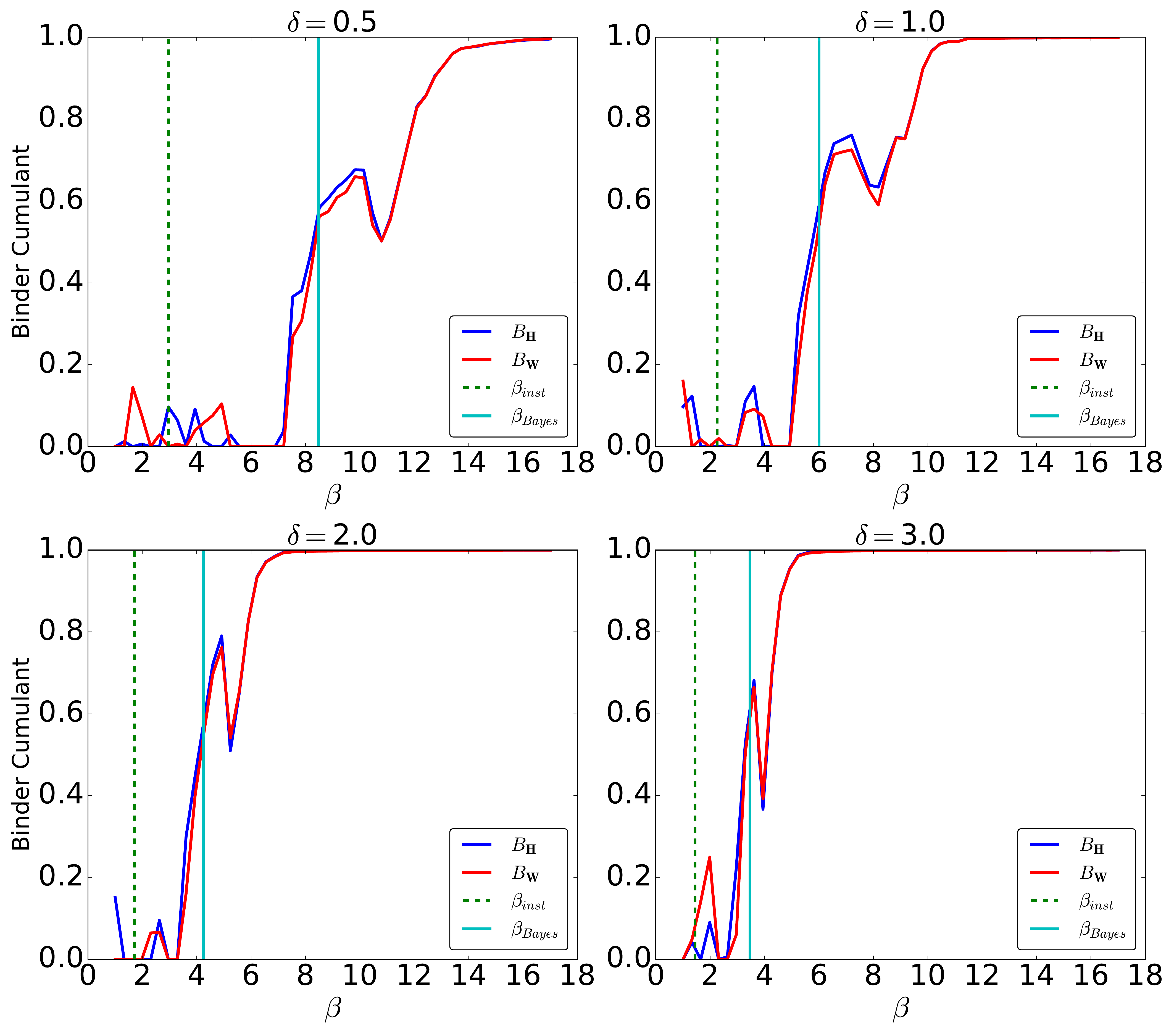}
\caption{Binder cumulant for the correlation between AMP estimates $\hbH$ and the true topics $\bH$, and between $\hbW$ and $\bW$, see
Eq.~(\ref{eq:Binder_Def}). Here $k=2$, $d=1000$,  $n=d\delta$ and estimates are obtained by averaging over $400$ realizations.} 
\label{fig:AMP_corrs_k_2}
\end{figure}

\begin{figure}[ht!]
\phantom{A}\hspace{-1.85cm}\includegraphics[height=2.66in]{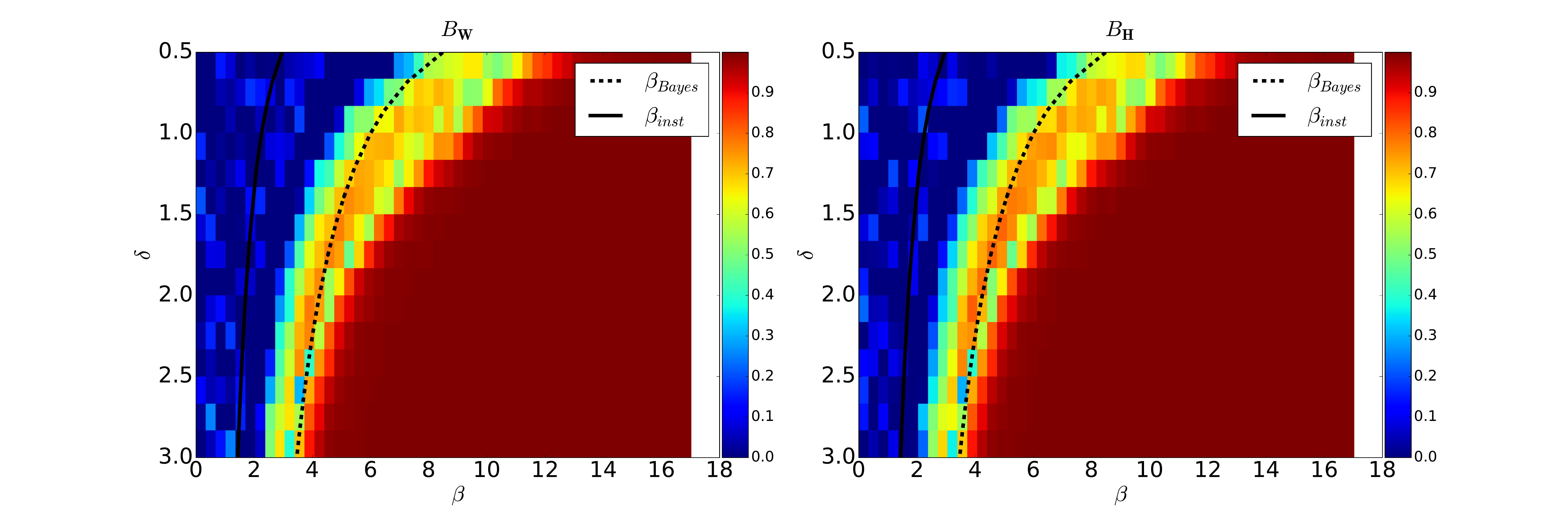}
\caption{Binder cumulant for the correlation between AMP estimates  $\hbW$, $\hbH$ and the true weights and topics $\bW, \bH$.
Here $k=2$,$d=1000$ and estimates are obtained by averaging over $400$ realizations.}
\label{fig:AMP_corrs_k_2_HM}
\end{figure}
In order to confirm the stability analysis at the previous section,  we carried out numerical simulations 
analogous to the ones of Section \ref{sec:NMF_numerical}. We found that the AMP iteration of Eqs.~(\ref{eq:AMP1}), 
(\ref{eq:AMP2}) is somewhat unstable when  $\beta\approx \beta_{\sp}$. In order to remedy  this problem, we used a damped version of the
same iteration, see Appendix \ref{app:Damped}. 
Notice that damping does not change the stability of a local minimum or saddle, 
it merely reduces oscillations due to aggressive step sizes.

We initialize the iteration as for naive mean field, and monitor the same quantities, as in Section
\ref{sec:NMF_numerical}.
 In particular, here we report results on the distance from the uninformative subspace $\Norm(\hbH)$, $\Norm(\hbW)$,
in Figures \ref{fig:AMP_norm_k_2} and \ref{fig:AMP_norm_k_2_HM}, and  the Binder cumulants $\Bind_{\bH}$ and $\Bind_{\bW}$,
measuring the correlation between AMP estimates and the true factors $\bW, \bH$, in Figures \ref{fig:AMP_corrs_k_2},  \ref{fig:AMP_corrs_k_2_HM}.
We focus on the case $k=2$, deferring $k=3$ to the appendices.

In the intermediate regime $\beta\in (\beta_{\inst},\beta_{\sp})$, the behavior of AMP is strikingly different from the one of 
naive mean field. AMP remains close to the uninformative fixed point, confirming that this is a local minimum of the TAP 
free energy. The distance from the uninformative subspace starts growing only at the spectral threshold $\beta_{\sp}$ 
(which coincides, in the present cases, with the Bayes threshold $\beta_{\sBayes}$). At the same point, the correlation with the
true factors $\bW$, $\bH$ also becomes strictly positive.

\section{Discussion}

Bayesian methods are particularly attractive in unsupervised learning problems such as topic modeling. 
Faced with a collection of documents $\bx_1$,\dots $\bx_n$, it is not clear a priori whether they should be 
modeled as convex combinations of topics, or how many topics should be used. Even after a low-rank factorization
$\bX\approx \bW\bH^{\sT}$ is computed, it is still unclear how to evaluate it, or to which extent it should be trusted.

Bayesian approaches provide estimates of the factors $\bW$, $\bH$, but also a probabilistic measure of how much
these estimates should be trusted. To the extent that the posterior concentrates around its mean, this can be considered as 
a good estimate of a true underlying signal.

It is well understood that Bayesian estimates can be unreliable if the prior is not chosen carefully. 
Our work points at a second reason for caution. When variational inference is used for approximating the posterior,
the result can be incorrect even if the data are generated according to the prior. More precisely, we showed that
for a certain regime of parameters, naive mean field `believes' that there is a signal, even if it is information-theoretically
impossible to extract any non-trivial estimate from the data. 

Given that naive mean field is the method of choice for inference with topic models \cite{blei2003latent},
it would be of great interest to remedy this instability. We showed that the TAP free energy provides a better 
mean field approximation, and in particular does not have the same instability. However, this approximation
is also based on the correctness of the generative model, and further investigation is warranted on its robustness.

\section*{Acknowledgements}

H.J. and A.M. were partially supported by grants NSF CCF-1714305 and NSF IIS-1741162. B.G. was supported by Stanford's Caroline and Fabian Pease Graduate Fellowship.

\bibliographystyle{amsalpha}

\newcommand{\etalchar}[1]{$^{#1}$}
\providecommand{\bysame}{\leavevmode\hbox to3em{\hrulefill}\thinspace}
\providecommand{\MR}{\relax\ifhmode\unskip\space\fi MR }
\providecommand{\MRhref}[2]{%
  \href{http://www.ams.org/mathscinet-getitem?mr=#1}{#2}
}
\providecommand{\href}[2]{#2}

\addcontentsline{toc}{section}{References}

\newpage
\appendix

\section{Some remarks on alternating minimization}

Let $f: \reals^{n}\times \reals^d \to \reals$ be twice continuously differentiable in an open neighborhood $\Omega_1\times \Omega_2\subseteq\reals^{n}\times\reals^d$ 
of a critical point $(\bx^*, \by^*)$ (i.e. a point for which $\nabla_{(\bx,\by)}f(\bx,\by) = \bzero$). Further assume that, fixing $\bx_0\in \Omega_1$, 
$f(\bx_0,\,\cdot\,)$ is strongly convex with a minimizer in $\Omega_2$, and fixing $\by_0\in\Omega_2$, $f(\,\cdot\,,\by_0)$ is strongly convex 
with a minimizer in $\Omega_1$. 
By taking $\Omega_1$ and $\Omega_2$ sufficiently small, these conditions follow by requiring that the partial Hessians satisfy
$\nabla^2_{\bx}f(\bx^*,\by^*)\succ \bzero$ and $\nabla^2_{\by}f(\bx^*,\by^*)\succ \bzero$ (i.e. they are strictly positive definite).

By strong convexity, the minimizers of $f(\bx_0,\,\cdot\,)$ and $f(\,\cdot\,,\by_0)$ are unique, and 
we can  define the functions $g:\reals^d\to\reals^n$ and  $h:\reals^n\to\reals^d$ by
\begin{align}
h(\bx_0) = \arg\min_{\by\in\Omega_2} f(\bx_0,\by)\, ,\\
g(\by_0) = \arg\min_{\bx\in\Omega_1} f(\bx,\by_0)\, .
\end{align}
We then define the alternating minimization iteration
\begin{align}
\label{eq:alternatemin}
\bx^{t+1} = h(\by^t),\;\;\;\;\;\;
\by^t = g(\bx^t)\, .
\end{align}
If $d=n$ and $h:\Omega_1\to\Omega_2$, $g:\Omega_2\to\Omega_1$ are bijective, we also define the dual iteration
\begin{align}
\label{eq:alternatemin_dual}
\obx^{t+1} = g^{-1}(\oby^t),\;\;\;\;\;\;
\oby^t = h^{-1}(\obx^t)\, .
\end{align}
\begin{lemma}
\label{lemma:hessianstable}
Let $f: \reals^{n}\times \reals^d \to \reals$ by twice continuously differentiable in $\Omega_1\times \Omega_2$,
satisfying the above assumptions. Then the following are equivalent:
\begin{itemize}
\item[{\bf (A1)}] The Hessian $\bH = \nabla^2_{(\bx, \by)} f\big|_{(\bx, \by) = (\bx^*, \by^*)}$  is strictly positive definite.
\item[{\bf (A2)}] $(\bx^*,\by^*)$ is a stable fixed point of the alternate minimization algorithm (\ref{eq:alternatemin}). 
\item[{\bf (A3)}] $f_1(\bx) \equiv \min_{\by\in\Omega_2} f(\bx,\by)$ is strongly convex in a neighborhood of $\bx^*$
(and in particular, $\bx^*$ is a local minimum of $f_1$).
\end{itemize}

Further, if $n=d$ and the matrix $\left.\frac{\partial f}{\partial\bx\partial \by}\right|_{\bx^*,\by^*}$ is invertible,
then the following are equivalent:
\begin{itemize}
\item[{\bf (B1)}] $(\bx^*,\by^*)$ is a stable fixed point of the dual algorithm (\ref{eq:alternatemin_dual}). 
\item[{\bf (B2)}] $f_1(\bx) \equiv \min_{\by\in\Omega_2} f(\bx,\by)$ is strongly concave in a neighborhood of $\bx^*$
(and in particular, $\bx^*$ is a local maximum).
\end{itemize}
\end{lemma}

\begin{proof}
Let 
\begin{align}
\bH = \begin{bmatrix}
\bH_{\bx\bx} & \bH_{\bx\by} \\
\bH_{\bx\by}^\sT & \bH_{\by\by}
\end{bmatrix} = \nabla^2_{(\bx, \by)} f\big|_{(\bx, \by) = (\bx^*, \by^*)}.
\end{align}
\noindent {\bf (A1)}$\equiv${\bf (A2)} We compute the linearization of the iterations in \eqref{eq:alternatemin}
around the fixed point $(\bx^*, \by^*)$. Note that since 
$\bx^*$ is a minimizer of $f(\,\cdot\, ,\by^*)$, using the implicit function theorem for 
the Jacobian of the update rule for $\bx$ in \eqref{eq:alternatemin} we have
\begin{align}
\frac {\partial^2 f}{\partial \bx \partial \by}\bigg|_{(\bx, \by) = (\bx^*, \by^*)} + \left[\frac{\partial^2 f}{\partial \bx^{2}}\bigg|_{(\bx, \by) = (\bx^*, \by^*)}\right]\left[\bD h(\by^*)\right] = 0.
\end{align}
Hence, we get
\begin{align}
\bD h(\by^*)= - \left[\left(\frac{\partial^2 f}{\partial \bx^{2}}\right)^{-1} \left(\frac {\partial^2 f}{\partial \bx \partial \by}\right)\right]_{(\bx, \by) = (\bx^*, \by^*)} = -\bH_{\bx\bx}^{-1}\bH_{\bx\by}.
\end{align}
Similarly, for the Jacobian of the update rule for $\by$ in \eqref{eq:alternatemin} we have 
\begin{align}
\bD g(\bx^*)= - \left[\left(\frac{\partial^2 f}{\partial \by^{2}}\right)^{-1} \left(\frac {\partial^2 f}{\partial \by \partial \bx}\right)\right]_{(\bx, \by) = (\bx^*, \by^*)} = -\bH_{\by\by}^{-1}\bH_{\bx\by}^\sT.
\label{eq:JacobianG}
\end{align}
Hence, $(\bx^*, \by^*)$ is stable if and only if the operator 
\begin{align}
\bL = \bD h(\bx^*) \cdot \bD g(\by^*) = \bH_{\bx\bx}^{-1}\bH_{\bx\by} \bH_{\by\by}^{-1}\bH_{\bx\by}^\sT\, ,
\end{align}
has spectral radius
\begin{align}
\sigma(\bL)  \equiv \max_{i}\left|\lambda_i\left(\bL\right)\right| < 1.
\end{align}
Since $f(\,\cdot\,,\bx^*)$ is strongly convex, the matrices $\bH_{\bx\bx}, \bH_{\bx\bx}^{-1}$ are positive definite.
Hence, the eigenvalues of $\bH_{\bx\bx}^{-1}\bH_{\bx\by}\bH_{\by\by}^{-1}\bH_{\bx\by}^\sT$ are 
real and equal to the eigenvalues of the symmetric positive semi-definite matrix $\bH_{\bx\bx}^{-1/2}\bH_{\bx\by}\bH_{\by\by}^{-1}\bH_{\bx\by}^\sT\bH_{\bx\bx}^{-1/2}$.  
Therefore, $\sigma(\bL)<1$ if and only if
\begin{align}
\bH_{\bx\bx}^{-1/2}\bH_{\bx\by}\bH_{\by\by}^{-1}\bH_{\bx\by}^\sT\bH_{\bx\bx}^{-1/2} \prec \id_n \iff 
\bH_{\bx\by}\bH_{\by\by}^{-1}\bH_{\bx\by}^\sT \prec \bH_{\bx\bx} \iff \bH_{\bx\bx} - \bH_{\bx\by}\bH_{\by\by}^{-1}\bH_{\bx\by}^\sT\succ 0.
\end{align}
Note that since $f(\bx^*,\,\cdot\,)$ is convex, $\bH_{\by\by} \succ 0$. Therefore, $\bH_{\bx\bx} - \bH_{\bx\by}\bH_{\by\by}^{-1}\bH_{\bx\by}^\sT\succ 0$ if and only if $\bH \succ 0$. Hence, the fixed point is stable 
if and only if $\bH \succ 0$ and this completes the proof.

\vspace{0.25cm}

\noindent {\bf (A1)}$\equiv$ {\bf (A3)} By differentiating $f_1(\bz) = f(\bx,g(\bx))$, we obtain
\begin{align}
\left .\frac{\partial^2f_1}{\partial\bx^2}\right|_{\bx^*} &= \left.\frac{\partial^2f}{\partial\bx^2}\right|_{\bx^*,\by^*} +
\left.\frac{\partial^2f}{\partial\bx\partial\by}\right|_{\bx^*,\by^*} \cdot\bD g(\bx^*)\\
& = \bH_{\bx\bx} -\bH_{\bx\by}\bH_{\by\by}^{-1}\bH_{\bx\by}^{\sT}\, ,
\end{align}
where in the last line we used Eq.~(\ref{eq:JacobianG}). Hence $\left .\frac{\partial^2f_1}{\partial\bx^2}\right|_{\bx^*} \succ \bzero$
if and only if $\bH_{\bx\bx} \succ \bH_{\bx\by}\bH_{\by\by}^{-1}\bH_{\bx\by}^\sT$ which, by Schur's complement formula is equivalent to
$\bH\succ \bzero$. Further, since $f\in C^2(\reals^{n+d})$,   $\left .\frac{\partial^2f_1}{\partial\bx^2}\right|_{\bx^*} \succ \bzero$
if and only if  $\frac{\partial^2f_1}{\partial\bx^2} \succ \bzero$ in a neighborhood of $\bx^*$.

\vspace{0.25cm}

\noindent {\bf (B1)}$\equiv$ {\bf (B2)} Linearizing the iteration (\ref{eq:alternatemin_dual}), we get 
that $(\bx^*,\by^*)$ is a stable fixed point if and only if the operator 
\begin{align}
\bL^{-1} = \bD g(\bx^*)^{-1} \bD h(\by^*)^{-1}=(\bH_{\bx\by}^\sT)^{-1}\bH_{\by\by}\bH_{\bx\by}^{-1}\bH_{\bx\bx}
\end{align}
has spectral radius
\begin{align}
\sigma(\bL^{-1}) \equiv \max_{i\le n}\left|\lambda_i\left(\bL^{-1}\right)\right| < 1.
\end{align}
Using the fact that $\bH_{\bx\bx}\succ \bzero$, we have that $\sigma(\bL^{-1})<1$ if and only if
\begin{align}
\bH_{\bx\bx}^{1/2}(\bH^\sT_{\bx\by})^{-1}\bH_{\by\by}\bH_{\bx\by}^{-1}\bH_{\bx\bx}^{1/2} \prec \id_n \iff 
(\bH_{\bx\by}^{\sT})^{-1}\bH_{\by\by}\bH_{\bx\by}^{-1} \prec \bH^{-1}_{\bx\bx} \iff \bH_{\bx\bx} - \bH_{\bx\by}\bH_{\by\by}^{-1}\bH_{\bx\by}^\sT\prec \bzero.
\end{align}
As shown above, the last condition is equivalent to $\left.\frac{\partial^2 f_1}{\partial\bx^2}\right|_{\bx^*}\prec \bzero$,
and by continuity of the Hessian, this is equivalent to $f_1$ being strongly concave in a neighborhood of $\bx^*$.
\end{proof}

\section{Proof of Proposition \ref{propo:Toy}}

It is useful to first prove a simple random matrix theory remark.
\begin{lemma}\label{lemma:Submatrix}
For $S\subseteq [n]$, let $\bX_{S,S}$ be the submatrix of $\bX$ with rows and columns with index in $S$. Then, for any $\eps\in [0,1)$, 
the following holds with high probability:
\begin{align}
\min\big\{\lambda_{\max}(\bX_{S,S}):\, |S|\ge n(1-\eps)\big\}  \ge 2\sqrt{1-\eps}-o_n(1)\, .
\end{align}
\end{lemma}
\begin{proof}
Without loss of generality we can assume $\bX\sim\GOE(n)$ (because the rank-one deformation cannot decrease the maximum eigenvalue),
and $|S|=n(1-\eps)$ (because  $\lambda_{\max}(\bX_{S,S})$ is non-decreasing in $S$). Note that $\bX_{S,S}$ is distributed as $\sqrt{1-\eps}$ times a 
$\GOE(n(1-\eps))$ matrix. Large deviation bounds on the eigenvalues of $\GOE$ matrices imply that, for any $\delta>0$, there exists $c(\delta)>0$ such
that
\begin{align}
\prob\big(\lambda_{\max}(\bX_{S,S})\le 2\sqrt{1-\eps}-\delta\big)\le 2\, e^{-c(\delta)n^2} \, ,
\end{align}
for all $n$ large enough. The claim follows by union bound since there is at most $2^n$ such sets $S$.
\end{proof}

\begin{proof}[Proof of Proposition \ref{propo:Toy}]
First notice that Lemma \ref{lemma:Submatrix} continues to hold if $\bX$ is replaced by $\bX_0$ since
$\|\bX_{S,S}-(\bX_0)_{S,S}\|_{\op}\le \max_{i\le n}|X_{ii}|\le 4\sqrt{\log n/n}$ (where the last bound holds with high probability 
since $(X_{ii})_{i\le n}\sim\normal(0,2/n)$.

Note that $\nabla\cF(\bm)_i=\pm \infty$ if $m_i= \pm 1$, whence any local minimum must be in the interior of $[-1,+1]^n$.
Let $\bm\in (-1,-1)^n$ be a local minimum of $\cF(\,\cdot\,)$. By the second-order minimality conditions,  we must have
\begin{align}
\nabla^2\cF(\bm) = -\lambda\bX_0+ \diag\left((1-m_i^2)^{-1}_{i\le n}\right)\succeq \bzero\, .
\end{align}
Denote by $m_{(1)}$, $m_{(2)}$, $\dots$ the entries of $\bm$ ordered by decreasing absolute value, and let
$S_{\ell}$ be the set of indices corresponding to entries $m_{(\ell+1)},\dots, m_{(n)}$. Finally let $\bv^{(\ell)}\in\reals^n$ be the eigenvector corresponding
to the largest eigenvalue of $(\bX_0)_{S_{\ell},S_{\ell}}$ (extended with zeros outside $S_{\ell}$). We then have, for $\ell=n\eps$
\begin{align}
0&\le \<\bv^{(\ell)},\nabla^2\cF(\bm) \bv^{(\ell)}\> \\
&= -\lambda\cdot\lambda_{\max}\big((\bX_0)_{S_{\ell},S_{\ell}}\big)+ \sum_{i\in S_{\ell}}\frac{(v^{(\ell)}_i)^2}{1-m_i^2}\\
&\le -2\lambda\sqrt{1-\eps} +\frac{1}{1-m_{(n\eps)}^2}+o_n(1)\, .
\end{align}
The last inequality holds with high probability by Lemma \ref{lemma:Submatrix}. Inverting it, we get
\begin{align}
m^2_{(n\eps)}\ge 1-\frac{1}{2\lambda\sqrt{1-\eps}}-o_n(1)\, ,
\end{align}
and therefore 
\begin{align}
\frac{1}{n}\|\bm\|_2^2\ge \eps\left(1-\frac{1}{2\lambda\sqrt{1-\eps}}\right)-o_n(1).
\end{align}
The claim follows by taking $\eps=c_1$ a small constant (for which the right-hand side is lower bounded by $c_0$ for all $\lambda\ge 1$),
or $\eps = c_2(2\lambda-1)$ (for which the right-hand side is lower bounded by $c_0(2\lambda-1)^2$).
\end{proof}

\section{Information-theoretic limits}
\label{app:IT}

\subsection{Proof of Lemma \ref{lemma:IT-Threshold-Z2}}
\label{app:LemmaITZ2}

Let $\hbQ:\reals^{n\times n}\mapsto\reals^{n\times n}$, $\bX\mapsto \hbQ(\bX)$ be any estimator of $\bsigma\bsigma^{\sT}$.
By \cite[Theorem 1.6]{deshpande2017asymptotic}, for $\lambda\in [0,1]$, 
\begin{align}
\lim\inf_{n\to\infty}\frac{1}{n^2}\E\Big\{\big\|\bsigma\bsigma^{\sT}-\hbQ(\bX)\big\|_F^2\Big\} \ge  1\, .
\end{align}
Given $\hbsigma:\reals^{n\times n}\to\reals^n\setminus \{\bzero\}$, set 
\begin{align}
\hbQ(\bX) = c\, \frac{\hbsigma(\bX)\hbsigma(\bX)^{\sT}}{\|\hbsigma(\bX)\|_2^2}
\, ,\;\;\;\;\; c = \E\left(\frac{\<\hbsigma(\bX),\bsigma\>^2}{\|\hbsigma(\bX)\|_2^2}\right)\, .
\end{align}
By a simple calculation
\begin{align}
1-o_n(1)\le \frac{1}{n^2}\E\Big\{\big\|\bsigma\bsigma^{\sT}-\hbQ(\bX)\big\|_F^2\Big\} = 1-\E\left(\frac{\<\hbsigma(\bX),\bsigma\>^2}{\|\hbsigma(\bX)\|_2^2}\right)^2
\, ,
\end{align}
which obviously implies the claim.

\subsection{Proof of Proposition \ref{propo:Bayes}}
\label{app:ProofBayes}

We begin by providing the expression for the free energy functional $\RS(\bM;k,\delta,\nu)$ of Theorem \ref{thm:IT_Limit},
which is obtained by specializing the expression in \cite{miolane2017fundamental}.
Recall the functions $\phi(\,\cdots\, )$, $\tphi(\,\cdots\, )$, introduced in Eq.~(\ref{eq:densityforms_main}). 
We then define  a function $\RS_0(\,\cdot\,,\,\cdot\,;k,\delta,\nu):\bbS_k\times\bbS_k\to\reals$ by
\begin{align}
\RS_0(\bM,\tbM;k,\delta,\nu)&= \frac{\beta\delta(\nu+1)}{k\nu+1}+
\frac{1}{2\beta}\<\bM,\tbM\>\label{eq:PsiGeneral}\\
&-\E\,\phi(\bM\bh+\bM^{1/2}\bz;\bM)-\delta\,\E\, \tphi(\tbM\bw+\tbM^{1/2}\bz;\tbM)\, ,\nonumber
\end{align}
where expectations are with respect to $\bz\sim\normal(0,\id_k)$ independent of $\bh\sim\normal(0,\id_k)$ and $\bw\sim\Dir(\nu;k)$.
We then have
\begin{align}
\RS(\bM;k,\delta,\nu) = \sup_{\tbM\in\bbS_k}\RS_0(\bM,\tbM;k,\delta,\nu)\, .
\end{align}
Further, the function  $\RS_0(\bM,\tbM;k,\delta,\nu)$  on Eq.~(\ref{eq:PsiGeneral}) is separately strictly concave in $\bM$ and $\tbM$,
and in particular the last supremum is uniquely achieved at a point $\tbM = \tbM(\bM)$.

A simple calculation shows that
\begin{align}
\frac{\partial\RS_0}{\partial\bM}(\bM,\tbM;k,\delta,\nu) & = \frac{1}{2\beta}\left\{\tbM - \E\Big\{\sF(\bM\bh+\bM^{1/2}\bz;\bM)^{\otimes 2}\Big\}\right\}\, ,
\label{eq:PdPsi1}\\
\frac{\partial\RS_0}{\partial\tbM}(\bM,\tbM;k,\delta,\nu) & = \frac{1}{2\beta}\left\{\bM - \delta\E\Big\{\tsF(\tbM\bw+\tbM^{1/2}\bz;\tbM)^{\otimes 2}\Big\}\right\}\, .
\label{eq:PdPsi2}
\end{align}
By Lemma \ref{lemma:UsefulFormulae}, for $\bM = a\bJ_k$, $\tbM = b\bJ_k$, we have 
\begin{align}
\frac{\partial\RS_0}{\partial\bM}(\bM,\tbM;k,\delta,\nu) & = \frac{1}{2\beta}\left\{b\bJ_k- \frac{\beta a}{1+ka}\bJ_k\right\}\, ,\\
\frac{\partial\RS_0}{\partial\tbM}(\bM,\tbM;k,\delta,\nu) & = \frac{1}{2\beta}\left\{a\bJ_k - \frac{\beta\delta}{k^2}\bJ_k\right\}\, .
\end{align}
Therefore, this is a stationary point of $\RS_0$ provided $a = \beta\delta/k^2$ and $b=\beta^2\delta/(k(k+\beta\delta))$
(in particular, $\bM = \bM^*$).
Since $\RS(\bM;k,\delta,\nu) = \RS_0(\bM,\tbM(\bM);k,\delta,\nu)$, for $\tbM(\,\cdot\, )$ a differentiable function, it also
follows that $\bM_*$ is a stationary point of $\RS$.

In order to prove that $\bM^*$ is a local minimum of $\RS$ for $\beta<\beta_{\sp}$, we apply Lemma \ref{lemma:hessianstable}
to the function $f(\bx,\by)= -\RS_0(\bx ,\by;k,\delta,\nu)$, whence $f_1(\bx) = -\RS(\bx;k,\delta,\nu)$. It follows from Eqs. ~(\ref{eq:PdPsi1})
and (\ref{eq:PdPsi2}) that the dynamics (\ref{eq:alternatemin_dual}) then coincides with the 
state evolution dynamics discussed in Section \ref{sec:StateEvol}, namely
\begin{align}
\bM_{t+1} & = \delta\,  \E\Big\{\tsF(\tbM_t\bw+\tbM_t^{1/2}\bz;\tbM_t)^{\otimes 2}\Big\}\, ,\\
\tbM_{t} & = \E\Big\{\sF(\bM_t\bh+\bM_t^{1/2}\bz;\bM_t)^{\otimes 2}\Big\}\, . 
\end{align}
Hence, the claim follows immediately from Theorem \ref{thm:StateEvolStable} and Lemma \ref{lemma:hessianstable}.

Finally, we prove that Eq.~(\ref{eq:TrivialEst}) holds for $\beta<\beta_{\sBayes}$. Note that the estimator
$\hbF_n(\bX)$ that minimizes the left-hand side is $\hbF_n(\bX) = \E\{\bW\bH^{\sT}|\bX\}$.
By \cite[Proposition 29]{miolane2017fundamental}, for $\beta<\beta_{\sBayes}$,
\begin{align}
\lim_{n\to\infty}\frac{1}{nd}\E\left\{\left\|\bW\bH^{\sT}-\E\{\bW\bH^{\sT}|\bX\}\right\|_F^2\right\} &= 
\lim_{n\to\infty}\frac{1}{nd}\E\left\{\left\|\bW\bH^{\sT}\right\|_F^2\right\}-\frac{1}{\beta^2\delta}\Tr(\bM^*\tbM^*)\\
&= \lim_{n\to\infty}\frac{1}{nd}\E\left\{\left\|\bW\bH^{\sT}\right\|_F^2\right\}-\frac{\beta\delta}{k(\beta\delta+k)}\, .\label{eq:CBA}
\end{align}
On the other hand,
\begin{align}
\lim_{n\to\infty}\frac{1}{nd}\E\left\{\left\|\bW\bH^{\sT} -c\bfone_n(\bX^{\sT}\bfone_n)^\sT\right\|_F^2\right\}& = \lim_{n\to\infty}\frac{1}{nd}\E\left\{\left\|\bW\bH^{\sT}\right\|_F^2\right\}-2c\, A+c^2 B\, . \label{eq:ABC}
\end{align}
Here, we defined $A$ via
\begin{align}
A &\equiv\lim_{n\to\infty}\frac{1}{nd}\E\Tr\Big(\bH\bW^{\sT}\bfone_n(\bX^{\sT}\bfone_n)^\sT\Big)\\
& = \lim_{n\to\infty}\frac{\sqrt{\beta}}{nd^2}\E\Tr\Big(\bW^{\sT}\bfone_n\bfone_n^{\sT}\bW\bH^{\sT}\bH\Big)\\
& = \sqrt{\beta} \delta \Tr\Big(\frac{\bfone_k}{k}\frac{\bfone_k^{\sT}}{k}\id_k\Big) = \frac{\sqrt{\beta}\delta}{k}\, ,\
\end{align}
(where we used $\bW^{\sT}\bfone_n/n\to \bfone_k/k$ and $\bH^{\sT}\bH/d\to \id_k$ by the law of large numbers)
and
\begin{align}
B &\equiv \lim_{n\to\infty} \frac{1}{nd}\E\Tr\big(\bfone_n(\bX^{\sT}\bfone_n)^{\sT}(\bX^{\sT}\bfone_n)\bfone_n\big)\\
&= \lim_{n\to\infty} \frac{1}{d}\E\big\<\bfone_n,\bX\bX^{\sT}\bfone_n\big\>\\
& = \lim_{n\to\infty} \frac{1}{d}\E\left\{\frac{\beta}{d^2}\Tr\big((\bW^{\sT}\bfone_n)^{\sT}\bH^{\sT}\bH(\bW^{\sT}\bfone_n)\big)+n\right\}\\
& = \beta\delta^2 \Tr\Big(\frac{\bfone_k^{\sT}}{k}\frac{\bfone_k}{k}\id_k\big)+\delta = \frac{\beta\delta^2}{k}+\delta\, .
\end{align}
Setting $c=A/B$, and substituting in Eq.~(\ref{eq:ABC}), we obtain
\begin{align}
\lim_{n\to\infty}\frac{1}{nd}\E\left\{\left\|\bW\bH^{\sT} -c\bfone_n(\bX^{\sT}\bfone_n)^\sT\right\|_F^2\right\}& = 
\lim_{n\to\infty}\frac{1}{nd}\E\left\{\left\|\bW\bH^{\sT}\right\|_F^2\right\}-\frac{\beta\delta}{k(\beta\delta+k)}\, ,
\end{align}
which coincides with Eq.~(\ref{eq:CBA}) as claimed.

\section{Naive Mean Field: Analytical results}

\subsection{Preliminary definitions}

The functions $\sF, \tsF:\reals^k\times \reals^{k\times k}\to\reals^k$ are defined in Eq.~(\ref{eq:defF_main}). Explicitly
\begin{align}
\label{eq:defF}
\sF(\by; \bQ) &\equiv\sqrt{\beta}\, \frac{\int \bh \, e^{\<\by,\bh\>-\<\bh,\bQ\bh\>/2} \, q_0(\de \bh)}{\int\, e^{\<\by,\bh\>-\<\bh,\bQ\bh\>/2} \,
              q_0(\de \bh)}\, ,\\
\label{eq:deftF}
\tsF(\hby; \tbQ) &\equiv \sqrt{\beta}\, \frac{\int \bw \, e^{\<\hby,\bw\>-\<\bw,\tbQ\bw\>/2} \, \tq_0(\de \bw)}{\int \, e^{\<\hby,\bw\>-\<\bw,\tbQ\bw\>/2} \,
              \tq_0(\de \bw)}\, ,
\end{align}
where $q_0(\,\cdot\, )$ is the prior distribution of the rows of $\bH$, and $\tq_0(\,\cdot\,)$ is the prior distribution of the rows of $\bW$. 

For $\bQ$ positive semidefinite and symmetric,
$\sF(\by;\bQ)/\sqrt{\beta}$ can be interpreted as the posterior  expectation of $\bh\sim q_0(\,\cdot\,)$, given observations $\by = \bQ\bh+\bQ^{1/2}\bz$, 
where $\bz\sim\normal(0,\id_k)$, and analogously for $\tsF(\hby;\tbQ)$. Explicitly
\begin{align}
\sF(\by; \bQ) = \sqrt{\beta}\,\E\Big\{ \bh \Big|\; \bQ\bh +\bQ^{1/2}\bz = \by\Big\}\, ,\;\;\;\;\;\;\;\;
\tsF(\hby; \tbQ) = \sqrt{\beta}\,\E\Big\{ \bw \Big|\; \tbQ\bw +\tbQ^{1/2}\bz = \hby\Big\}\, .
\end{align}

In our specific application $q_0(\,\cdot\,)$ is $\normal(0,\id_k)$, and  $\tq_0(\,\cdot\,)$ is $\Dir(\nu;k)$, namely
\begin{align}
\label{eq:initmeasures}
q_0(\de\bh) = \frac{1}{(2\pi)^{k/2}}\, e^{-\|\bh\|_2^2/2}\de\bh\, ,\;\;\;\;\;\;\;\;
\tq_0(\de\bw) = \frac{1}{Z(\nu;k)} \prod_{i=1}^kw_i^{\nu-1} \, \oq(\de\bw)\, ,
\end{align}
where $\oq(\,\cdot\,)$ is the uniform measure over the simplex $\sP_1(k) = \{\bw\in\reals^k_{\ge 0}\; :\;\;\<\bw,\bfone_k\> =1\}$. 
In particular, $\sF(\by;\bQ)$ can be computed explicitly, yielding
\begin{align}
\label{eq:sFexplicit}
\sF(\by;\bQ) = \sqrt{\beta}(\id_k+\bQ)^{-1}\by\, .
\end{align}

We also define the second moment functions $\sG, \tsG:\reals^k\times \reals^{k\times k}\to \reals^{k\times k}$ by
\begin{align}
\label{eq:defG}
     \sG(\by; \bQ) &\equiv {\beta}\, \frac{\int \bh^{\otimes 2} \, e^{\<\by,\bh\>-\<\bh,\bQ\bh\>/2} \, q_0(\de \bh)}{\int \, e^{\<\by,\bh\>-\<\bh,\bQ\bh\>/2} \,
        q_0(\de \bh)}\,,\\
        \label{eq:deftG}
     \tsG(\tby; \tbQ) &\equiv {\beta}\, \frac{\int \bw^{\otimes 2} \, e^{\<\tby,\bw\>-\<\bw,\tbQ\bw\>/2} \, \tq_0(\de \bw)}{\int \, e^{\<\tby,\bw\>-\<\bw,\tbQ\bw\>/2} \,
         \tq_0(\de \bw)}\,.
\end{align}
Again, $\sG(\,\cdots\,)$ can be written explicitly as
\begin{align}
\label{eq:sGexplicit}
 \sG(\by; \bQ) & = \beta\Big\{(\id_k+\bQ)^{-1}\by\by^{\sT}(\id_k+\bQ)^{-1}+(\id_k+\bQ)^{-1}\Big\}\, .
\end{align}

\subsection{Derivation of the iteration (\ref{eq:NMF1_Main}), (\ref{eq:NMF2_Main})}
\label{app:NMF_ansatz}

Let $\mathcal D$, the set of joint distributions
$\hat{q}\left(\bW,\bH\right)$ that factorize over the rows of $\bW, \bH$, namely 
\begin{align}
\hat{q}\left(\bW,\bH\right) = q\left(\bH\right)\tilde q\left(\bW\right) = \prod_{i=1}^d q_i\left(\bh_i\right)\prod_{a=1}^n\tilde q_a\left(\bw_a\right)\, .
\end{align}
The goal in variational inference is to find the distribution in $\cD$  that minimizes the Kullback-Leibler  (KL) divergence with respect to the 
actual posterior distribution of $\bX,\bW$ given $\bX$
\begin{align}
\hat{q}^*\left(\,\cdot\,,\,\cdot\,\right) = \arg\min_{\hq\in\cD}\KL\left(\hat{q}\left(\,\cdot\,,\,\cdot\,\right)||\;p\left(\, \cdot\, ,\,\cdot\,|\bX\right)\right)
\end{align}
The KL divergence can also be written as (denoting by $\E_{\hq}$ expectation over $(\bW,\bH)\sim \hq(\,\cdot\,,\,\cdot\,)$)
\begin{align}
\label{eq:KL}
\KL\left(\hat{q}\left(\,\cdot\,,\,\cdot\,\right)||\;p\left(\,\cdot\,,\,\cdot\,|\bX\right)\right) &= \E_{\hq}\left[\log \hat{q}\left(\bW,\bH\right) \right]
- \E_{\hq}\left[\log p\left(\bX,\bW,\bH\right)\right] + \log p\left(\bX\right)\\
& \equiv \cF(\hq)+ \log p\left(\bX\right)\, .
\end{align}
The function $\cF(\hq)$ is known as Gibbs free energy or --within the topic models literature-- as the 
opposite of the evidence lower bound $\cF(\hq) = -\ELBO(\hq)$ \cite{blei2017variational}. Since $\log p\left(\bX\right)$ does not depend on $\hq$,
minimizing the KL divergence is equivalent to minimizing the Gibbs free energy. 

In order to find $\hat{q}^*\left(\bW,\bH\right) = q^*\left(\bH\right)\tilde q^*\left(\bW\right)$, the naive 
mean field iteration minimizes the Gibbs free energy by alternating minimization: we minimize the Gibbs free energy over $q\left(\bH\right)$ (while keeping 
$\tilde q\left(\bW\right)$ fixed), then minimize over $\tilde q\left(\bW\right)$ (while keeping $q\left(\bH\right)$ fixed), and repeat.
With a slight abuse of notation, we will write  $\cF(\hq)=\cF(q,\tq)$.
Note that if we keep $\tq\left(\bW\right)$ fixed, we have
\begin{align}
\arg\min _{q}\cF(q,\tq) &= 
\arg\min _{q}\left\{\E_{q(\bH)}\left[\log q\left(\bH\right)\right] - \E_{q(\bH)}\left[\E_{\tilde q\left(\bW\right)}\left[\log p\left(\bX,\bW,\bH\right)\right]\right]\right\}\nonumber\\
&= \arg\min _{q} {\KL}\left(q\left(\bH\right)||\, C\exp\left\{\E_{\tilde q(\bW)}\left[\log p\left(\bX,\bW,\bH\right)\right]\right\}\right)\nonumber\\
&\propto \exp\left\{\E_{\tilde q(\bW)}\left[\log p\left(\bX,\bW,\bH\right)\right]\right\}\, .
\end{align}
Similarly, by taking $q\left(\bH\right)$ fixed, we have
\begin{align}
\arg\min _{\tilde q}\cF(q,\tq) \propto \exp\left\{\E_{q\left(\bH\right)}\left[\log p\left(\bX,\bW,\bH\right)\right] \right\}.
\end{align}
Therefore, the naive mean field iterations have the form
\begin{align}
\label{eq:densityevol}
\begin{split}
&q^{t+1}\left(\bH\right) = \prod_{i=1}^d q_i^{t+1}\left(\bh_i\right) \propto \exp\left\{\E_{\tilde q^t\left(\bW\right)}\left[\log p\left(\bX,\bW,\bH\right)\right]\right\},\\
&\tilde q^{t}\left(\bW\right) = \prod_{a=1}^n \tilde q_a^{t}\left(\bw_a\right) \propto\exp\left\{ \E_{q^t\left(\bH\right)}\left[\log p\left(\bX,\bW,\bH\right)\right] \right\}.
\end{split}
\end{align}
with initialization
\begin{align}
q^0\left(\bH\right) = \prod_{i=1}^d q_0\left(\bh_i\right)\,,\;\;\;\;\;\tilde q^0\left(\bW\right) = \prod_{a=1}^n \tilde q_0\left(\bw_a\right)
\end{align}
where $q_0\left(\bh_i\right)$, $\tilde q_0\left(\bw_a\right)$ are the prior distributions on the rows of $\bH$ and $\bW$,
cf. Eq.~\eqref{eq:initmeasures}. Note that 
the iterations in \eqref{eq:densityevol} can be further simplified by noting that the densities $q_i^{t}$ and
$\tilde q_i^t$ have the form
\begin{align}
\label{eq:densityforms}
\begin{split}
&q_i^t\left(\bh\right) \propto \exp\left\{\left\langle\bm_i^t,\bh\right\rangle-\frac{1}{2}\left\langle\bh, \bQ^t\bh\right\rangle\right\}q_0\left(\bh\right),\\
&\tilde q_a^t\left(\bw\right) \propto \exp\left\{\left\langle\tbm_a^t,\bw\right\rangle-\frac{1}{2}\left\langle\bw, \tbQ^t\bw\right\rangle\right\}\tilde q_0\left(\bw\right).
\end{split}
\end{align}
In order to see this, note that the initial densities $q_0\left(\bh\right)$, $\tilde q_0\left(\bw\right)$ are in the form 
\eqref{eq:densityforms}. Further, if we assume that $q_i^t\left(\bh\right)$, $\tilde q_a^t\left(\bw\right)$
are in the form \eqref{eq:densityforms}, using the update equations \eqref{eq:densityevol}, we have
\begin{align}
q^{t+1}\left(\bH\right) = \prod_{i=1}^d q_{i}^{t+1}\left(\bh_i\right) &\propto \exp\left\{\E_{\tilde q^t\left(\bW\right)}\log p\left(\bX,\bH, \bW\right)\right\}\\
& \propto \exp\left\{\E_{\tilde q^t\left(\bW\right)}\log p \left(\bH, \bX|\bW\right)\right\}\\
& \propto q_0\left(\bH\right) \exp\left\{\E_{\tilde q^t\left(\bW\right)}\log p\left(\bX|\bH, \bW\right)\right\}\\
& \propto q_0\left(\bH\right)\exp\left\{-\E_{\tilde q^t\left(\bW\right)}\left[\frac{d}{2}\left\|\bX - \frac{\sqrt{\beta}}{d}\bW\bH^\sT\right\|_F^2\right]\right\}\\
&\propto q_0\left(\bH\right)\exp\left\{\E_{\tilde q^t\left(\bW\right)}\Tr\left(\sqrt{\beta}\bX\bH\bW^\sT - \frac{\beta}{2d}\bW\bH^\sT\bH\bW^\sT\right)\right\}\\
&= q_0\left(\bH\right)\exp\left\{\E_{\tilde q^t\left(\bW\right)}\sum_{a=1}^n\left(\sqrt{\beta}\left\langle\bx_a,\bH\bw_a\right\rangle-\frac{\beta}{2d}\left\langle\bw_a,\bH^\sT\bH\bw_a\right\rangle\right)\right\}\\
&= q_0\left(\bH\right)\exp\left\{\sum_{a=1}^n\left\langle\bx_a, \bH\tsF\left(\tbm_a^t; \tbQ^t\right)\right\rangle - \frac{1}{2d}\left\langle\bH^\sT\bH, \sum_{a=1}^n\tsG\left(\tbm_a^t; \tbQ_t\right)\right\rangle\right\}\\
&= \prod_{i=1}^d \left(q_0\left(\bh_i\right)\exp\left\{\left\langle\bm_i^{t+1},\bh_i\right\rangle-\frac{1}{2}\left\langle\bh_i, \bQ^{t+1}\bh_i\right\rangle\right\}\right)
\end{align}
where $\tsF(\,\cdot\, ;\,\cdot\, ), \tsG(\,\cdot\,;\,\cdot\,)$ are given in \eqref{eq:deftF}, \eqref{eq:deftG} and
\begin{align}
\label{eq:densityevolH}
\begin{split}
&\bm^{t+1} = \bX^\sT\tsF\left(\tbm^t; \tbQ^t\right),\\
&\bQ^{t+1} = \frac{1}{d}\sum_{a=1}^n \tsG\left(\tbm_a^t; \tbQ^t\right).
\end{split}
\end{align}
Therefore, $q_i^{t+1}\left(\bh\right)$ has the form in \eqref{eq:densityforms} and the update formula
for $\bm^{t+1}$, $\bQ^{t+1}$ are given in \eqref{eq:densityevolH}. Similarly, for $\tq^{t+1}\left(\bW\right)$ we have
\begin{align}
\tq^{t+1}\left(\bW\right) = \prod_{a=1}^n \tq_{a}^{t+1}\left(\bw_a\right) &\propto \exp\left\{\E_{q^{t+1}\left(\bH\right)}\log p\left(\bX, \bH,\bW\right)\right\}\\
& \propto \exp\left\{\E_{ q^{t+1}\left(\bH\right)}\log p \left(\bW, \bX|\bH\right)\right\}\\
& = \tq_0\left(\bW\right) \exp\left\{\E_{q^{t+1}\left(\bH\right)}\log p\left(\bX|\bH, \bW\right)\right\}\\
& \propto \tq_0\left(\bW\right)\exp\left\{\E_{q^{t+1}\left(\bH\right)}\left[-\frac{d}{2}\left\|\bX - \frac{\sqrt{\beta}}{d}\bW\bH^\sT\right\|_F^2\right]\right\}\\
&\propto \tq_0\left(\bW\right)\exp\left\{\E_{q^{t+1}\left(\bH\right)}\Tr\left(\sqrt{\beta}\bW\bH^\sT\bX^\sT - \frac{\beta}{2d}\bW\bH^\sT\bH\bW^\sT\right)\right\}
\end{align}
Hence,
\begin{align}
\tq^{t+1}\left(\bW\right) &\propto \tq_0\left(\bW\right)\exp\left\{\E_{ q^{t+1}\left(\bH\right)}\sum_{a=1}^n\left(\sqrt{\beta}\left\langle\bw_a,\bx_a\bH\right\rangle-\frac{\beta}{2d}\left\langle\bw_a,\bH^\sT\bH\bw_a\right\rangle\right)\right\}\\
&= \tq_0\left(\bW\right)\exp\left\{\sum_{a=1}^n\left\langle\bw_a, \bx_a\sF\left(\bm^{t+1}; \bQ^{t+1}\right)\right\rangle - \frac{1}{2d}\left\langle\bw_a,\left(\sum_{i=1}^d \sG\left(\bm_i^{t+1}; \bQ^{t+1}\right)\right)\bw_a\right\rangle\right\}\\
&= \prod_{a=1}^n \left(\tq_0\left(\bw_a\right)\exp\left\{\left\langle\bw_a, \tbm_a^{t+1}\right\rangle-\frac{1}{2}\left\langle\bw_a, \tbQ^{t+1}\bw_a\right\rangle\right\}\right)
\end{align}
where $\sF(\,\cdot\,;\,\cdot\,), \sG(\,\cdot\,;\,\cdot\,)$ are given in \eqref{eq:defF}, \eqref{eq:defG} and
\begin{align}
\label{eq:densityevolW}
\begin{split}
&\tbm^{t+1} = \bX\sF\left(\bm^{t+1}; \bQ^{t+1}\right),\\
&\tbQ^{t+1} = \frac{1}{d}\sum_{i=1}^d \sG\left(\bm_i^{t+1}; \bQ^{t+1}\right).
\end{split}
\end{align}
Therefore, $\tq_a^{t+1}\left(\bw\right)$ has the form in \eqref{eq:densityforms} and the update formula
for $\tbm^{t+1}$, $\tbQ^{t+1}$ are given in \eqref{eq:densityevolW}.

\subsection{Derivation of the variational free energy (\ref{eq:FreeEnergy_main})}
\label{app:NMF_Free_Energy}

\def\mw{\tsF(\tbm;\tbQ)}
\def\mh{\sF(\bm;\bQ)}
\def\vh{\sG(\bm^t_{i};\bQ_t)}
\def\vw{\tsG(\tbm^t_{a};\tbQ_t)}

As already mentioned, naive mean field minimizes the 
KL divergence between a factorized distribution $\hat q(\bW, \bH) = \prod_{a=1}^{n} \tq(\bw_a) \prod_{i = 1}^{d} q(\bh_i)$
and the real posterior $p(\bW,\bH|\bX)$. The KL divergence takes the form
\begin{align}
\KL(\hat q(\,\cdot\,, \,\cdot\, )||p(\,\cdot\,,\,\cdot\,|\bX)) = \cF(\hq) +\log p(\bX) +\frac{d}{2}\|\bX\|_{F}^2\, ,
\end{align}
where $\cF(\hq)$ is the Gibbs free energy.
In this appendix we derive an explicit form for $\cF(\hq)$ when $\hq$ is factorized.
We have
\begin{align}
\cF(\hq) &= \E_{\hat q}[-\log p(\bW, \bH| \bX)] +  \E_{\hat q}[\log \hat q(\bW,\bH)] -\frac{d}{2}\|\bX\|_{F}^2\\
&= \E_{\hat q}[-\log p(\bW, \bH, \bX)] +  \E_{\hat q}[\log \hat q(\bW,\bH)] -\frac{d}{2}\|\bX\|_{F}^2\\
&= \E_{\hat q}[-\log p(\bX|\bW, \bH) -\log p(\bW, \bH)] +  \E_{\hat q}[\log \hat q(\bW,\bH)] -\frac{d}{2}\|\bX\|_{F}^2\\
&= \E_{\hat q}\left[\frac{d\Vert \bX - \frac{\sqrt{\beta}}{d}\bW\bH^\sT \Vert_F^2}{2}-\frac{d}{2}\|\bX\|_{F}^2 -\log (p(\bW, \bH))\right] +  \E_{\hat q}[\log \hat q(\bW,\bH)] \\
&= \frac{d}{2}\E_{\hat q}\left[\Vert \bX - \frac{\sqrt{\beta}}{d}\bW\bH^\sT \Vert_F^2\right] -\frac{d}{2}\|\bX\|_{F}^2+  
\KL(\hat q(\,\cdot\,, \,\cdot\,)\| q_0(\,\cdot\,, \,\cdot\,))\, .
\end{align}
(The last term is the KL divergence between $\hq$ and the prior.)

We can explicitly calculate each term. Let's denote by $\br_i,\bOmega_i$ the first and second moments of $q_i$ and by
$\tbr_a$, $\tbQ_a$ the first and second moments of $\tq_a$:
\begin{align}
\br_i&= \int \bh\,\, q_i(\de \bh) \, ,\;\;\;\;\;
\tbr_a = \int \bw \, \, \tq_a(\de \bw)\, ,\label{eq:Moment1}\\
\bOmega_i&= \int \bh^{\otimes 2} \,\,q_i(\de \bh) \, ,\;\;\;\;\;
\tbOmega_a  = \int \bw^{\otimes 2} \, \, \tq_a(\de \bw)\, . \label{eq:Moment2}
\end{align}
We then have
\begin{align} 
\frac{d}{2}\E_{\hat q}\Vert \bX - \frac{\sqrt{\beta}}{d}\bW\bH^\sT \Vert_F^2 &-\frac{d}{2}\|\bX\|_{F}^2
= \frac{d}{2}\E_{\hat q}\left[\Tr\left(-\frac{2\sqrt{\beta}}{d}\bX^\sT\bW\bH^\sT\right) + \Tr\left(\frac{\beta}{d^2}\bH\bW^\sT\bW\bH^\sT\right)\right] 
\\
&= -\sqrt{\beta} \Tr\left(\bX^\sT\E_{\hat q}[\bW\bH^\sT]\right) + \frac{\beta}{2d} \Tr\left(\E_{\hat q}[\bH\bW^\sT\bW\bH^\sT]\right) \\
&= -\sqrt{\beta} \Tr\left(\bX^\sT\br \tbr^\sT\right) +\frac{\beta}{2d} \sum_{i=1}^d \sum_{a=1}^n\<\bOmega_i,\tbOmega_a\>\,. \label{eq:expected_loglike} 
\end{align}

Since both $\hq$ and $q_0$ have product form, their KL divergence is just a sum of KL divergences for each 
row of $\bW$ and each row of $\bH$:
\begin{align} \label{kl_div_2}
\KL(\hat q(\,\cdot\,, \,\cdot\,)\|q_0(\,\cdot\,, \,\cdot\,))) &= \sum_{i = 1}^{d} \KL(q_i\| q_0) +\sum_{a = 1}^{n} \KL(\tq_a\| q_0)\, .
\end{align}
Each of these terms is treated in the same manner: we minimize over $q_i$ or $\tq_a$ subject to the 
moment constraints (\ref{eq:Moment1}), and define
\begin{align}
\psi_*(\br_i,\bOmega_i) = \min\left\{  \KL(q_i\| q_0)  :\;\;\;  \int \bh\,\, q_i(\de \bh) =\br_i\, ,\;\;
\int \bh^{\otimes 2} \, \, q_i(\de \bh) =\bOmega_i\right\}\, ,\label{eqs:PsiKL1}\\
\tpsi_*(\tbr_a,\tbOmega_a) = \min\left\{  \KL(\tq_a\| \tq_0) :\;\;\; \int \bw\,\, \tq_a(\de \bw)=\tbr_a\, ,\;\;
\int \bw^{\otimes 2} \, \, \tq_a(\de \bw) =\tbOmega_a\right\}\, . \label{eqs:PsiKL2}
\end{align}
 Standard  duality between entropy and moment generating functions yields that
$\psi_*$, $\tpsi_*$ are defined as per Eq.~(\ref{eq:LegendrePhi}). We briefly recall the argument 
for the reader's convenience. Considering for instance $\tpsi_*(\tbr,\tbOmega)$, we introduce the Lagrangian 
\begin{align}
\cL(\tq_a,\tbm_a,\tbQ_a) = \KL(\tq_a\| \tq_0) +\<\tbm_a,\tbr_a\>
-\frac{1}{2}\<\tbQ_a,\tbOmega_a\>-\int \Big\{\<\tbm_a,\bw\>-\frac{1}{2}\<\bw,\tbQ_a\bw\>\Big\}\, \tq_a(\de \bw)\, .
\end{align}
This is minimized easily with respect to $\tq_a$. The minimum is achieved at
the distribution (\ref{eq:densityforms_main}), with 
\begin{align}
\min_{\tq_a}\cL(\tq_a,\tbm_a,\tbQ_a) =
\< \tbr_a, \tbm_a\> -\frac{1}{2}\<\tbOmega_a,\tbQ_a\>- \tphi(\tbm_a, \tbQ_a)\, ,
\end{align}
and the claim (\ref{eq:LegendrePhi}) follows by strong duality.

Putting together Eqs.~(\ref{eq:expected_loglike}), (\ref{kl_div_2}), and (\ref{eqs:PsiKL1})-(\ref{eqs:PsiKL2}), we obtain the desired expression 
(\ref{eq:FreeEnergy_main}).

Using (\ref{eq:LegendrePhi}), we get the following expressions for the gradients of $\psi_*$
\begin{align}
\frac{\partial \psi_*}{\partial \br}(\br,\bOmega) = \bm\, ,\;\;\;\;\;\frac{\partial \psi_*}{\partial \bOmega}(\br,\bOmega) = -\frac{1}{2}\bQ\,,
\end{align}
and similarly for $\tpsi_*$ (where $\bm,\bQ$ are related to $\br,\bOmega$ via Eqs.~(\ref{eq:r_def}), (\ref{eq:r_def})).
 Hence, the gradients of $\cF$ with respect to $\br_i$, $\bOmega_i$ read
\begin{align}
\frac{\partial \cF}{\partial \br_i}(\br,\tbr,\bOmega,\tbOmega)& = -\sqrt{\beta}(\bX^{\sT}\tbr)_{i,\cdot}+\bm_i\, ,\;\;\;\;\;\;\;
\frac{\partial \cF}{\partial \tbr_a}(\br,\tbr,\bOmega,\tbOmega) = -\sqrt{\beta}(\bX\br)_{a,\cdot}+\tbm_a\,,\\
\frac{\partial \cF}{\partial \bOmega_i}(\br,\tbr,\bOmega,\tbOmega) & = -\frac{1}{2}\bQ_i+\frac{\beta}{2d}\sum_{a=1}^n\tbOmega_a\, ,\;\;\;\;\;\;\;
\frac{\partial \cF}{\partial \tbOmega_a}(\br,\tbr,\bOmega,\tbOmega)  = -\frac{1}{2}\bQ_a+\frac{\beta}{2d}\sum_{i=1}^d\bOmega_i
\, .\label{eq:DF_Omega}
\end{align}
Notice that at stationarity points, we have $\bQ_i = \bQ = (\beta/d)\sum_{a=1}^n\tbOmega_a$ independent of $i$.

\subsection{Proof of Lemma  \ref{lemma:Uninf}}
\label{app:Uninformative}

We start with some useful formulae.
\begin{lemma}\label{lemma:UsefulFormulae}
For $q\in\reals$ define $\sE(q)$ by 
\begin{align}
\sE(q;\nu) = \frac{\int w_1^2 e^{-q\|\bw\|_2^2}\, \tq_0(\de\bw)}{\int e^{-q\|\bw\|_2^2}\, \tq_0(\de\bw)}\, . \label{eq:Edef}
\end{align}
Then, we have 
\begin{align}
\sF(\by = y \bfone_k;\bQ =q_1\id_k+q_2\bJ_k) &= \frac{\sqrt{\beta}\, y}{1+q_1+kq_2} \, \bfone_k\, ,\label{eq:sfsimplifiedsymm}\\
\sG(\by = y \bfone_k;\bQ = q_1\id_k+q_2\bJ_k) &= \frac{\beta}{(1+q_1)} \,\id_k+\beta \left\{\frac{y^2}{(1+q_1+kq_2)^2}-\frac{q_1}{(1+q_1)(1+q_1+kq_2)}\right\} \, \bJ_k \, ,\label{eq:sgsimplifiedsymm}\\
\tsF(\tby = \ty \bfone_k;\tbQ = \tq_1\id_k+\tq_2\bJ_k) &= \frac{\sqrt{\beta}}{k} \, \bfone_k\, ,\label{eq:tsfsimplifiedsymm}\\
\tsG(\tby = \ty \bfone_k;\tbQ = \tq_1\id_k+\tq_2\bJ_k) &= \beta \, \frac{k^2\sE(\tq_1;\nu)-1}{k(k-1)}\, \id_k-\beta \, \frac{k\sE(\tq_1;\nu)-1}{k(k-1)}\, \bJ_k \, .\label{eq:tsgsimplifiedsymm}
\end{align}
In particular
\begin{align}
\sF(\by = y \bfone_k;\bQ =q\bJ_k) &= \frac{\sqrt{\beta}\, y}{1+kq} \, \bfone_k\, ,\\
\sG(\by = y \bfone_k;\bQ = q\bJ_k) &= \beta \,\id_k+ \beta \, \frac{y^2}{(1+kq)^2} \, \bJ_k \, ,\\
\tsF(\tby = \ty \bfone_k;\tbQ = \tq\bJ_k) &= \frac{\sqrt{\beta}}{k} \, \bfone_k\, ,\\
\tsG(\tby = \ty \bfone_k;\tbQ = \tq\bJ_k) &= \frac{\beta}{k(k\nu+1)}\, \left(\id_k+\nu\bJ_k\right)\, .
\end{align}
\end{lemma}
\begin{proof}
First note that
\begin{align}
\left[\left(1+q_1\right)\id_k + q_2\bJ_k\right]^{-1} = \frac{1}{1+q_1} \id_k - \frac{q_2}{(1+q_1)(1+q_1+kq_2)}\bJ_k.
\end{align}
Hence, by \eqref{eq:sFexplicit} we have
\begin{align}
\sF(\by = y \bfone_k;\bQ =q_1\id_k+q_2\bJ_k) &= \sqrt{\beta}y\left[\left(1+q_1\right)\id_k + q_2\bJ_k\right]^{-1}\bfone_k\\
&=\sqrt{\beta}y  \left(\frac{1}{1+q_1} \id_k - \frac{q_2}{(1+q_1)(1+q_1+kq_2)}\bJ_k\right) \bfone_k\\
&= \sqrt{\beta}y \left(\frac{1}{1+q_1} - \frac{kq_2}{(1+q_1)(1+q_1+kq_2)}\right)\bfone_k\\
&= \frac{\sqrt{\beta}\, y}{1+q_1+kq_2} \, \bfone_k\,.
\end{align}
Thus, by \eqref{eq:sGexplicit}
\begin{align}
\sG(\by = y \bfone_k;\bQ =q_1\id_k+q_2\bJ_k) &= \frac{\beta\, y^2}{(1+q_1+kq_2)^2} \bJ_k\,
+ \beta\left(\frac{1}{1+q_1} \id_k - \frac{q_2}{(1+q_1)(1+q_1+kq_2)}\bJ_k\right)\\
&=  \frac{\beta}{(1+q_1)} \,\id_k+\beta \left\{\frac{y^2}{(1+q_1+kq_2)^2}-\frac{q_1}{(1+q_1)(1+q_1+kq_2)}\right\} \, \bJ_k \,.
\end{align}
In addition,
using \eqref{eq:deftF},
by symmetry, all entries of $\tsF(\tby = \ty \bfone_k;\tbQ = \tq_1\id_k+\tq_2\bJ_k)$ are equal. Further, 
\begin{align}
\left\langle \bfone_k\,, \tsF(\tby = \ty \bfone_k;\tbQ = \tq_1\id_k+\tq_2\bJ_k)\right\rangle &= 
\sqrt{\beta}\, \frac{\int \left\langle\bfone_k,\bw\right\rangle \, e^{\<\hby,\bw\>-\<\bw,\tbQ\bw\>/2} \, \tq_0(\de \bw)}{\int \, e^{\<\hby,\bw\>-\<\bw,\tbQ\bw\>/2} \,
              \tq_0(\de \bw)}\\
&= \sqrt{\beta}\, \frac{\int e^{\<\hby,\bw\>-\<\bw,\tbQ\bw\>/2} \, \tq_0(\de \bw)}{\int \, e^{\<\hby,\bw\>-\<\bw,\tbQ\bw\>/2} \,
              \tq_0(\de \bw)} = \sqrt{\beta}.
\end{align}
Therefore, 
\begin{align}
\tsF(\tby = \ty \bfone_k;\tbQ = \tq_1\id_k+\tq_2\bJ_k) &= \frac{\sqrt{\beta}}{k} \, \bfone_k.
\end{align}
Finally, again by symmetry, 
$\tsG(\tby = \ty \bfone_k;\tbQ = \tq_1\id_k+\tq_2\bJ_k)$ has the same diagonal entries.
Further, the off-diagonal entries of this matrix are equal. 
Thus, we have
\begin{align}
\tsG(\tby = \ty \bfone_k;\tbQ = \tq_1\id_k+\tq_2\bJ_k) = \left(\tsG_{11}-\tsG_{12}\right)\id_k + \tsG_{12}\bJ_k.\label{eq:diagG}
\end{align}
Note that by \eqref{eq:defG}, \eqref{eq:Edef}
\begin{align}
\tsG_{1,1} &=
\beta\,\frac{\int w_1^2 e^{\ty\left\langle \bw, \bfone_k\right\rangle-\tq_1\|\bw\|_2^2/2-\tq_2\left\langle \bw, \bfone_k\right\rangle^2/2}\, \tq_0(\de\bw)}{\int e^{\ty\left\langle \bw, \bfone_k\right\rangle-\tq_1\|\bw\|_2^2/2-\tq_2\left\langle \bw, \bfone_k\right\rangle^2/2}\tq_0(\de\bw)} \\
&= \beta\,\frac{e^{\ty - \tq_2/2}\int w_1^2 e^{-\tq_1\|\bw\|_2^2/2}\, \tq_0(\de\bw)}{e^{\ty - \tq_2/2}\int e^{-\tq_1\|\bw\|_2^2/2}\tq_0(\de\bw)} = \beta\,\sE(\tq_1;\nu).
\end{align}
Further, by \eqref{eq:defG}
\begin{align}
k\tsG_{1,1} + k(k-1)\tsG_{1,2} &= \left\langle\tsG(\tby = \ty \bfone_k;\tbQ = \tq_1\id_k+\tq_2\bJ_k), \bJ_k\right\rangle \\
&=  {\beta}\, \frac{\int \left\langle\bw, \bfone_k\right\rangle^2 \, e^{\<\tby,\bw\>-\<\bw,\tbQ\bw\>/2} \, \tq_0(\de \bw)}{\int \, e^{\<\tby,\bw\>-\<\bw,\tbQ\bw\>/2} \,
         \tq_0(\de \bw)}\,\\
         &=  {\beta}\, \frac{\int e^{\<\tby,\bw\>-\<\bw,\tbQ\bw\>/2} \, \tq_0(\de \bw)}{\int \, e^{\<\tby,\bw\>-\<\bw,\tbQ\bw\>/2} \,
         \tq_0(\de \bw)}\, = \beta.\label{eq:sumallG}
\end{align}
Therefore, by \eqref{eq:sumallG}, \eqref{eq:diagG}, we get
\begin{align}
\tsG_{1,1} = \beta\,\sE(\tq_1;\nu)\,, \;\;\;\;\; \tsG_{1,2} = -\beta \, \frac{k\sE(\tq_1;\nu)-1}{k(k-1)}.
\end{align}
Hence, 
\begin{align}
\tsG(\tby = \ty \bfone_k;\tbQ = \tq_1\id_k+\tq_2\bJ_k) &= \beta \, \frac{k^2\sE(\tq_1;\nu)-1}{k(k-1)}\, \id_k-\beta \, \frac{k\sE(\tq_1;\nu)-1}{k(k-1)}\, \bJ_k. 
\end{align}
In addition, note that
\begin{align}
\sE(0;\nu) = \int w_1^2 \tq_0(\de\bw) = \frac{\nu+1}{k(k\nu+1)}.
\end{align}
Using this, and replacing $q_1, \tq_1 = 0$ in \eqref{eq:sfsimplifiedsymm} - \eqref{eq:tsgsimplifiedsymm} will complete the proof.
\end{proof}

\begin{proof}[Proof of Lemma  \ref{lemma:Uninf}]
Note that  $q\geq 0$
\begin{align}
k^2\sE(q;\nu) =  \frac{\int k^2w_1^2 e^{-q\|\bw\|_2^2}\, \tq_0(\de\bw)}{\int e^{-q\|\bw\|_2^2}\, \tq_0(\de\bw)} = 
\frac{\int k\|\bw\|_2^2 e^{-q\|\bw\|_2^2}\, \tq_0(\de\bw)}{\int e^{-q\|\bw\|_2^2}\, \tq_0(\de\bw)} \geq 
\frac{\int \|\bw\|_1^2 e^{-q\|\bw\|_2^2}\, \tq_0(\de\bw)}{\int e^{-q\|\bw\|_2^2}\, \tq_0(\de\bw)} = 1.
\end{align}
In addition, we have
\begin{align}
&\sE(0;\nu) = \int w_1^2 \tq_0(\de\bw) = \frac{\nu+1}{k(k\nu+1)}\, ,\\
&\lim_{q_1 \to \infty} \frac{k\beta\delta}{k-1}\, \left\{\sE\left(\frac{\beta}{1+q_1};\nu\right) - \frac{1}{k^2}\right\} = \frac{k\beta\delta}{k-1}\left\{\sE\left(0;\nu\right) - \frac{1}{k^2}\right\} = \frac{\beta\delta}{k(k\nu+1)} < \infty.
\end{align}
Therefore, the right hand side of \eqref{eq:qs_1_main} is non-negative, continuous, bounded for $q_1^* \in [0, \infty)$.
Hence, using intermediate value theorem, \eqref{eq:qs_1_main} has a solution in $[0,\infty)$.

Now we will check that equations (\ref{eq:NMF1_Main}) and (\ref{eq:NMF2_Main}) hold for $\bm^{t+1} = \bm^t = \bm^*$, 
$\tbm^t = \tbm^*$, $\bQ^t = \bQ^{t+1} = \bQ^*$, $\tbQ^t = \tbQ^*$. 
We start with the first  equation in (\ref{eq:NMF1_Main}). Using Lemma \ref{lemma:UsefulFormulae}, we have
\begin{align}
\tsF(\tbm^*_a; \tbQ^*) = \frac{\sqrt{\beta}}{k}\bfone_k.
\end{align}
Therefore,
\begin{align}
\tsF(\tbm^*; \tbQ^*) = \frac{\sqrt{\beta}}{k}\bfone_n\otimes \bfone_k\,, \;\;\;\;\; \bX^\sT\tsF(\tbm^*; \tbQ^*) = \frac{\sqrt{\beta}}{k}\, (\bX^{\sT}\bfone_n)\otimes \bfone_k = \bm^*.
\end{align}
Now we consider the first equation in \eqref{eq:NMF2_Main}. Using Lemma \ref{lemma:UsefulFormulae}, we have
\begin{align}
\sF(\bm_i^*; \bQ^*) = \frac{\beta}{k(1+q_1^*+kq_2^*)}\left\langle\bX_{.,i}, \bfone_n\right\rangle\bfone_k.
\end{align}
Hence,
\begin{align}
\sF(\bm^*; \bQ^*) &= \frac{\beta}{k(1+q_1^*+kq_2^*)}(\bX^\sT\bfone_n)\otimes \bfone_k\,,\\
\bX\sF(\bm^*; \bQ^*) &= \frac{\beta}{k(1+q_1^*+kq_2^*)}(\bX\bX^\sT\bfone_n)\otimes \bfone_k = \tbm^*.
\end{align}
For the second equation in \eqref{eq:NMF1_Main}, note that using Lemma \ref{lemma:UsefulFormulae}, we have
\begin{align}
\frac{1}{d}\sum_{a=1}^n \tsG(\tbm^*_{a};\tbQ^*) = \delta\beta \left(\, \frac{k^2\sE(\tq_1^*;\nu)-1}{k(k-1)}\, \id_k- \frac{k\sE(\tq_1^*;\nu)-1}{k(k-1)}\, \bJ_k\right).
\end{align}
Note that using \eqref{eq:qs_1_main}, \eqref{eq:qs_2_main}
\begin{align}
\frac{k^2\sE(\tq_1^*;\nu)-1}{k(k-1)} &= \frac{1}{k(k-1)}\left[k^2\sE\left(\frac{\beta}{1+q_1^*};\nu\right)-1\right] = \frac{q_1^*}{\delta\beta},\\
\frac{-k\sE(\tq_1^*;\nu)+1}{k(k-1)} &= \frac{-1}{k(k-1)}\left[k\sE\left(\frac{\beta}{1+q_1^*};\nu\right)-1\right] = \frac{-1}{k(k-1)}\left[\frac{k-1}{\delta\beta}q_1^*+\frac{1-k}{k}\right] \\
&= \frac{1}{\delta\beta}\left(\frac{\beta\delta-kq_1^*}{k^2}\right) = \frac{q_2^*}{\delta\beta}.
\end{align}
Therefore, 
\begin{align}
\frac{1}{d}\sum_{a=1}^n \tsG(\tbm^*_{a};\tbQ^*) = q_1^* \id_k + q_2^*\bJ_k = \bQ^*.
\end{align}
Finally, we check the second equation in \eqref{eq:NMF2_Main}. Using Lemma \ref{lemma:UsefulFormulae}, we have
\begin{align}
\sG(\bm^*_{i};\bQ^*) = \frac{\beta}{(1+q_1^*)} \,\id_k+\beta \left\{\frac{\left\langle\bX_{.,i}, \bfone_n\right\rangle^2}{(1+q_1^*+kq_2^*)^2}-\frac{q_1^*}{(1+q_1^*)(1+q_1^*+kq_2^*)}\right\} \, \bJ_k.
\end{align}
Hence,
\begin{align}
\frac{1}{d}\sum_{i=1}^d \sG(\bm^*_{i};\bQ^*) &= \frac{\beta}{(1+q_1^*)} \,\id_k+\beta \left\{\frac{\|\bX^\sT\bfone_n\|_2^2}{d(1+q_1^*+kq_2^*)^2}-\frac{q_1^*}{(1+q_1^*)(1+q_1^*+kq_2^*)}\right\} \, \bJ_k\\
&= \tq_1^* \id_k + \tq_2^*\bJ_k = \tbQ^*,
\end{align}
this completes the proof.
\end{proof}

\subsection{Proof of Theorem \ref{thm:Main}}
\label{app:ProofMain}

We will first prove that, if $L(\beta,k,\delta,\nu)>1$, then the uninformative fixed point $(\br^*,\tbr^*,\bOmega^*,\tbOmega^*)$
(or equivalently, its conjugate $(\bm^*,\tbm^*,\bQ^*,\tbQ^*)$)
 is  (with high probability) a saddle point of the naive mean field free energy (\ref{eq:FreeEnergy_main}). This implies immediately
that the naive mean field iteration is unstable at that fixed point.

Note that the mapping $(\br,\tbr,\bOmega,\tbOmega)\to (\br,\tbr,\bQ,\tbQ)$ is a diffeomorphism (since the Jacobian is always invertible 
by strict convexity of  $\phi$, $\tphi$). We define $\cF_*$ to be the restriction of $\cF$ to the submanifold defined by
$\bQ=\bQ_*$, $\tbQ=\tbQ_*$. Explicitly, this can be written in terms of the partial Legendre transforms
(we repeat the definition of Eq.~(\ref{eq:LegendrePhiPartial}) for the reader's convenience):
\begin{align}
\psi(\br,\bQ)  \equiv \sup_{\bm}\left\{\< \br, \bm\> - \phi(\bm, \bQ)\right\} \, ,\;\;\;\;
\tpsi(\tbr,\tbQ)  \equiv \sup_{\tbm}\left\{\< \tbr, \tbm\>- \tphi(\tbm, \tbQ)\right\} \, .
\end{align}
We then have
\begin{align}
\cF_*(\br,\tbr) = & \sum_{i=1}^d\psi(\br_i,\bQ_*)+\sum_{a=1}^n\tpsi(\tbr_a,\tbQ_*) -\sqrt{\beta}\Tr\left(\bX\br\tbr^{\sT}\right)\nonumber\\
&- \frac{d}{2}\<\bQ_*,\bOmega\>- \frac{n}{2}\<\tbQ_*,\tbOmega\>+\frac{\beta n}{2}\<\bOmega,\tbOmega\>\, ,\\
\bOmega &\equiv \frac{1}{d\beta} \sum_{i=1}^d\sG(\bm_i;\bQ^*)\,,\;\;\;\;\;\;\;\;\tbOmega \equiv \frac{1}{n\beta}\sum_{a=1}^n\tsG(\tbm_a;\tbQ^*)\, ,
\label{eq:OmegaM}\\
\br_i &\equiv \frac{1}{\sqrt{\beta}}\sF(\bm_i;\bQ^*)\,,\;\;\;\;\;\;\;\;\tbr_a \equiv \frac{1}{\sqrt{\beta}}\tsF(\tbm_a;\tbQ^*)\,,\label{eq:r_def_app}\, .
\end{align}
In order to prove that $(\br_*,\tbr_*)$ is a saddle point of $\cF$, it is sufficient to show that it is a saddle along a submanifold,
and hence that the Hessian of $\cF_*$ has a negative eigenvalue at $(\br_*,\tbr_*)$.

Next notice that
\begin{align}
\cF_*(\br, \tbr) &= \cG_1(\br, \tbr) + \cG_2(\br, \tbr)\, ,\\
\cG_1(\br, \tbr) &\equiv \sum_{i=1}^d\psi(\br_i,\bQ_*)+\sum_{a=1}^n\tpsi(\tbr_a,\tbQ_*) -\sqrt{\beta}\Tr\left(\bX\br\tbr^{\sT}\right)\, ,\\
\cG_2(\br, \tbr)& \equiv - \frac{d}{2}\<\bQ_*,\bOmega\>- \frac{n}{2}\<\tbQ_*,\tbOmega\>+\frac{\beta n}{2}\<\bOmega,\tbOmega\>\, .
\end{align}
Consider deviations from the stationary point $\br_i = \br_i^*+\bdelta_i$, $\tbr_a = \tbr_a^*+\tbdelta_a$. By Eqs.~(\ref{eq:OmegaM})
and (\ref{eq:r_def_app}), we have (for some tensors $\bT,\tbT\in (\reals^k)^{\otimes 3})$)
\begin{align}
\bOmega = \bOmega^*+\frac{1}{d}\sum_{i=1}^{d}\bT\bdelta_i +\bDelta\, ,\;\;\;\;\;\tbOmega = \tbOmega^*+\frac{1}{n}\sum_{a=1}^{n}\tbT\tbdelta_a +\tbDelta\, ,
\end{align}
where $\bDelta$, $\tbDelta$ are of second order in $\bdelta,\tbdelta$.  
At the stationary point, by Eq.~(\ref{eq:DF_Omega}), we have $\bQ^*= \beta\bOmega^*$,
$\tbQ^*= \beta\delta\tbOmega^*$. Hence, substituting in $\cG_2$, and letting $M_{ij}= \sum_{s,t}T_{st,i}\tilde{T}_{st,j}$, we obtain
\begin{align}
\cG_2(\br, \tbr)& = \cG_2(\br_*, \tbr_*)  +\frac{\beta}{2d}\sum_{i=1}^d\sum_{a=1}^n\<\bdelta_i,\bM\tbdelta_a\> +o(\bdelta^2)
\end{align}
Therefore, the Hessian $\nabla^2\bG_2(\br_*,\tbr_*)$ has rank at most $k$.

Since $\psi(\,\cdot\,,\bQ^*)$, 
$\tilde \psi(\,\cdot\,, \tbQ^*)$ are Legendre transforms of  $\phi(\,\cdot\,, \bQ^*)$, $\tilde \phi(\,\cdot\,, \tbQ^*)$, respectively,
we have
\begin{align}
&\nabla_{\br\br}^2 \psi(\br, \bQ^*) = \left(\nabla_{\bm\bm}^2\phi(\bm, \bQ^*)\right)^{-1} = \id_k + \bQ^*,\\
&\nabla_{\tbr\tbr}^2 \tilde\psi(\tbr, \tbQ^*) = \left(\nabla_{\tbm\tbm}^2 \tilde\phi(\tbm, \tbQ^*)\right)^{-1} = \bD^{-1}
\end{align}
where $\bD \in \reals^{k\times k}$ is as
\begin{align}
D_{ij} = \frac{1}{\sqrt{\beta}}\frac{\partial \tsF_i\left(\tbm; \tbQ\right)}{\partial \tilde m_j}\Bigg|_{\tbm = 0, \tbQ = \tbQ^*}.
\end{align}
Thus,
\begin{align}
\bD &= \frac{\left(\int \bw^{\otimes 2}e^{-\tq_1^*\|\bw\|_2^2/2}\tq_0(\de \bw)\right)\left(\int e^{-\tq_1^*\|\bw\|_2^2/2}\tq_0(\de \bw)\right) - \left(\int \bw e^{-\tq_1^*\|\bw\|_2^2/2}\tq_0(\de \bw)\right)^{\otimes 2}}{\left(\int e^{-\tq_1^*\|\bw\|_2^2/2}\tq_0(\de \bw)\right)^2}\\
&= \frac{\bQ^*}{\delta\beta} - \frac{\bJ_k}{k^2}.
\end{align}
Hence, 
\begin{align}
\nabla^2\mathcal G_1 =
\begin{bmatrix}
\id_d \otimes \left(\id_k + \tbQ^*\right) & -\sqrt{\beta}\bX^\sT \otimes \id_k \\
-\sqrt{\beta}\bX \otimes \id_k & \id_n \otimes \bD^{-1}
\end{bmatrix}.
\end{align}
Since $\id_k + \tbQ^*$ is positive definite, $\nabla^2\mathcal G \succeq 0$ if and only if
\begin{align}
&\id_n \otimes \bD^{-1} \succeq \beta\left(\bX\otimes \id_k\right)\left(\id_d\otimes\left(\id_k+\tbQ^*\right)\right)^{-1}\left(\bX^\sT\otimes \id_k\right) \iff \\
&\id_n\otimes \bD^{-1} \succeq \beta\left(\bX\bX^\sT\right) \otimes\left(\id_k+\tbQ^*\right)^{-1} \iff\\
&\id_n\otimes \id_k \succeq \beta\left(\bX\bX^\sT\right)\otimes\left(\id_k + \tbQ^*\right)^{-1}\bD.
\end{align}
Hence, $\nabla^2\mathcal G_1$ has a negative eigenvalue if and only if
\begin{align}
\beta \lambda_{\max}\left(\bX\bX^\sT \right) \lambda_{\max}\left(\left(\id_k + \tbQ^*\right)^{-1}\bD\right) > 1.
\end{align}
Further, by the same argument,  if $\beta\lambda_{\ell}(\bX\bX^\sT) \lambda_{\max}((\id_k + \tbQ^*)^{-1}\bD) > 1$, then
$\nabla^2\cG_1$ has at least $\ell$ negative eigenvalues (recall that $\lambda_{\ell}(\bM)$ denotes the $\ell$-th eigenvalue of $\bM$
in decreasing order).

Note that 
\begin{align}
\left(\id_k + \bQ^*\right)^{-1}\bD &= \left(\frac{\id_k}{1+q_1^*}-\frac{q_2^*}{(1+q_1^*)(1+q_1^*+kq_2^*)}\bJ_k\right)\left(\frac{q_1^*}{\delta\beta}\id_k+\left(\frac{q_2^*}{\delta\beta}-\frac{1}{k^2}\right)\bJ_k\right)\\
&= \frac{1}{1+q_1^*}\left(\frac{q_1^*}{\delta\beta}\id_k + \left(\frac{q_2^*}{1+q_1^*+kq_2^*}\left(\frac{1}{\delta\beta}+\frac{1}{k}\right)-\frac{1}{k^2}\right)\bJ_k\right),\\
\mu(\beta,\delta)&\equiv \lambda_{\max}\left(\left(\id_k + \bQ^*\right)^{-1}\bD\right)\\
&= \frac{1}{1+q_1^*}\left(\frac{q_1^*}{\delta\beta} + k\left[\frac{q_2^*}{1+q_1^*+kq_2^*}\left(\frac{1}{\delta\beta}+\frac{1}{k}\right)-\frac{1}{k^2}\right]_+\right).
\end{align}
where $q_1^*, \tq_1^*$, $q_2^*$ are given in \eqref{eq:qs_1_main}, \eqref{eq:qs_2_main}, \eqref{eq:qs_3_main}.
Further $\bX\bX^\sT$ is a low-rank deformation of a Wishart matrix. Hence, for any fixed $\ell$, we have,
almost surely
\begin{align}
\lim\inf_{n,d\to\infty}\lambda_{\ell}(\bX\bX^{\sT})\ge \left(1+\frac{1}{\sqrt{\delta}}\right)^2\, .
\end{align}
Thus, if 
\begin{align}
L(\beta, \delta) = \beta\lambda_{\max}\left(1+\frac{1}{\sqrt{\delta}}\right)^2\mu(\beta,\delta) > 1,
\end{align}
we have $\lambda_{n}(\nabla^2 \cG_1) \le \dots\le \lambda_{n-\ell}(\nabla^2 \cG_1)<0$ with high probability for any fixed $\ell$. 

As explained above, $\nabla^2\cG_2$
has rank at most $k$. Therefore, by Cauchy's interlacing inequality, if $L(\beta,k, \delta,\nu)  > 1$, 
\begin{align}
\lambda_{\min}\left(\nabla^2\cF_*\right) \le \lambda_{n+k}\left(\nabla^2\cG_1 + \nabla^2\cG_2\right) < 0.
\end{align}
Hence, for $L(\beta, \delta) > 1$, $\nabla^2 \cF_*$ has a negative eigenvalue.

Note that the mapping $(\br,\tbr,\bOmega,\tbOmega)\to (\bm,\tbm,\bQ,\tbQ)$ is a diffeomorphism,
and therefore, uninformative fixed point $(\bm^*,\tbm^*,\bQ^*,\tbQ^*)$ is a saddle also when we consider the free energy as
a function of the parameters $(\bm,\tbm,\bQ,\tbQ)$.
The claim that $(\bm^*,\bQ^*)$ is unstable under the naive mean field iteration follows immediately 
from the above, by using Lemma \ref{lemma:hessianstable}, applied to
$f(\bx,\by) = \cF(\bm,\tbm,\bQ,\tbQ)$, whereby $\bx = (\bm,\bQ)$, $\by=(\tbm,\tbQ)$.

\section{Naive Mean Field: Further numerical results}
\label{app:Numerical_MF}

In this section we report on additional numerical simulations using the alternate minimization to
minimize the naive mean field free energy. These results confirm the one presented in the main text in Section \ref{sec:NMF_numerical}.

\subsection{Credible intervals}

\begin{figure}[h!]
\centering
\includegraphics[height=3.in]{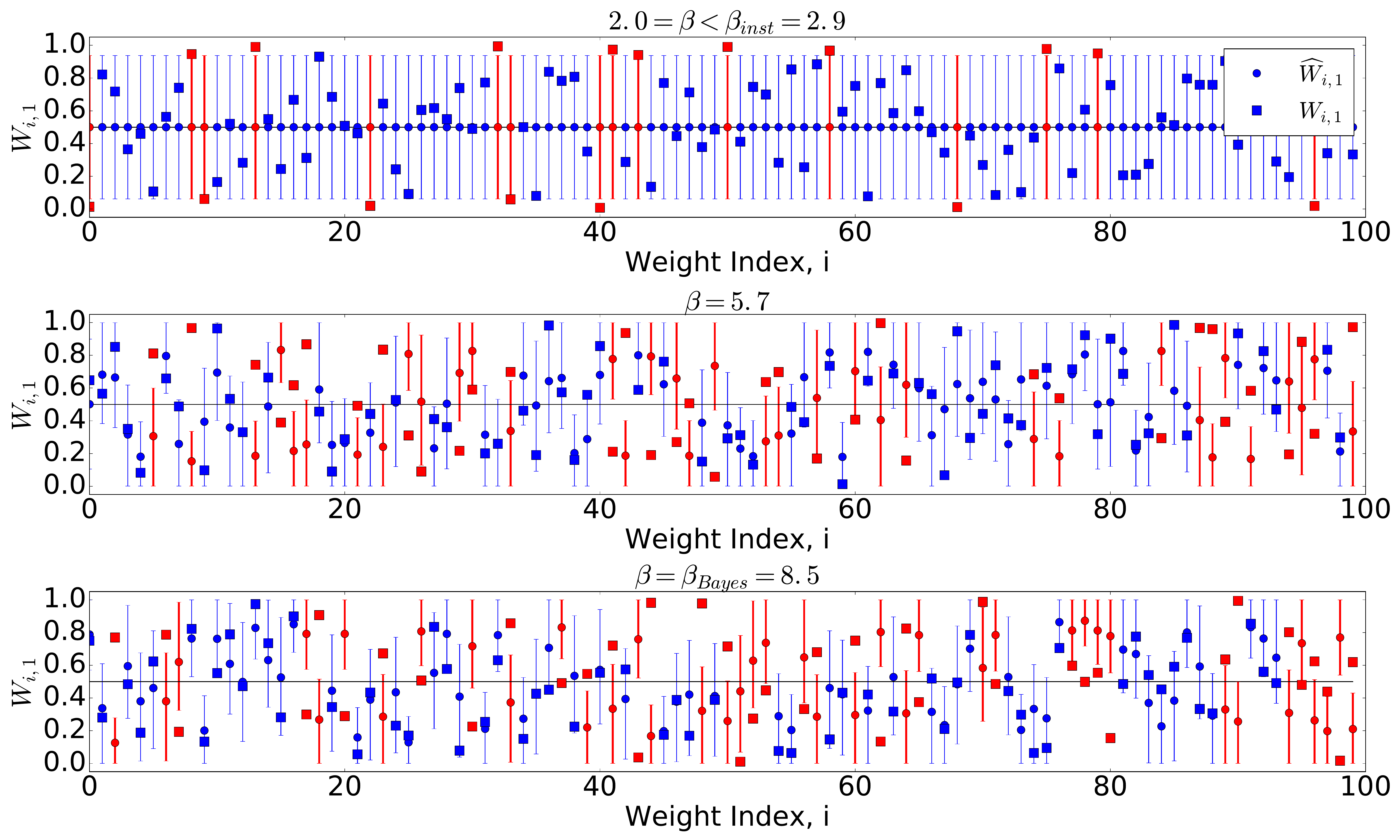}
\caption{Bayesian credible intervals as computed by variational inference at nominal coverage level $1-\alpha= 0.9$.
Here $k=2$, $d=5000$, $n=2500$ and we consider three values of $\beta$: $\beta\in\{2,5.7,8.5\}$ (for reference $\beta_{\inst}\approx 2.9$, $\beta_{\sBayes}\approx 8.5$.
Circles correspond to the posterior mean, and squares to the actual weights. We use red for the coordinates on which the 
credible interval does not cover the actual value of $w_{i,1}$.}
\label{fig:Uncertainty_delta_p5}
\end{figure}

\begin{figure}[h!]
\centering
\includegraphics[height=3.in]{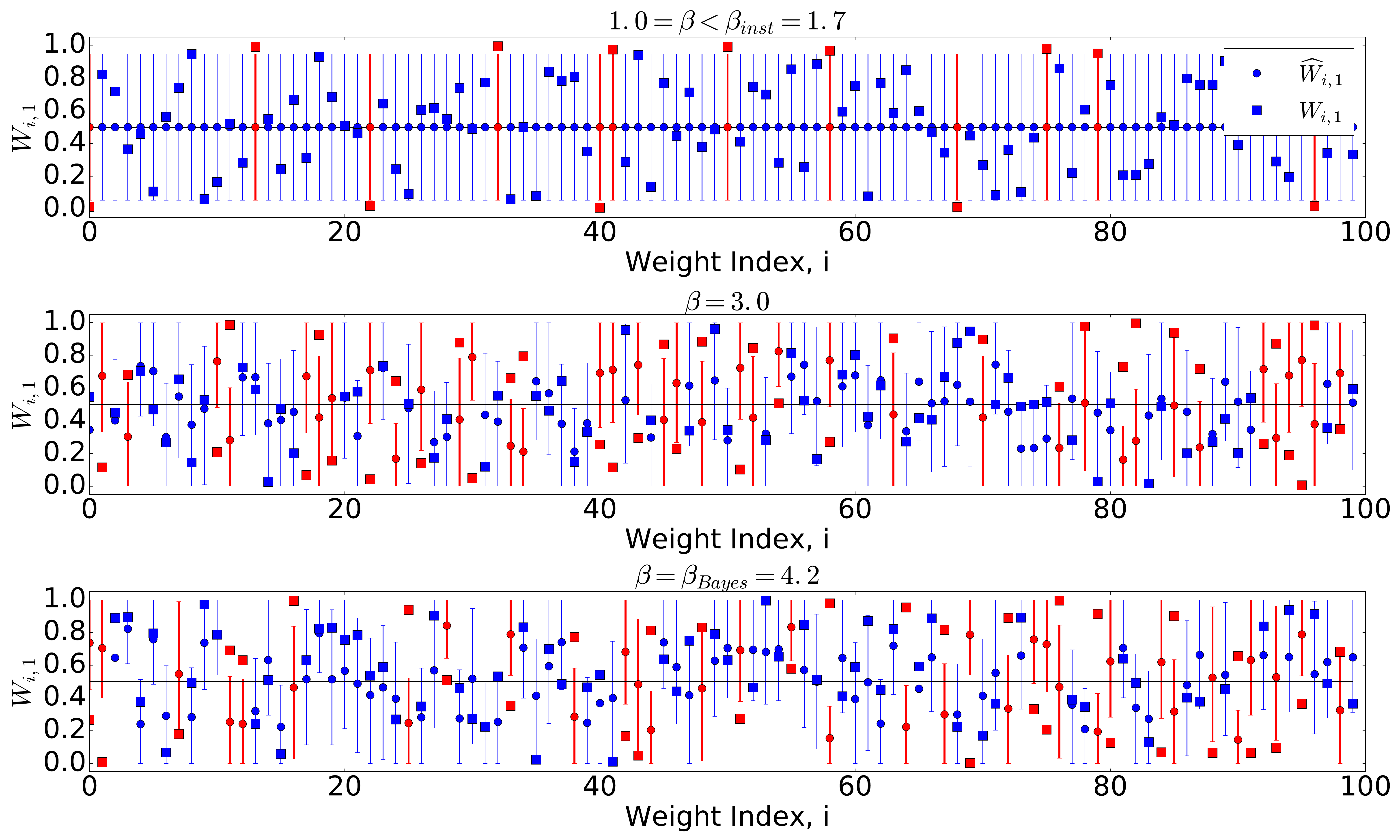}
\caption{Bayesian credible intervals as computed by variational inference at nominal coverage level $1-\alpha= 0.9$.
Here $k=2$, $d=5000$, $n=10000$ and we consider three values of $\beta$: $\beta\in\{1,3,4.2\}$ (for reference $\beta_{\inst}\approx 1.7, \beta_{\sBayes}\approx 4.2$.
Circles correspond to the posterior mean, and squares to the actual weights. We use red for the coordinates on which the 
credible interval does not cover the actual value of $w_{i,1}$.}
\label{fig:Uncertainty_delta_2}
\end{figure}
In Figures \ref{fig:Uncertainty_delta_p5} and \ref{fig:Uncertainty_delta_2} we plot Bayesian credible intervals for the weights $w_{i,1}$ as computed 
within naive mean field, for $k=2$, $d=5000$. These simulations are analogous to the one reported in the main text in Figure \ref{fig:Uncertainty_delta1},
but we use $n = 2500$ ($\delta=0.5$) in Figure \ref{fig:Uncertainty_delta_p5} and  $n = 10000$ ($\delta=2$) in Figure \ref{fig:Uncertainty_delta_p5}.

The nominal coverage of these intervals is $0.9$, but we obtain a smaller empirical coverage.
For $\delta=0.5$, the empirical coverage was $0.87$ (for $\beta = 2<\beta_{\inst}$), $0.61$ (for $\beta=5.7\in(\beta_{\inst},\beta_{\sBayes})$),
and $0.64$ (for $\beta=8.5\approx\beta_{\sBayes}$). For $\delta =2$, the empirical coverage was $0.89$ (for $\beta = 1<\beta_{\inst}$), $0.69$ 
(for $\beta=3\in(\beta_{\inst},\beta_{\sBayes})$), and $0.65$ (for $\beta=4.2\approx\beta_{\sBayes}$). 

\subsection{Results for $k=3$ topics}
\label{sec:NMF_k3}
\begin{figure}[h!]
\phantom{A}
\vspace{-1cm}

\centering
\includegraphics[height=5.5in]{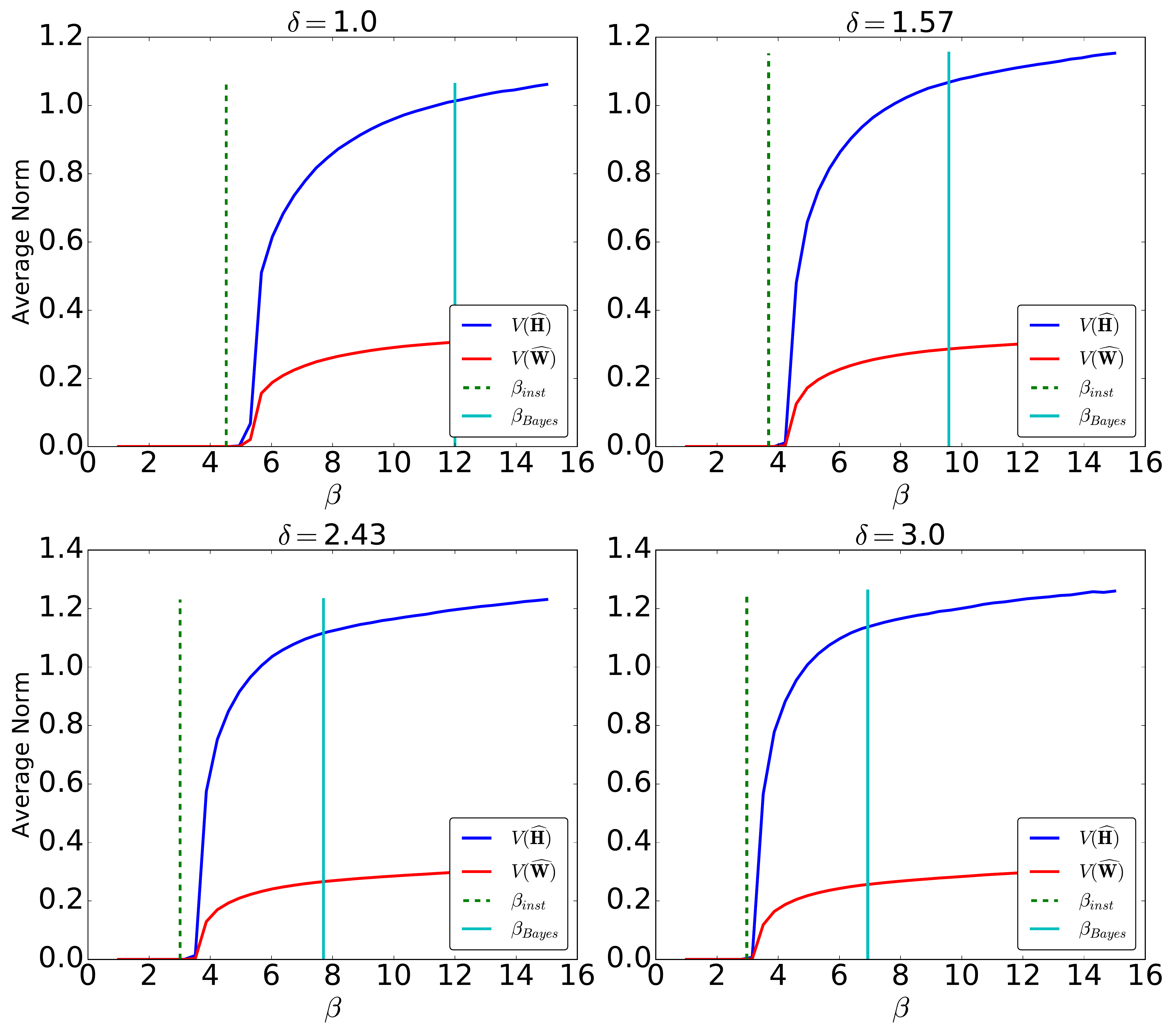}
\caption{Normalized distances $\Norm(\hbH)$, $\Norm(\hbW)$  of the naive mean field estimates from the uninformative fixed point.  Here
 $d = 1000$ and changed $n= d\delta$: each data point corresponds to an average over $400$ random realizations.} 
\label{fig:H_norm_k_3}
\end{figure}

\begin{figure}[h!]
\phantom{A}\hspace{-1.85cm}\includegraphics[height=2.66in]{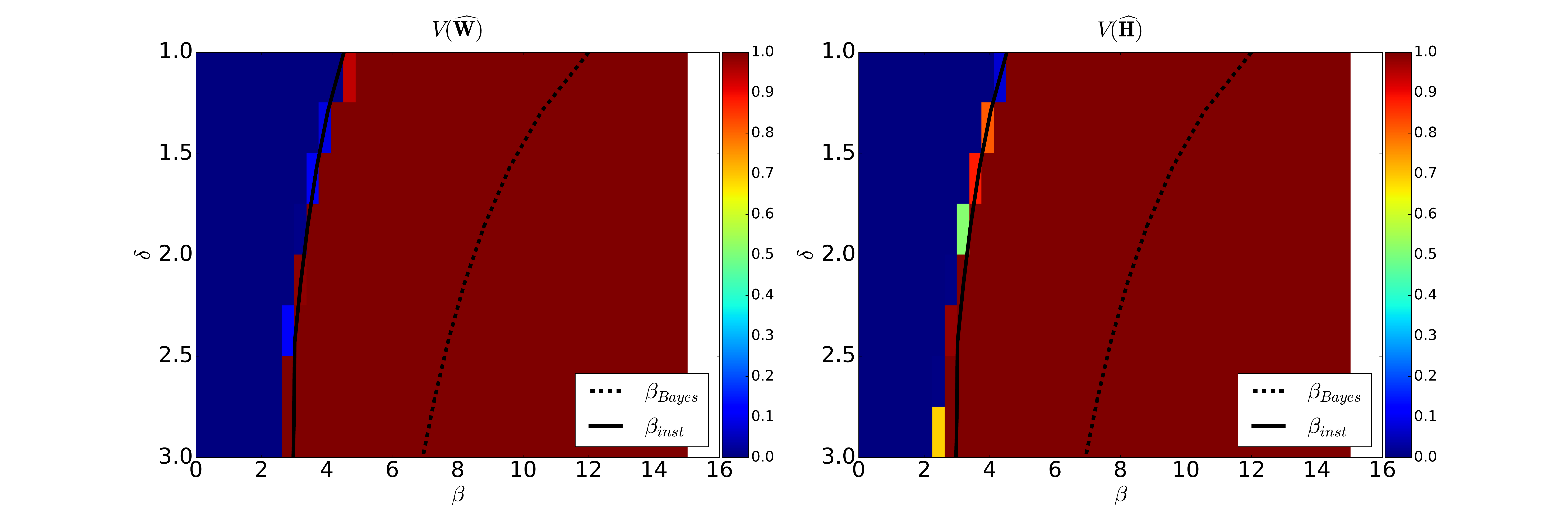}
\caption{Empirical fraction of instances such that  
$\Norm(\hbW)\ge \eps_0=5\cdot 10^{-3}$ (left) or $\Norm(\hbH)\ge \eps_0$ (right), where $\hbW, \hbH$ are the naive mean field
estimate. Here $k=3$, $d=1000$ and, for each $(\delta,\beta)$ point on a grid, we used $400$ random realizations to estimate the probability of $\Norm(\hbW)\ge \eps_0$.}
\label{fig:H_norm_k_3_HM}
\end{figure}

\begin{figure}[h!]
\phantom{A}
\vspace{-1cm}

\includegraphics[height=5.5in]{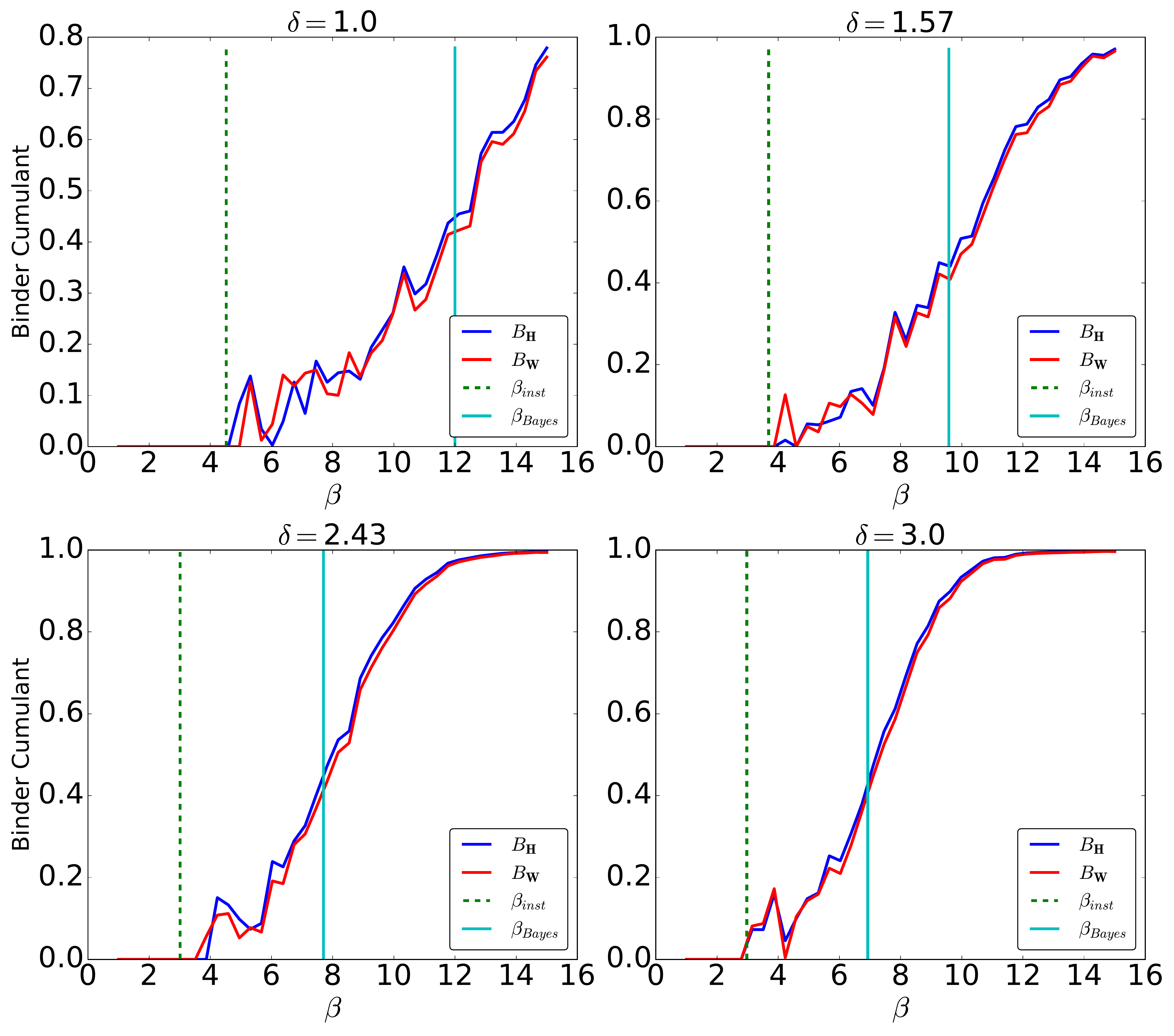}

\caption{Binder cumulant for the correlation between the naive mean field estimates $\hbH$ and the true topics $\bH$.
 Here we report results for $k=3$, $d =1000$ and $n=d\delta$, obtained by averaging over $400$
realizations. Note that for $\beta<\beta_{\sBayes}(k,\nu,\delta)$,
 $\Bind_{\bH}$ decreases with the dimensions, suggesting asymptotically vanishing correlations.}
\label{fig:Binder_k_3}
\end{figure}

\begin{figure}[h!]
\phantom{A}\hspace{-1.85cm}\includegraphics[height=2.66in]{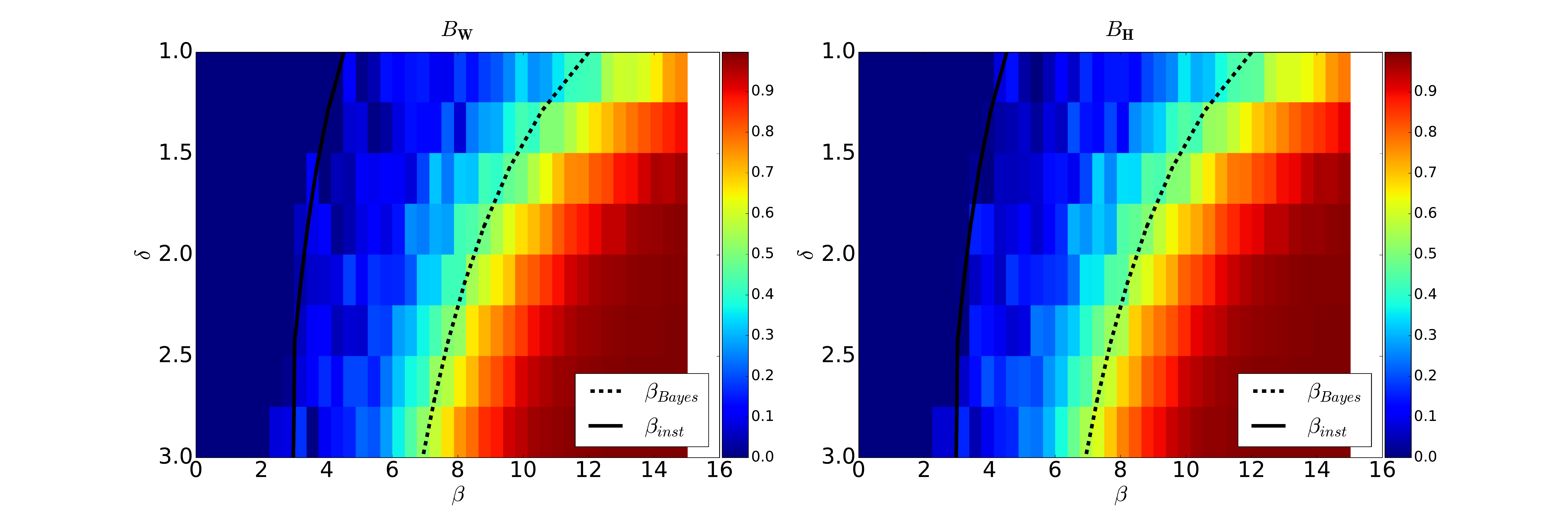}
\caption{Binder cumulant for the correlation between the naive mean field estimates $\hbW$, $\hbH$ and the true weights and topics
$\bW$, $\bH$. Here $k=3$, $d=1000$ and $n=d\delta$, and we averaged over $400$ realizations.}
\label{fig:Binder_k_3_HM}
\end{figure}

In Figures \ref{fig:H_norm_k_3} to \ref{fig:Binder_k_3_HM} we report our results using alternating minimization to
minimize the naive mean field free energy for $k=3$.

In Figures  \ref{fig:H_norm_k_3},  \ref{fig:H_norm_k_3_HM} we plot (respectively) the normalized distances $\Norm(\hbH)$, $\Norm(\hbW)$
from the uninformative subspaces $\{\bH = \bv\otimes\bfone_k:\; \bv\in\reals^d\}$ and $\{\bW = \bv\otimes\bfone_k:\; \bv\in\reals^d\}$.
Data are consistent with the claim that this distance becomes significant when $\beta\ge \beta_{\inst}(k,\nu,\delta)$.

In Figures \ref{fig:Binder_k_3}, \ref{fig:Binder_k_3_HM} we consider the correlation between the estimates $\hbH, \hbW$ and the true factorization
$\bH, \bW$, and  define a Binder cumulant as follows for $k\ge 3$.
Let $\Corr_{\eta}(\bH,\hbH)$ be the $k \times k$ matrix with entries
\begin{eqnarray}
\Corr_{\eta}(\bH,\hbH)_{i, j} &=&\frac{\<(\hbH_{\perp})_i + \eta \bg,(\bH_{\perp})_j\>}{\Vert (\hbH_{\perp})_i + \eta \bg \Vert_2 \Vert (\bH_{\perp})_j \Vert_2}
\end{eqnarray}
We then define
\begin{eqnarray} \label{eqn:Binder_General}
\hat{\bR} &\equiv& \frac{\hE\left\{\sum_{i, j \leq k} \Corr_{\eta}(\bH,\hbH)_{i, j}^4\right\}}{ \hE\left\{ \sum_{i, j \leq k} \Corr_{\eta}(\bH,\hbH)_{i, j}^2 \right\}^2} \\
\Bind_{\bH} &\equiv& \left\{ \begin{array}{cl}
6\bigg(\max\big\{\frac{2}{3} - \hat{\bR}\big\} - \frac{1}{3}\bigg) & \mbox{if } \hE\left\{ \sum_{i, j \leq k} \Corr_{\eta}(\bH,\hbH)_{i, j}^2 \right\} > 0.01 \, ,\\
0 & \mbox{otherwise.}
\end{array} \right.
\end{eqnarray}
Here $\hE$ denotes empirical average with respect to the sample and $\bg\sim\normal(0,\id_d)$. We set $\eta=10^{-4}$.
An analogous definition holds for $\Corr_{\eta}(\hbW)$, $\Bind_{\eta}(\hbW)$. In equation \eqref{eqn:Binder_General} we introduced a max thresholding step and a threshold on the denominator. These are added to ensure the stability of the fraction below the phase transition region where the denominator of $\hat{\boldsymbol{R}}$ vanishes. 

Figures \ref{fig:Binder_k_3}, \ref{fig:Binder_k_3_HM} are consistent with the prediction that the correlation between the AMP estimates and the true
factors $\bW,\bH$ starts to be non-negligible at the Bayes threshold.

\section{TAP free energy and approximate message passing}
\label{app:TAP}

\subsection{Heuristic derivation of the TAP free energy}
\label{app:TAP_Derivation}

Several heuristic approaches exist to construct the TAP free energy. Here we will derive the expression (\ref{eq:FreeEnergy_TAP_TM}) of the TAP
free energy for topic models as an approximation of the Bethe  free energy for the same problem: we refer to 
\cite{wainwright2008graphical,MezardMontanari,koller2009probabilistic} for background on the latter. 
Let us emphasize that our  derivation will be only heuristic, since our rigorous results are obtained by analyzing  the resulting expression
$\cF_{\sTAP}(\br,\tbr)$ and do not require a rigorous justification of Eq.~(\ref{eq:FreeEnergy_TAP_TM}).

The posterior $p_{\bH,\bW|\bX}$ takes the form
\begin{align}
p_{\bH,\bW|\bX}(\bH,\bW|\bX) = \frac{1}{Z(\bX)}\prod_{(a,i)\in [n]\times [d]} \exp\left\{\sqrt{\beta}X_{ai} \<\bw_a,\bh_i\> -\frac{\beta}{2d} \<\bw_a,\bh_i\>^2\right\}
\prod_{a=1}^d\tq_0(\bw_a)\prod_{i=1}^dq_0(\bh_i)\, .
\end{align}
This can be regarded as a pairwise graphical model whose underlying graph is the complete bipartite graph over
vertex sets $[n]$ (associated to variables $\bw_1$, \dots $\bw_n$) and $[d]$  (associated to variables $\bh_1$, \dots $\bh_d$).
The Bethe free energy $\cF_{\sBethe}$ takes as input messages $\bq \equiv (q_{i\to a})_{i\in [d],a\in [n]}$, $\tbq= (\tq_{a\to i})_{i\in [d],a\in [n]}$.
Messages are probability densities over the $\bh_i$'s (for  $q_{i\to a}$) or the $\bw_a$'s (for  $\tq_{a\to i}$),  
indexed  by the directed edges in this graph (each pair $(a,i)$, $a\in [n]$, $i\in [d]$ gives rise to two directed edges). 
The free energy takes the form \cite{MezardMontanari}
\begin{align}
\cF_{\sBethe}(\bq,\tbq)&= \sum_{a=1}^n\sum_{i=1}^d\log Z_{ai}-\sum_{i=1}^d\log Z_i-\sum_{a=1}^n\log \tZ_a  \,,\\
Z_i & = \int \prod_{a=1}^n e^{\sqrt{\beta}X_{ai} \<\bw_a,\bh_i\> -\frac{\beta}{2d} \<\bw_a,\bh_i\>^2} \de q_0(\bh_i)\,  \prod_{a=1}^n \de\tq_{a\to i}(\bw_a)\, ,\label{eq:Zi}\\
\tZ_a & = \int \prod_{i=1}^d e^{\sqrt{\beta}X_{ai} \<\bw_a,\bh_i\> -\frac{\beta}{2d} \<\bw_a,\bh_i\>^2} \de \tq_0(\bw_a)\,  \prod_{i=1}^d \de q_{i\to a}(\bh_i)\, ,\label{eq:Za}\\
Z_{ai} & = \int e^{\sqrt{\beta}X_{ai} \<\bw_a,\bh_i\> -\frac{\beta}{2d} \<\bw_a,\bh_i\>^2} \, \, \de q_{i\to a}(\bh_i)\,\de \tq_{a\to i}(\bw_a)\, . \label{eq:Zai}
\end{align}
The stationarity conditions for $\cF_{\sBethe}(\bq,\tbq)$ correspond to the belief propagation fixed point 
equations
\begin{align}
q_{i\to b}(\bh_i) & = \frac{1}{C_{i\to b}}\, q_0(\bh_i) \, \prod_{a\in[n]\setminus b} \int e^{\sqrt{\beta}X_{ai} \<\bw_a,\bh_i\> -\frac{\beta}{2d} \<\bw_a,\bh_i\>^2} \de\tq_{a\to i}(\bw_a)\, ,\label{eq:BP_FP1}\\
\tq_{a\to j}(\bw_i) & = \frac{1}{\tilde{C}_{a\to j}}\, \tq_0(\bw_i) \, \prod_{i\in [d]\setminus j} \int e^{\sqrt{\beta}X_{ai} \<\bw_a,\bh_i\> -\frac{\beta}{2d} \<\bw_a,\bh_i\>^2} \de q_{i\to a}(\bh_i)\, .
\label{eq:BP_FP2}
\end{align}
We define $\fb_{i\to a}= \int \bh_i \de q_{i\to a}(\bh_i)$, $\tfb_{a\to i}= \int \bw_a \de \tq_{a\to i}(\bw_a)$, and
$\bg_{i\to a}= \int \bh^{\otimes 2}_i \de q_{i\to a}(\bh_i)$, $\tbg_{a\to i}= \int \bw^{\otimes 2}_a \de \tq_{a\to i}(\bw_a)$. Since  
$X_{ai} = O(1/\sqrt{n})$,  we have
\begin{align}
 \prod_{i=1}^d&\int e^{\sqrt{\beta}X_{ai} \<\bw_a,\bh_i\> -\frac{\beta}{2d} \<\bw_a,\bh_i\>^2} \de q_{i\to a}(\bh_i) =\\
&=\prod_{i=1}^d
\exp\left\{\sqrt{\beta} X_{ai}\<\fb_{i\to a}, \bw_a\>-\frac{\beta}{2d}\<\fb_{i\to a},\bw_a\>^2+\frac{\beta}{2}\Big(X_{ai}^2-\frac{1}{d}\Big)
\<\bg_{i\to a}-\fb_{i\to a}^{\otimes 2},\bw_a^{\otimes 2}\>+O(n^{-3/2})\right\}\\
& =\exp\left\{\sum_{i=1}^d\sqrt{\beta} X_{ai}\<\fb_{i\to a}, \bw_a\>-\frac{\beta}{2d}\sum_{i=1}^d\<\fb_{i\to a},\bw_a\>^2+O(n^{-1/2})\right\}\, ,
\label{eq:ExpansionLargeDeg}
\end{align}
where in the last step we used the fact that $\E\{X^2_{ai}-d^{-1}\}=O(n^{-3/2})$ and applied the central limit theorem.

Using the expression (\ref{eq:ExpansionLargeDeg}) in Eq.~(\ref{eq:Za}), and repeating a similar calculation for (\ref{eq:Zi}),
we get
\begin{align}
\log Z_i & =\phi\left(\sqrt{\beta}\sum_{a=1}^nX_{ai}\tfb_{a\to i},\frac{\beta}{d}\sum_{a=1}^n\tfb_{a\to i}^{\otimes 2}\right) +O(n^{-1/2})\, ,\label{eq:Zi_formula}\\
\log \tZ_a & = \tphi\left(\sqrt{\beta}\sum_{i=1}^dX_{ai}\fb_{i\to a},\frac{\beta}{d}\sum_{i=1}^d\fb_{i\to a}^{\otimes 2}\right) +O(n^{-1/2})\, ,\label{eq:Za_formula}
\end{align}
where the functions $\phi$, $\tphi$ are defined implicitly in Eq.~(\ref{eq:densityforms_main}).

We can similarly expand $Z_{ai}$ for large $n,d$:
\begin{align}
Z_{ai}&= 1+\sqrt{\beta} X_{ai}\<\tfb_{a\to i},\fb_{i\to a}\> +\frac{\beta}{2}\Big(X_{ai}^2-\frac{1}{d}\Big) \<\tbg_{a\to i},\bg_{i\to a}\> +O(n^{-3/2})\\
& = \exp\left\{\sqrt{\beta} X_{ai}\<\tfb_{a\to i},\fb_{i\to a}\>-\frac{\beta}{2} X_{ai}^2\<\tfb_{a\to i},\fb_{i\to a}\>^2 +\frac{\beta}{2}\Big(X_{ai}^2-
\frac{1}{d}\Big) \<\tbg_{a\to i},\bg_{i\to a}\> +O(n^{-3/2})\right\}\, .
\end{align}
Therefore, using again the central limit theorem,
\begin{align}
\sum_{a\le n, i\le d}\log Z_{ai}& = \sqrt{\beta} \sum_{a\le n, i\le d} X_{ai}\<\tfb_{a\to i},\fb_{i\to a}\>-\frac{\beta}{2d} \sum_{a\le n, i\le d}\<\tfb_{a\to i},\fb_{i\to a}\>^2+
O(n^{1/2})\, . \label{eq:Zai_formula}
\end{align}

Putting together Eqs.~(\ref{eq:Zi_formula}), (\ref{eq:Za_formula}), and (\ref{eq:Zai_formula}), we obtain
\begin{align}
\cF_{\sBethe}(\bq,\tbq)&= -\sum_{i=1}^d \phi\left(\sqrt{\beta}\sum_{a=1}^nX_{ai}\tfb_{a\to i},\frac{\beta}{d}\sum_{a=1}^n\tfb_{a\to i}^{\otimes 2}\right)
-\sum_{a=1}^n \tphi\left(\sqrt{\beta}\sum_{i=1}^dX_{ai}\fb_{i\to a},\frac{\beta}{d}\sum_{i=1}^d\fb_{i\to a}^{\otimes 2}\right)\nonumber\\
& +\sqrt{\beta}\sum_{a\le n, i\le d} X_{ai}\<\tfb_{a\to i},\fb_{i\to a}\>-\frac{\beta}{2d} \sum_{a\le n, i\le d}\<\tfb_{a\to i},\fb_{i\to a}\>^2 +O(n^{1/2})\,.
\end{align}
Close to the solution of the stationarity conditions (\ref{eq:BP_FP1}), (\ref{eq:BP_FP2}), the message $\fb_{i\to a}$ should be roughly independent of 
$a\in [n]$ and $\tfb_{a\to i}$ should be roughly independent of $i\in [d]$. Hence, we can approximate
\begin{align}
-\frac{\beta}{2d} \sum_{a\le n, i\le d}\<\tfb_{a\to i},\fb_{i\to a}\>^2= -\frac{\beta}{2nd^2} \sum_{a\le n, i\le d}\sum_{b\le n, j\le d}\<\tfb_{a\to j},\fb_{i\to b}\>^2+o(n)\, .
\label{eq:ApproximationMessages}
\end{align}
In order to obtain the expression of Eq.~(\ref{eq:FreeEnergy_TAP_TM}) we add auxiliary variables $\bm_i,\tbm_a\in\reals^k$,
and $\bQ_i,\tbQ_a\in\reals^{k\times k}$, alongside Lagrange multipliers $\br_i$, $\tbr_a$, $\bOmega_i$, $\tbOmega_a$ to enforce the constraints 
\begin{align}
\bm_i = \sqrt{\beta}\sum_{a=1}^nX_{ai}\tfb_{a\to i}\, ,\;\;\;\;\;\;\; \bQ_i = \frac{\beta}{d}\sum_{a=1}^n\tfb_{a\to i}^{\otimes 2}\, ,\\
\bm_a = \sqrt{\beta}\sum_{i=1}^dX_{ai}\fb_{i\to a}\, ,\;\;\;\;\;\;\; \tbQ_a = \frac{\beta}{d}\sum_{i=1}^d\fb_{i\to a}^{\otimes 2}\, .
\end{align}
Denoting by $\bm\in\reals^{d\times k}$ the matrix whose $i$-th row is $\bm_i$ (and analogously for $\tbm$, $\fb$, $\tfb$
and the Lagrange multipliers $\br$, $\tbr$), and using Eq.~(\ref{eq:ApproximationMessages}) we obtain the Lagrangian
(here all sums run over $a\in [n]$ and $i\in [d]$)
\begin{align}
\cL = & \<\br,\bm\>-\sqrt{\beta}\sum_{a, i}X_{ai}\<\br_i,\tfb_{a\to i}\>+\<\tbr,\tbm\>-\sqrt{\beta}\sum_{a, i}X_{ai}\<\tbr_a,\fb_{i\to a}\>
 +\sqrt{\beta}\sum_{a, i}X_{ai}\<\tfb_{a\to i},\fb_{i\to a}\>\nonumber \\
&+\frac{\sqrt{\beta}}{2n} \sum_{a, i}\<\tbOmega_a,\fb_{i\to a}^{\otimes 2}\>
-\frac{d}{2n\sqrt{\beta}} \sum_{a}\<\tbOmega_a,\tbQ_a\>+\frac{\sqrt{\beta}}{2d} \sum_{a, i}\<\tbOmega_i,\tfb_{a\to i}^{\otimes 2}\>
-\frac{d}{2d\sqrt{\beta}} \sum_{a}\<\bOmega_i,\tbQ_i\>\nonumber\\
&-\sum_{i}\phi(\bm_i,\bQ_i)-\sum_{a}\tphi(\tbm_i,\tbQ_i)-\frac{d}{2\beta dn}\sum_{a,i}\<\tbQ_a,\bQ_i\>\, . \label{eq:LagrangianForm}
\end{align}
We next minimize with respect to the message variables $(\fb_{i\to a})$, $(\tfb_{a\to i})$. The first order stationarity conditions read
\begin{align}
X_{ai} \tfb_{a\to i} & = X_{ai}\tbr_a-\frac{1}{n}\tbOmega_a \fb_{i\to a}\, ,\label{eq:Fstationarity}\\
X_{ai} \fb_{i\to a} & = X_{ai}\br_i -\frac{1}{d}\bOmega_i \tfb_{a\to i}\, . \label{eq:tFstationarity}
\end{align}
In particular these imply that $\tfb_{a\to i}  = \tbr_a+O(1/\sqrt{n})$ and $\tfb_{a\to i}  = \tbr_a+O(1/\sqrt{n})$.
Multiplying the first of these equations by $\fb_{i\to a}$ and the second by $\tfb_{a\to i}$, and summing over $i,a$
we obtain
\begin{align}
\sum_{a,i}X_{ai} \<\tfb_{a\to i},\fb_{i\to a}\> =& \frac{1}{2} \sum_{i,a}X_{ai}\Big(\<\fb_{i\to a},\tbr_a\>+\<\tfb_{a\to i},\br_i\>\Big)
-\frac{1}{2n} \sum_{i,a}\<\tbOmega_a,\fb_{i\to a}^{\otimes 2}\>-\frac{1}{2d} \sum_{i,a}\<\bOmega_i,\tfb_{a\to i}^{\otimes 2}\>\nonumber\\
=&\frac{1}{2} \sum_{i,a}X_{ai}\Big(\<\fb_{i\to a},\tbr_a\>+\<\tfb_{a\to i},\br_i\>\Big)
-\frac{1}{2n} \sum_{i,a}\<\tbOmega_a,\br_i^{\otimes 2}\>-\frac{1}{2d} \sum_{i,a}\<\bOmega_i,\tbr_{a}^{\otimes 2}\>+O(n^{1/2})\, .
\end{align}
Further, multiplying Eqs.~(\ref{eq:Fstationarity}), (\ref{eq:tFstationarity})  respectively by $\br_i$ and $\tbr_a$, we get
\begin{align}
\frac{1}{2} \sum_{i,a}X_{ai}\Big(\<\fb_{i\to a},\tbr_a\>&+\<\tfb_{a\to i},\br_i\>\Big) = \sum_{a,i} X_{ai}\<\tbr_a,\br_i\>-\frac{1}{2n}\sum_{a,i}\<\br_i,\tbOmega_a \fb_{i\to a}\>
-\frac{1}{2d}\sum_{a,i}\<\tbr_a,\bOmega_i \tfb_{a\to i}\>\nonumber\\
&= \sum_{a,i} X_{ai}\<\tbr_a,\br_i\>-\frac{1}{2n}\sum_{a,i}\<\tbOmega_a ,\br_i^{\otimes 2}\>
-\frac{1}{2d}\sum_{a,i}\<\bOmega_i ,\tbr_{a}^{\otimes 2}\>+O(n^{1/2})\, .
\end{align}
Substituting the last two expressions in Eq.~(\ref{eq:LagrangianForm}), we obtain
\begin{align}
\cL = & \; \<\br,\bm\>+\<\tbr,\tbm\>-\sqrt{\beta}\<\tbr,\bX\br\>+\frac{\sqrt{\beta}}{2n} \sum_{a, i}\<\tbOmega_a,\br_{i}^{\otimes 2}\>
+\frac{\sqrt{\beta}}{2d} \sum_{a, i}\<\bOmega_i,\tbr_{a}^{\otimes 2}\>
-\frac{d}{2n\sqrt{\beta}} \sum_{a}\<\tbOmega_a,\tbQ_a\>\nonumber\\
&-\frac{d}{2d\sqrt{\beta}} \sum_{i}\<\bOmega_i,\bQ_i\>
-\sum_{i}\phi(\bm_i,\bQ_i)-\sum_{a}\tphi(\tbm_i,\tbQ_i)-\frac{d}{2\beta dn}\sum_{a,i}\<\tbQ_a,\bQ_i\>+O(n^{1/2})\, . \label{eq:LagrangianForm2}
\end{align}
Setting $\bQ_i = \bQ$ independent of $i$, 
 $\tbQ_a = \tbQ$ independent of $a$, defining $\bOmega = d^{-1}\sum_{i=1}^d \bOmega_i$, $\tbOmega = n^{-1}\sum_{a=1}^n \tbOmega_a$,
and neglecting $o(n)$ terms, we get 
\begin{align}
\begin{split}
\tilde{\cF}_{\sTAP} = & \frac{d}{2}\|\bX\|_{F} -\sqrt{\beta}\Tr\left(\bX\br\tbr^{\sT}\right) +\Tr(\br^{\sT}\bm) +\Tr(\tbr^{\sT}\tbm) -\frac{d}{2\sqrt{\beta}}\Tr(\bQ\bOmega) -\frac{d}{2\sqrt{\beta}}\Tr(\tbQ\tbOmega)\\
&-\sum_{a=1}^n \tphi(\tbm_a, \tbQ) - \sum_{i=1}^d\phi(\bm_i, \bQ)
+\frac{\sqrt{\beta}}{2}\sum_{i=1}^d\<\tbOmega,\br_i^{\otimes 2}\>+\frac{\sqrt{\beta}}{2}\sum_{a=1}^n\<\bOmega,\tbr_a^{\otimes 2}\>\\
&-\frac{d}{2\beta}\<\bQ,\tbQ\> \, .
\end{split}
\end{align}
Finally, the expression (\ref{eq:FreeEnergy_TAP_TM}) is recovered by using the stationarity conditions with respect to
$\bOmega$ and $\tbOmega$, which imply $\bQ = (\sqrt{\beta}/d)\sum_{a=1}^n\tbr_a^{\otimes 2}$ and $\tbQ = (\sqrt{\beta}/d)\sum_{i=1}^d\br_i^{\otimes 2}$,
and maximizing with respect to $\bm$, $\tbm$.

\subsection{Gradient of the TAP free energy}

From the definition of the partial Legendre transforms $\psi(\br,\bQ)$, $\tpsi(\tbr,\tbQ)$, the following 
derivatives hold
\begin{align}
\frac{\partial\psi}{\partial\br}(\br,\bQ) = \bm(\br,\bQ)\, ,\;\;\;\;\;\;\;
\frac{\partial\psi}{\partial\bQ}(\br,\bQ) = -\frac{1}{2\beta}\sG\big(\bm(\br,\bQ),\bQ\big)\, ,
\end{align}
where $\bm(\br,\bQ)\in\reals^k$ is the unique solution of
\begin{align}
\br = \frac{1}{\sqrt{\beta}}\,\sF(\bm;\bQ) \,.\label{eq:mrApp}
\end{align}
Using these derivatives we can compute the gradient of the free energy
\begin{align}
\frac{\partial \cF_{\sTAP}}{\partial \br_i}(\br,\tbr) & = -\sqrt{\beta} (\bX^{\sT}\tbr)_i+\bm_i -\frac{\beta}{d}\sum_{a=1}^n\<\tbr_a,\br_i\>\, \tbr_a
+\frac{1}{d}\sum_{a=1}^n\tsG(\tbm_a,\tbQ)\br_i\nonumber\\
& = -\sqrt{\beta} (\bX^{\sT}\tbr)_i+\bm_i + \sqrt{\beta}\tbOmega\br_i\, ,\label{eq:GradTAP1}\\
\frac{\partial \cF_{\sTAP}}{\partial \tbr_a}(\br,\tbr) & =  -\sqrt{\beta} (\bX\br)_a+\tbm_a -\frac{\beta}{d}\sum_{i=1}^d\<\tbr_a,\br_i\>\, \br_i
+\frac{1}{d}\sum_{i=1}^d\sG(\bm_i,\bQ)\tbr_a\nonumber\\
& = -\sqrt{\beta} (\bX\br)_a+\tbm_a +\sqrt{\beta}\bOmega\tbr_a\, ,\label{eq:GradTAP2}
\end{align}
where $\bm_i = \bm(\br_i,(\beta/d)\sum_{a\le n}\tbr_a^{\otimes 2})$, $\tbm_a = \tbm(\tbr_a,(\beta/d)\sum_{i\le d}\br_i^{\otimes 2})$, 
are defined as above,  $\bQ=(\beta/d)\sum_{a\le n}\tbr_a^{\otimes 2}$, $\bQ=(\beta/d)\sum_{i\le d}\tbr_i^{\otimes 2}$, and
\begin{align}
\bOmega &= \frac{1}{d\sqrt{\beta}}\sum_{i=1}^d\big\{\sG(\bm_i,\bQ)-\sF(\bm_i,\bQ)^{\otimes 2}\big\}\, ,\\
\tbOmega &= \frac{1}{d\sqrt{\beta}}\sum_{a=1}^n\big\{\tsG(\tbm_a,\tbQ)-\tsF(\bm_a,\tbQ)^{\otimes 2}\big\}\, .
\end{align}
\begin{remark}
We can express $\br$, $\tbr$ in terms of $\bm$, $\tbm$ in Eqs.~(\ref{eq:GradTAP1}), (\ref{eq:GradTAP2}) by using Eq.~(\ref{eq:mrApp})
\begin{align}
\bm &= \bX^{\sT}\,\tsF(\tbm;\tbQ)-\sF(\bm;\bQ) \tbOmega\, ,\;\;\;\;\;\;
\tbm = \bX\,\sF(\bm;\bQ)-\tsF(\tbm;\tbQ) \bOmega\, ,\label{eq:CriticalTAP_1}\\
\bQ &= \frac{1}{d}\sum_{a=1}^n \tsF(\tbm_a;\tbQ)^{\otimes 2}\, ,\;\;\;\; \;\;\;\; \;\;\;\tbQ = \frac{1}{d}\sum_{i=1}^d \sF(\bm_i;\bQ)^{\otimes 2}\, . \label{eq:CriticalTAP_2}
\end{align}
These coincide with the fixed point of the AMP algorithm in Section \ref{sec:TAP-Topic}.
\end{remark}

\subsection{Uninformative critical point: Proof of Lemma \ref{lemma:Uninf_TAP}}
\label{app:UninformativeTAP}

Consider the stationarity conditions (\ref{eq:CriticalTAP_1}) and (\ref{eq:CriticalTAP_2}), together with the definitions of Eqs.~(\ref{eq:OmegaTAP1}),
(\ref{eq:OmegaTAP2}). Since these are invariant under permutations of the topics, they admit a solution of the form
$\bm = \bv \bfone^{\sT}_k$, $\tbm = \tbv \bfone^{\sT}_k$, $\bQ = q_0\bJ_k+q_0'\id_k$, $\tbQ = \tq_0\bJ_k+\tq_0'\id_k$.
Using Eq.~(\ref{eq:CriticalTAP_2}) and Lemma \ref{lemma:UsefulFormulae}, Eqs.~(\ref{eq:sfsimplifiedsymm}), (\ref{eq:tsfsimplifiedsymm}),
we get $q_0' = \tq'_0=0$. 

Substituting this in Eqs.~(\ref{eq:OmegaTAP1}), (\ref{eq:OmegaTAP2}), and using again Lemma \ref{lemma:UsefulFormulae}, we get
\begin{align}
\bOmega = \sqrt{\beta}\, \id_k\, ,\;\;\;\;\;\;\;\; \tbOmega = \frac{\sqrt{\beta}\delta}{k(k\nu+1)} \bPp\, ,
\end{align}
where we recall that $\bPp= \id_k -\bfone_k\bfone_k/k$. Substituting these in Eq.~(\ref{eq:CriticalTAP_1}), we obtained that this is satisfied provided
$\bv, \tbv$ are given as in Eqs.~(\ref{eq:UninfTAP_1}), (\ref{eq:UninfTAP_2}). Finally,  $q_0$, $\tq_0$ are fixed by substituting in Eq.~(\ref{eq:CriticalTAP_2}). 

\section{State evolution analysis}

\subsection{State evolution equations}

Note that there is an alternative way to express the state evolution recursion in Eqs.~(\ref{eq:FirstSE}), (\ref{eq:SecondSE}).
Given a probability measure $p$ on $\reals^k$ and a matrix $\bM\succeq 0$, $\bM\in\reals^{k\times k}$, we define
the minimum mean square error
\begin{align}
\mmse(\bM;p)\equiv \inf_{\hbx(\, \cdot\, )}\,\E\Big\{[\bx-\hbx(\by)][\bx-\hbx(\by)]^{\sT}\Big\}\, ,
\end{align}
where the expectation is with respect to $\bx\sim p(\,\cdot\,)$ and $\by = \bM^{1/2}\bx+\bz$ for $\bz\sim\normal(0,\id_k)$. 
The infimum is understood in the positive semidefinite order, and it is achieved by $\hbx(\by ) = \E \{\bx|\by\}$.
We then rewrite Eqs.~(\ref{eq:FirstSE}), (\ref{eq:SecondSE}) as
\begin{align}
\bM_{t+1} & =\beta\delta\, \Big\{\mmse(0;\tq_0)-\mmse(\tbM_t;\tq_0)\Big\}\, ,\label{eq:FirstSE_b}\\
\tbM_{t} & = \beta\, \Big\{\mmse(0;q_0)-\mmse(\bM_t;q_0)\Big\}\, .\label{eq:SecondSE_b}
\end{align}
\subsection{Uninformative fixed point}
\label{app:SE_FP}

\begin{lemma}
\label{lemma:symmfixedpointSE}
The state evolution recursion in \eqref{eq:FirstSE}, \eqref{eq:SecondSE} admit uninformative fixed point of the form
\begin{align}
\label{eq:fixedpoint}
\begin{split}
&\tbM^* = \rho_0\bJ_k,\quad\quad \rho_0 = \frac{\delta\beta^2}{k\delta\beta + k^2},\\
&\bM^* = \frac{\delta\beta}{k^2}\bJ_k.
\end{split}
\end{align}
\end{lemma}
\begin{proof}
First note that for this value of $\tbM^*$, $\tbM^*\bw+\tbM^{*^{1/2}}\bz = y\bfone_k$ for some (random) $y$. Hence,
using Eq.~(\ref{eq:tsfsimplifiedsymm})
\begin{align}
\delta\,  \E\Big\{\tsF(\tbM^*\bw+\tbM^{*^{1/2}}\bz;\tbM^*)^{\otimes 2}\Big\} = \frac{\delta\beta}{k^2}\bJ_k = \bM^*.
\end{align}
In addition, using the explicit form (\ref{eq:sFexplicit})
\begin{align}
 \E\Big\{\sF(\bM^*\bh+\bM^{*^{1/2}}\bz;\bM^*)^{\otimes 2}\Big\} =  \beta(\id_k+\bM^*)^{-1}\bM^* = \frac{\beta^2\delta}{k^2}\left(\id_k+\frac{\delta\beta}{k^2}\bJ_k\right)^{-1}\bJ_k = \rho_0\bJ_k = \tbM^*.
\end{align}
Hence, the pair $\bM^*,\tbM^*$ in \eqref{eq:fixedpoint} is a fixed point for the iterations in \eqref{eq:FirstSE}, \eqref{eq:SecondSE}.
\end{proof}

\subsection{Stability of state evolution and proof of Theorem \ref{thm:StateEvolStable}}
\label{sec:StabilitySE}

The following theorem characterizes the region of parameters in which the uninformative fixed point of the state evolution iterations
in Lemma \ref{lemma:symmfixedpointSE} is stable.

\begin{theorem}
Consider the state evolution equations in \eqref{eq:FirstSE}, \eqref{eq:SecondSE}.
The uninformative symmetric fixed point of these equations is stable if and only if
\begin{align}
\beta < \beta_{\sp} = \frac{k(k\nu+1)}{\sqrt{\delta}}.
\end{align}
\end{theorem}
\begin{proof}
We linearize Eqs.~\eqref{eq:FirstSE}, \eqref{eq:SecondSE}
around the fixed point in \eqref{eq:fixedpoint} by setting $\bM_t = \bM_*+\bDelta_t$, $\tbM_t = \tbM_*+\tbDelta_t$ and expanding 
Eqs.~\eqref{eq:FirstSE}, \eqref{eq:SecondSE} to first order in $\bDelta, \tbDelta_t$.
First note that Eq.~\eqref{eq:SecondSE} takes the explicit form
\begin{align}
\tbM_t = \beta(\id_k+\bM_t)^{-1}\bM_t\, .
\end{align}
Hence, expanding to linear order we get
\begin{align}
\tbDelta_t = \beta\left(\id_k + \frac{\delta\beta}{k^2}\bJ_k\right)^{-1}\bDelta_t \left(\id_k + \frac{\delta\beta}{k^2}\bJ_k\right)^{-1}+o(\bDelta_t)\, . \label{eq:LinearizationTilde}
\end{align}
In the following, we shall decompose $\bDelta_t$ and $\tbDelta_t$ in the components\ along $\bfone_k$ and the ones orthogonal
\begin{align}
\begin{split}
\bDelta_t & = \delta_t\, \bP+\bDelta_t^{(1)}+\bDelta_t^{(2)}\, ,\\
\bDelta_t^{(1)} & = \bP\bDelta_t\bPp+ \bPp\bDelta_t\bP\, ,\\
\bDelta_t^{(2)} & = \bPp\bDelta_t\bPp\, ,
\end{split}\label{eq:DecompositionDelta}
\end{align}
and similarly for $\tbDelta_t$. Note that the linearization (\ref{eq:LinearizationTilde}) preserves these
subspaces
\begin{align}
\tdelta_t &= \beta \left(1+\frac{\delta\beta}{k}\right)^{-2} \delta_t+o(\bDelta_t)\, ,\label{eq:Linearization1}\\
\tbDelta^{(1)}_t &= \beta  \left(1+\frac{\delta\beta}{k}\right)^{-1} \bDelta^{(1)}_t+o(\bDelta_t)\, ,\\
\tbDelta^{(2)}_t &= \beta \, \bDelta^{(2)}_t+o(\bDelta_t)\, . \label{eq:Linearization3}
\end{align}

Next we consider Eq.~(\ref{eq:FirstSE}). We compute the value of
\begin{align}
f_{\bw,\bz} &= \tsF(\tbM_t\bw+\tbM_t^{1/2}\bz;\tbM_t)\\
&= \sqrt{\beta}\frac{\int\bw_1\exp\left\{\left\langle \tbM_t\bw+\tbM_t^{1/2}\bz,\bw_1\right\rangle -\frac{1}{2}\left\langle\bw_1,\tbM_t\bw_1\right\rangle\right\}\tq_0(\de \bw_1)}{\int\exp\left\{\left\langle \tbM_t\bw+\tbM_t^{1/2}\bz,\bw_1\right\rangle -\frac{1}{2}\left\langle\bw_1,\tbM_t\bw_1\right\rangle\right\}\tq_0(\de \bw_1)} = \sqrt{\beta}\frac{A_{\bw,\bz}}{B_{\bw,\bz}}.
\end{align}
for $\bw\in \sP_1(k)$. We have
\begin{align}
&\tbM_{t}\bw = \rho_0\bfone_k + \tbDelta^t\bw,\\
&\left\langle\bw_1, \tbM_{t}\bw_1\right\rangle = \rho_0 +\left\langle\bw_1, \tbDelta^t\bw_1\right\rangle.
\end{align}
Hence,
\begin{align}
A_{\bw,\bz} &=\int\bw_1\exp\left\{\left\langle \rho_0\bfone_k + \tbDelta^t\bw+\left(\rho_0 \bJ_k + \tbDelta^t\right)^{1/2}\bz, \bw_1\right\rangle - \frac{\rho_0}{2} - \frac{1}{2}\left\langle\bw_1, \tbDelta^t\bw_1\right\rangle\right\}\tq_0(\de \bw_1)\\
&= \int\bw_1\exp\left\{\frac{\rho_0}{2} + \left\langle\bw_1,\tbDelta^t\bw\right\rangle - \frac{1}{2}\left\langle \bw_1, \tbDelta^t\bw_1\right\rangle + \sqrt{\rho_0/ k} \left\langle \bJ_k \bz, \bw_1\right\rangle + \left\langle\bC_\bDelta^t\bz, \bw_1\right\rangle \right\}\tq_0(\de \bw_1)
\end{align}
where $\bC_\bDelta^t \equiv \left(\rho_0\bJ_k + \tbDelta^t\right)^{1/2} - (\rho_0/k)^{1/2}\bJ_k$. Therefore, we have
\begin{align}
A_{\bw,\bz} = a \int\bw_1\exp\left\{\left\langle\bw_1,\tbDelta^t\bw\right\rangle - \frac{1}{2}\left\langle \bw_1, \tbDelta^t\bw_1\right\rangle +\left\langle\bC_\bDelta^t\bz, \bw_1\right\rangle \right\}\tq_0(\de \bw_1)
\end{align}
where $a = \exp\left\{ \rho_0/2 + \sqrt{\rho_0/k}\left\langle\bz, \bfone_k\right\rangle\right\}$. Expanding 
the exponential, we get
\small
\begin{align}
A_{\bw,\bz} = a\int\bw_1\left\{1+\left\langle\bw_1, \tbDelta^t\bw\right\rangle - \frac{1}{2}\left\langle\bw_1, \tbDelta^t\bw_1\right\rangle +\left\langle\bz, \bC_\bDelta^t\bw_1\right\rangle + \frac{1}{2}\left\langle\bz,\bC_\bDelta^t\bw_1\right\rangle^2 + o\left(\tbDelta^t\right)\right\}\tq_0(\de \bw_1).
\end{align}
\normalsize
Thus,
\footnotesize
\begin{align}
A_{\bw,\bz} = a\left(\frac{1}{k}\bfone_k + \bS\tbDelta^t\bw - 
\frac{1}{2}\begin{pmatrix} \left\langle \tbDelta^t, \bT_1\right\rangle  \\  \left\langle \tbDelta^t, \bT_2\right\rangle \\ \vdots \\ \left\langle \tbDelta^t, \bT_k\right\rangle \end{pmatrix} 
+ \bS\bC_\bDelta^t\bz + \frac{1}{2} \begin{pmatrix} \left\langle \bC_\bDelta^t\bz^{\otimes2}\bC_\bDelta^t, \bT_1\right\rangle \\ \left\langle \bC_\bDelta^t\bz^{\otimes2}\bC_\bDelta^t, \bOmega_2^\prime\right\rangle \\ \vdots \\ \left\langle \bC_\bDelta^t\bz^{\otimes2}\bC_\bDelta^t, \bT_k\right\rangle \end{pmatrix} + o\left(\tbDelta^t\right)\right) 
\end{align}
\normalsize
where $\bS, \bT \in \reals^{k\times k}$ are the moment tensors 
\begin{align}
&\bS = \int\bw_1^{\otimes 2}\tq_0(\de \bw_1) = \frac{\nu}{k\nu(k\nu+1)}\left(\id_k+\nu\bJ_k\right)= \frac{1}{k(k\nu+1)}\bPp+ \frac{1}{k}\bP\, ,\label{eq:S_Formula}\\
&\bT = \int\bw_1^{\otimes 3}\tq_0(\de \bw_1), \;\label{eq:bSdef}\\
&(T_i)_{jl} = \frac{1}{k\nu(k\nu+1)(k\nu+2)}.
\begin{cases}
\nu(\nu+1)(\nu+2)\quad \text{if}\; j = l = i,\\
\nu^2(\nu+1)\quad \text{if}\; j = i,\, l\neq i\; \text{or}\; l=i,\, j\neq i\; \text{or}\; l=j, j\neq i,\\
\nu^3\quad \text{otherwise}.
\end{cases}\label{eq:bTdef}
\end{align}
Similarly, we have
\small
\begin{align}
B_{\bw,\bz} = a\int\left\{1 + \left\langle \bw_1, \tbDelta^t\bw\right\rangle - \frac{1}{2}\left\langle\bw_1, \tbDelta^t\bw_1\right\rangle + \left\langle\bz, \bC_\bDelta^t\bw_1\right\rangle + \frac{1}{2}\left\langle\bz, \bC_\bDelta^t\bw_1\right\rangle^2 + o\left(\tbDelta^t\right)\right\}\tq_0(\de \bw_1).
\end{align}
\normalsize
Therefore,
\begin{align}
B_{\bw,\bz} = a\left(1 + \frac{1}{k}\left\langle\bfone_k\otimes\bw, \tbDelta^t\right\rangle - \frac{1}{2}\left\langle\bS, \tbDelta^t\right\rangle + \frac{1}{k}\left\langle\bfone_k\otimes\bz,\bC_\bDelta^t\right\rangle + \frac{1}{2}\left\langle\bz, \bC_\bDelta^t\bS\bC_\bDelta^t\bz\right\rangle + o\left(\tbDelta^t\right)\right).
\end{align}
Hence, we can write
\scriptsize
\begin{align}
f_{\bw,\bz} = \sqrt{\beta}\frac{A_{\bw,\bz}}{B_{\bw,\bz}} 
&= \sqrt{\beta}\Bigg(\frac{1}{k}\bfone_k 
+ \bS\tbDelta^t\bw 
- \frac{1}{2}\begin{pmatrix} \left\langle \tbDelta^t, \bT_1\right\rangle \\ \left\langle \tbDelta^t, \T_2\right\rangle \\ \vdots \\ \left\langle \tbDelta^t, \bT_k\right\rangle \end{pmatrix} 
+ \bS\bC_\bDelta^t\bz + \frac{1}{2}\begin{pmatrix} \left\langle   \bC_\bDelta^t\bz^{\otimes2}\bC_\bDelta^t, \bT_1\right\rangle \\  \left\langle \bC_\bDelta^t\bz^{\otimes2}\bC_\bDelta^t,    \bT_2\right\rangle \\ \vdots \\ \left\langle
    \bC_\bDelta^t\bz^{\otimes2}\bC_\bDelta^t, \bT_k\right\rangle\end{pmatrix} \\
&- \frac{1}{k^2}\left\langle\bfone_k\otimes\bw, \tbDelta^t\right\rangle\bfone_k 
+ \frac{1}{2k}\left\langle\bS, \tbDelta^t\right\rangle\bfone_k  \nonumber
- \frac{1}{k^2}\left\langle\bfone_k\otimes\bz,\bC_\bDelta^t\right\rangle\bfone_k 
- \frac{1}{2k}\left\langle\bz, \bC_\bDelta^t\bS\bC_\bDelta^t\bz\right\rangle\bfone_k \\
&- \frac{1}{k} \left\langle\bfone_k\otimes \bz, \bC_\bDelta^t\right\rangle\bS\bC_\bDelta^t\bz 
- \frac{1}{k^3}\left\langle\bfone_k\otimes \bz, \bC_\bDelta^t\right\rangle^2\bfone_k
+ o\left(\tbDelta^t\right)\Bigg).
\end{align}
\normalsize
Therefore, linearizing Eq.~(\eqref{eq:FirstSE}), we get (below, we denote by $[\bA]_s$ the symmetric part of matrix $\bA$,
namely $[\bA]_s = (\bA+\bA^{\sT})/2$)
\scriptsize
\begin{align}
\bDelta_{t+1}&= \delta\E_{\bw,\bz}\left(f_{\bw,\bz}^{\otimes2}\right) -\frac{\delta\beta}{k^2}\bJ_k \\
&= 
\delta\beta\Bigg(
\frac{2}{k^2}\big[\bS(\tbDelta^t-(\bC_{\bDelta}^t)^2)\bJ_k\big]_s 
- \frac{1}{2k}\begin{pmatrix} \left\langle \tbDelta^t-(\bC_{\bDelta}^t)^2, \bT_1\right\rangle \\ \left\langle \tbDelta^t-(\bC_{\bDelta}^t)^2, \bT_2\right\rangle \\ \vdots \\ \left\langle \tbDelta^t-(\bC_{\bDelta}^t)^2, \bT_k\right\rangle \end{pmatrix}\otimes\bfone_k
- \frac{1}{2k}\bfone_k\otimes\begin{pmatrix} \left\langle \tbDelta^t-(\bC_{\bDelta}^t)^2, \bT_1\right\rangle \\ \left\langle \tbDelta^t-(\bC_{\bDelta}^t)^2, \bT_2\right\rangle \\ \vdots \\ \left\langle {\tbDelta^t}-(\bC_{\bDelta}^t)^2, \bT_k\right\rangle \end{pmatrix}\label{eq:DeltaIteration}\\
&- \frac{2}{k^4}\left\langle\bJ_k,\tbDelta^t\right\rangle\bJ_k 
+ \frac{1}{k^2}\left\langle\bS,\tbDelta^t-(\bC_\bDelta^t)^2\right\rangle\bJ_k\nonumber
\\
&- \frac{2}{k^4}\left\langle\bJ_k, (\bC_{\bDelta}^t)^2\right\rangle\bJ_k
+ \bS(\bC_{\bDelta}^t)^2\bS
- \frac{2}{k^2}\big[\bS(\bC_{\bDelta}^t)^2\bJ_k \big]_s
+ \frac{1}{k^4}\left\langle\bJ_k, (\bC_{\bDelta}^t)^2\right\rangle\bJ_k
+ o(\tbDelta_t)\Bigg).\nonumber
\end{align}
\normalsize

We next decompose  $\tbDelta_t$ in the component along $\bJ_k$ and the one orthogonal, as per Eq.~(\ref{eq:DecompositionDelta}), and 
note that
\begin{align}
\bC_{\bDelta}^t &=\Big((k\rho_0+\tdelta_t)\bP+\tbDelta_t^{(1)}+\tbDelta^{(2)}_t\Big)^{1/2}-(k\rho_0)^{1/2}\bP\\
&=\sqrt{k\rho_0+\tdelta_t}\, \bP +\big(\tbDelta_t^{(2)}\big)^{1/2}-\sqrt{k\rho_0}\, \bP+O(\tbDelta_t) -(k\rho_0)^{1/2}\bP = \big(\tbDelta_t^{(2)}\big)^{1/2}+O(\tbDelta_t) \, ,
\end{align}
whence
\begin{align}
(\bC_{\bDelta}^t)^2 = \tbDelta_t^{(2)}+o(\bDelta)\, .
\end{align}
Using this identity together with Eqs.~(\ref{eq:bSdef}), (\ref{eq:bTdef}) in Eq.~(\ref {eq:DeltaIteration}) we get
\begin{align}
\delta_{t+1} & = o(\tbDelta_t)\, ,\\
\bDelta^{(1)}_{t+1} &= o(\tbDelta_t)\, ,\\
\bDelta^{(2)}_{t+1} &= \frac{\beta\delta}{k^2(k\nu+1)^2}\, \tbDelta^{(2)}_t+o(\tbDelta_t)\, .
\end{align}

Together with Eqs.~(\ref{eq:Linearization1}) to (\ref{eq:Linearization3}), these yield 
\begin{align}
\delta_{t+1} & = o(\bDelta_t)\, ,\\
\bDelta^{(1)}_{t+1} &= o(\bDelta_t)\, ,\\
\bDelta^{(2)}_{t+1} &= \frac{\beta^2\delta}{k^2(k\nu+1)^2}\, \bDelta^{(2)}_t +o(\tbDelta_t)\, .
\end{align}
Hence the uninformative fixed point is stable if and only if 
\begin{align}
\beta \leq \frac{k(k\nu+1)}{\sqrt{\delta}}.
\end{align}
Note that this is the same condition as the spectral threshold. 
\end{proof}

\subsection{Stability of the uninformative point: Proof of Theorem \ref{thm:StabilityTAP}}

In this section we compute the Hessian of the TAP free energy around the uninformative stationary point.
We will establish a second order approximation of $\tcF_{\sTAP}(\br,\tbr)$ near the stationary point. 
Namely, we denote by $\br^*_i = r^*_i\bfone_k$, $\tbr^*_a = \tr^*_a\bfone_k$ the uninformative stationary point,
and by $\bm^*_i = m^*_i\bfone_k$,  $\tbm^*_a = \tm^*_a\bfone_k$ the dual variables, where
\begin{align}
m^*_i &= \frac{\sqrt{\beta}}{k} (\bX^{\sT}\bfone_n)_i\, ,\;\;\;\;\;\;\; \tm^*_a = \frac{\beta}{k(1+kq_0)} (\bX\bX^{\sT}\bfone_n)_a-\frac{\beta}{k+\delta\beta}\, ,\\
r^*_i &=\frac{\sqrt{\beta}}{k(1+kq_0^*)} (\bX^{\sT}\bfone_n)_i\, ,\;\;\;\;\;\;\; \tr^*_a =\frac{1}{k}\,.
\end{align}
For any other assignment of the  variables, $\br,\tbr$, $\bm,\tbm$, we introduce the decomposition
\begin{align}
\br_i &= r_i^s\bfone_k+\bdelta_i \, ,\;\;\;\;\;\;\;\tbr_a= \tr_a^s\bfone_k+\tbdelta_a\, ,\label{eq:ExpansionP1}\\
r_i^s &= r_i^*+\delta^s_i\, ,\;\;\;\;\;\;\;\;\;\; \tr_a^s = \tr_a^*+\tdelta^s_a\, ,\\
\bm_i &= m_i^s\bfone_k+\bfeta_i \, ,\;\;\;\;\;\;\;\tbm_a= \tm_a^s\bfone_k+\tbfeta_a\, ,\\
m_i^s &= m_i^*+\eta^s_i\, ,\;\;\;\;\;\;\;\;\;\; \tm_a^s = \tm_a^*+\teta^s_a\, ,\label{eq:ExpansionP2}
\end{align}
where $\<\bdelta_i,\bfone_k\>=\<\tbdelta_a,\bfone_k\>=\<\bfeta_i,\bfone_k\>=\<\tbfeta_a,\bfone_k\>=0$.
Note that, by construction $\tr_a^s = 1/k$.

We will establish an expansion of the form
\begin{align}
\cF_{\sTAP}(\br,\tbr) = \tcF_{\sTAP}(\br^*,\tbr^*) + 
\cF^{(2)}_{\sTAP}(\bdelta,\tbdelta,\delta^s,\tdelta^s) +o(\delta^2)\, ,\label{eq:TAP_Expansion}
\end{align}
where $\cF^{(2)}_{\sTAP}$ is a quadratic function, and when using the $O(\,\cdot\, )$ notation, we implicitly consider  all $\delta,\eta$
parameters to be of the same order and use $\delta$ for denoting that order. Notice that the first-order term is missing 
from this expansion since $(\br^*,\tbr^*)$ is a stationary point. 

The crucial step in obtaining the expansion (\ref{eq:TAP_Expansion}) is to derive a second order expansion for the logarithmic
moment generating functions $\phi$, $\tphi$, and subsequently for the entropy functions $\psi$, $\tpsi$.
\begin{lemma}
Setting variables as per Eq.~(\ref{eq:ExpansionP1}), we have
\begin{align}
\phi\left(\bm_i,\frac{\beta}{d}\sum_{a=1}^n\tbr_a^{\otimes 2}\right) =& -\frac{1}{2}\log(1+ka_0) +\frac{k(m_i^s)^2}{2(1+ka_0)}+\frac{\beta^2(1+\beta\delta/k+k(m_i^*)^2)}{2d^2k(1+\beta\delta/k)^2}
\left\|\sum_{a=1}^n\tbdelta_a\right\|^2_2\label{eq:PhiFormula}\\
&-\frac{\beta m_i^*}{d(1+\beta\delta/k)}
\sum_{a=1}^n\<\bfeta_i,\tbdelta_a\>+\frac{1}{2}\|\bfeta_i\|_2^2 -\frac{\beta}{2d}\sum_{a=1}^n\|\tbdelta_a\|_2^2+o(\delta^2)\, ,\nonumber
\end{align}
where $a_0 = (\beta/d)\sum_{a=1}^n  (\tr^s_a)^2$.
\end{lemma}
\begin{proof}
Let $\bQ = (\beta/d)\sum_{a=1}^n\tbr_a^{\otimes 2}$, and define the orthogonal decomposition $\bQ = \bQ_0 +\bQ_1+\bQ_2$,
where $\bQ_0= \bP\bQ\bP$, $\bQ_1= \bP\bQ\bPp+\bPp\bQ\bP$, $\bQ_2= \bPp\bQ\bPp$. Using the representation (\ref{eq:ExpansionP1}),
we get
\begin{align}
\bQ_0 & =a_0\bfone_k\bfone_k^{\sT}\, ,\;\;\;\;\;\;\;\;\;\;\; a_0 = \frac{\beta}{d}\sum_{a=1}^n  (\tr^s_a)^2\, ,\\
\bQ_1 & =\bfone_k\ba_1^{\sT}+\ba_1\bfone_k^{\sT}\, ,\;\;\;\;\;\; 
\ba_1 = \frac{\beta}{d}\sum_{a=1}^n \tr^s_a \tbdelta_a\, ,\\
\bQ_2 & =\frac{\beta}{d}\sum_{a=1}^n \tbdelta_a\tbdelta_a^{\sT}\, .
\end{align}
By Gaussian integration, we have 
\begin{align}
\phi(\bm_i,\bQ) = -\frac{1}{2}\Tr \log\big(\id+\bQ\big)+\frac{1}{2}\<\bm_i,(\id+\bQ)^{-1}\bm_i\>\, .\label{eq:PhiExact}
\end{align}
Expanding the logarithm, we get
\begin{align}
\Tr \log\big(\id+\bQ\big) =&\Tr \log\big(\id+\bQ_0\big) +\Tr\big\{(\id+\bQ_0)^{-1}(\bQ_1+\bQ_2)\big\}\nonumber\\
&-\frac{1}{2} \Tr\big\{(\id+\bQ_0)^{-1}\bQ_1 (\id+\bQ_0)^{-1}\bQ_1\big\} +o(\delta^2)\nonumber\\
 =&\Tr \log\big(\id+\bQ_0\big)+\Tr(\bQ_2)- \, \<\ba_1,(\id+\bQ_0)^{-1}\ba_1\>\, \<\bfone,(\id+\bQ_0)^{-1}\bfone\> +o(\delta^2)\nonumber\\
 = &\log(1+ka_0) + \frac{\beta}{d}\sum_{a=1}^n\|\tbdelta_a\|_2^2-\frac{k}{1+ka_0}\left\|\frac{\beta}{d}\sum_{a=1}^n \tr^s_a \tbdelta_a\right\|_2^2  +o(\delta^2)\nonumber\\
= &\log(1+ka_0) +\frac{\beta}{d}\sum_{a=1}^n\|\tbdelta_a\|_2^2-\frac{\beta^2}{kd^2(1+kq^*_0)}\left\|\sum_{a=1}^n \tbdelta_a\right\|_2^2  +o(\delta^2)
%
\end{align}
Considering next the second term in Eq.~(\ref{eq:PhiExact}), we get
\begin{align}
\<\bm_i,(\id+\bQ)^{-1}\bm_i\> = & (m_i^s)^2\<\bfone,(\id+\bQ_0+\bQ_1+\bQ_2)^{-1}\bfone\>+2 m_i^s\<\bfeta_i,(\id+\bQ_0+\bQ_1)^{-1}\bfone\>\nonumber\\
&+\<\bfeta_i,(\id+\bQ_0)^{-1}\bfeta_i\>+o(\delta^2)\nonumber \\
=&  (m_i^s)^2\<\bfone,(\id+\bQ_0)^{-1}\bfone\>+(m_i^s)^2\<\bfone,(\id+\bQ_0)^{-1}\bQ_1 (\id+\bQ_0)^{-1}\bQ_1 (\id+\bQ_0)^{-1}\bfone\>\nonumber\\
&-2 m_i^s\<\bfeta_i,(\id+\bQ_0)^{-1}\bQ_1 (\id+\bQ_0)^{-1}\bfone\>+\|\bfeta_i\|_2^2+o(\delta^2)\nonumber\\
=&  \frac{k(m_i^s)^2}{1+ka_0}+\frac{(km_i^s)^2}{(1+ka_0)^2}\|\ba_1\|^2_2-\frac{2km_i^s}{(1+ka_0)}\<\bfeta_i,\ba_1\>+\|\bfeta_i\|_2^2+o(\delta^2)\nonumber\\
=&  \frac{k(m_i^s)^2}{1+ka_0}+\frac{(\beta m_i^s)^2}{d^2(1+kq^*_0)^2}
\left\|\sum_{a=1}^n\tbdelta_a\right\|^2_2-\frac{2\beta m_i^s}{d(1+kq^*_0)}
\sum_{a=1}^n\<\bfeta_i,\tbdelta_a\>+\|\bfeta_i\|_2^2+o(\delta^2)\nonumber
\, .
\end{align}
\end{proof}

\begin{lemma}
Setting variables as per Eq.~(\ref{eq:ExpansionP1}), we have
\begin{align}
\tphi\left(\tbm_a,\frac{\beta}{d}\sum_{i=1}^d\br_i^{\otimes 2}\right) & =  \tm_a^s -\frac{1}{2}b_0 +\frac{1}{2k(k\nu+1)}\left\|\tbfeta_a-\frac{\beta}{d}\sum_{i=1}^d r^*_i \bdelta_i\right\|_2^2
-\frac{\beta}{2dk(k\nu+1)}\sum_{i=1}^d \|\bdelta_i\|_2^2 +o(\delta^2)\,,
\end{align}
where  $b_0 = (\beta/d)\sum_{i=1}^d  (r^s_i)^2$.
\end{lemma}
\begin{proof}
Let $\tbQ = (\beta/d)\sum_{i=1}^d\br_i^{\otimes 2}$ and, as in the previous proof, define the orthogonal decomposition $\tbQ = \tbQ_0 +\tbQ_1+\tbQ_2$,
where $\tbQ_0= \bP\tbQ\bP$, $\tbQ_1= \bP\tbQ\bPp+\bPp\tbQ\bP$, $\tbQ_2= \bPp\tbQ\bPp$. Using the representation (\ref{eq:ExpansionP1}),
we get
\begin{align}
\tbQ_0 & =b_0\bfone_k\bfone_k^{\sT}\, ,\;\;\;\;\;\;\;\;\;\;\; b_0 = \frac{\beta}{d}\sum_{i=1}^d  (r^s_i)^2\, ,\\
\tbQ_1 & =\bfone_k\bb_1^{\sT}+\bb_1\bfone_k^{\sT}\, ,\;\;\;\;\;\; 
\bb_1 = \frac{\beta}{d}\sum_{i=1}^d r^s_i \bdelta_i\, ,\\
\tbQ_2 & =\frac{\beta}{d}\sum_{i=1}^d \bdelta_i\bdelta_i^{\sT}\, .
\end{align}
For $\bw\in\supp(\tq_0)$, we have $\<\bfone,\bw\>=1$ and therefore
\begin{align}
\tphi(\tbm_a,\tbQ) & = \log\left\{\int e^{\<\tbm,\bw\>-\frac{1}{2} \<\bw,\tbQ\bw\>}\tq_0(\de\bw)\right\}\\
& = \tm_a^s -\frac{1}{2}b_0+\log\left\{\int e^{\<\tbfeta_a-\bb_1,\bw\>-\frac{1}{2} \<\bw,\tbQ_2\bw\>}\tq_0(\de\bw)\right\}\\
& = \tm_a^s -\frac{1}{2}b_0+ \frac{1}{2}\<(\tbfeta_a-\bb_1)(\tbfeta_a-\bb_1)^{\sT}-\tbQ_1,\bS_{\perp}\> +o(\delta^2)\, ,
\end{align}
where, cf. Eq.~(\ref{eq:S_Formula}),
\begin{align}
\bS_{\perp} = \int (\bPp\bw)^{\otimes 2}\tq_0(\de\bw) = \frac{1}{k(k\nu+1)} \bPp\, .
\end{align}
Hence, we obtain immediately the claim.
\end{proof}

We next transfer the above results on the moment generating functions $\phi$, $\tphi$, to analogous
results on the entropy functions $\psi$, $\tpsi$.
\begin{lemma}\label{lemma:PsiExp}
Setting variables as per Eq.~(\ref{eq:ExpansionP1}), we have
\begin{align}
\psi\left(\br_i,\frac{\beta}{d}\sum_{a=1}^n\tbr_a^{\otimes 2}\right) =& \frac{1}{2}\log(1+ka_0) + \frac{1}{2}k(1+ka_0)(r_i^s)^2-\frac{\beta^2(1+\beta\delta/k+k(m_i^*)^2)}{2d^2k(1+\beta\delta/k)^2}
\left\|\sum_{a=1}^n\tbdelta_a\right\|^2_2\\
&+\frac{1}{2}\left\|\bdelta_i+\frac{\beta m_i^*}{d(1+\beta\delta/k)}\sum_{a=1}^n\tbdelta_a\right\|_2^2 
+\frac{\beta}{2d}\sum_{a=1}^n\|\tbdelta_a\|_2^2+o(\delta^2)\, ,\nonumber
\end{align}
where $a_0 = (\beta/d)\sum_{a=1}^n  (\tr^s_a)^2$.
\end{lemma}
\begin{proof}
By definition
\begin{align}
\psi(\br_i,\bQ) = \max_{m_i^s,\bfeta_i}\big\{km_i^sr_i^s+\<\bfeta_i,\bdelta_i\>- \phi(\bm_i,\bQ)\big\}\, .
\end{align}
Since $\phi(\,\cdot\,,\bQ)$ is strongly convex, the maximum is realized when $\eta_i^s,\bfeta_i = O(\delta)$ and can be computed
order-by-order in $\delta$. Hence, substituting (\ref{eq:PhiFormula}) we obtain the claim. 
\end{proof}

\begin{lemma}\label{lemma:TPsiExp}
Setting variables as per Eq.~(\ref{eq:ExpansionP1}), we have
\begin{align}
\tpsi\left(\tbr_a,\frac{\beta}{d}\sum_{i=1}^d\br_i^{\otimes 2}\right) & = \frac{1}{2}b_0 
+\frac{1}{2}k(k\nu+1)\|\tbdelta_a\|^2_2+\frac{\beta}{d}\sum_{i=1}^d r^*_i \<\bdelta_i,\tbdelta_a\>
+\frac{\beta}{2dk(k\nu+1)}\sum_{i=1}^d \|\bdelta_i\|_2^2 +o(\delta^2)\, ,
\end{align}
where  $b_0 = (\beta/d)\sum_{i=1}^d  (r^s_i)^2$.
\end{lemma}
\begin{proof}
By definition
\begin{align}
\tpsi(\tbr_i,\tbQ) = \max_{\tm_i^s,\tbfeta_i}\big\{k\tm_i^s\tr_i^s+\<\tbfeta_i,\tbdelta_i\>- \tphi(\tbm_i,\tbQ)\big\}\, .
\end{align}
The proof is again obtained by maximizing order by order in $\delta$, and using $\tr_a^s = 1/k$.
\end{proof}

\begin{lemma}
Setting variables as per Eq. ~(\ref{eq:ExpansionP1}), and introducing the vectors $\br^s = (r_i^s)_{i\le d}\in\reals^d$,
$\tbr^s = (\tr_a^s)_{a\le n}\in\reals^n$, we obtain
\begin{align}
\cF_{\sTAP}(\br,\tbr)= &\cF_{\sTAP}^{(s)}(\br^s,\tbr^s) + \cF_{\sTAP}^{(a)}(\bdelta,\tbdelta) +o(\delta^2)\, ,\label{eq:F_SecondOrder}\\
\cF_{\sTAP}^{(s)}(\br^s,\tbr^s)= &  \frac{d}{2}\log\Big(1+\frac{\beta \delta}{k}\Big) + \frac{1}{2}k\Big(1+\frac{\beta \delta}{k}\Big)\|\br^s\|_2^2
-k\sqrt{\beta}\<\bfone, \bX\br^s\>\label{eq:SecondSymmetric}\, ,\\
\cF_{\sTAP}^{(a)} (\bdelta,\tbdelta) & = \frac{1}{2}\left(1+\frac{\beta\delta}{k(k\nu+1)}\right)\|\bdelta\|_F^2+\frac{1}{2}\big(\beta+k(k\nu+1)\big) \|\tbdelta\|_F^2-
\frac{\beta^2}{2dk(1+\beta\delta/k)}\left\|\sum_{a\le n}\tbdelta_a\right\|_2^2\nonumber\\
&-\sqrt{\beta}\Tr(\bX\bdelta\tbdelta^{\sT})+\frac{\beta}{d(1+\beta\delta/k)}\sum_{i\le d, a\le n} m_i^*\<\bdelta_i,\tbdelta_a\>\, .
\end{align}
\end{lemma}
\begin{proof}
Using the decomposition (\ref{eq:ExpansionP1}), we get 
\begin{align}
\Tr(\bX\br\tbr^{\sT}) &= k \Tr\big(\bX\br^{s}(\tbr^s)^{\sT}\big) +  \Tr(\bX\bdelta\tbdelta^{\sT})\, , \\
\sum_{i\le d, a\le n} \<\br_i,\tbr_a\>^2 & = k^2\sum_{i\le d, a\le n} (r^s_i)^2(\tr_a^s)^2+ 2k  \sum_{i\le d, a\le n} (r^s_i\tr_a^s)\<\bdelta_i,\tbdelta_a\> +o(\delta^2)\\
& = k^2\sum_{i\le d, a\le n} (r^s_i)^2(\tr_a^s)^2+ 2  \sum_{i\le d, a\le n} r^s_i\<\bdelta_i,\tbdelta_a\> +o(\delta^2)\, ,
\end{align}
where we used the fact that $\tr^s_a=1/k$. Using these, together with Lemma \ref{lemma:PsiExp}, \ref{lemma:TPsiExp} in 
Eq.~(\ref{eq:FreeEnergy_TAP_TM}), we get the decomposition (\ref{eq:F_SecondOrder}) where
\begin{align}
\cF_{\sTAP}^{(s)}(\br^s,\tbr^s)= &  \frac{d}{2}\log\Big(1+\frac{\beta k}{d}\,\|\tbr^s\|_2^2\Big) + \frac{1}{2}k^2\Big(1+\frac{\beta k}{d}\,\|\tbr^s\|_2^2\Big)\|\br^s\|_2^2
+\frac{1}{2}\beta\delta\|\br^s\|^2_2\nonumber\\
&-k\sqrt{\beta}\Tr(\bX\br^s(\tbr^s)^{\sT})-\frac{\beta k^2}{2d}\|\br^s\|_2^2\|\tbr^s\|_2^2\, ,
\end{align}
Substituting $\tbr^s= \bfone_n/k$, we obtain Eq.~(\ref{eq:SecondSymmetric}).
\end{proof}

Notice that $\cF_{\sTAP}^{(s)}(\br^s,\tbr^s)$ is a positive definite quadratic function in $\br^s$, minimized
at $\br^s = \br^*$. Hence, in order to establish the stability of the uninformative stationary point, it is sufficient to check
that the quadratic form $\cF_{\sTAP}^{(a)}(\bdelta,\tbdelta)$ is positive definite. The matrix representation of this quadratic form yields
\begin{align}
\Hess = \left[\begin{matrix}
\Big(1+\frac{\delta\beta}{k(k\nu+1)}\Big)\id_d & -\sqrt{\beta}\bX^{\sT}\Big(\id_n-\frac{\beta}{d(k+\delta \beta)}\allone_n\Big)\\
-\sqrt{\beta}\Big(\id_n-\frac{\beta}{d(k+\delta \beta)}\allone_n\Big)\bX & 
\big(\beta+k(k\nu+1)\big)\id_n-\frac{\beta^2}{d(k+\delta\beta)}\allone_n
\end{matrix}\right]\, . \label{eq:HessianFinalFormula}
\end{align}
We are left with the task of proving that $\Hess\succ \bzero$ for $\beta< \beta_{\sp}(k,\delta,\nu)$. We will use the following 
random matrix theory lemma.
\begin{lemma}\label{lemma:RMT_lemma}
Let $\bu\in\reals^n$, $\bv\in\reals^d$ be vectors with $\|\bu\|_2=\|\bv\|_2=1$, $\gamma,\alpha_{\|}$, $\alpha_{\perp}, \lbar\in\reals$ be numbers,
and let $\bPu = \bu\bu^{\sT}$ be the orthogonal projector onto $\bu$, and $\bPup = \id-\bu\bu^{\sT}$ be its orthogonal complement. 
Denote by $\bZ\in\reals^{n\times d}$ random matrices with $(Z_{ij})_{i\le n,j\le d}\sim\normal(0,1/d)$, with $n/d\to \delta\in(0,\infty)$ as $n\to\infty$, and define the matrix
\begin{align}
\bM = \gamma\bu\bv^{\sT} +\alpha_{\|}\bPu\bZ+\alpha_{\perp} \bPup\bZ\, .\label{eq:M_lemma}
\end{align}
Finally define $\gamma_*^2 \equiv (1+\sqrt{\delta})\alpha_{\perp}^2-\alpha_{\|}^2$, and 
\begin{align}
\lambda_*^2\equiv \begin{cases}
\frac{(\gamma^2+\alpha_{\|}^2)(\gamma^2+\alpha_{\|}^2-\alpha_{\perp}^2(1-\delta))}{\gamma^2+\alpha_{\|}^2-\alpha_{\perp}^2}
& \mbox{ if $\gamma^2>\gamma_*^2$,}\\
\alpha^2_{\perp}(1+\sqrt{\delta})^2 & \mbox{ otherwise.}\label{eq:LimitSingValue}
\end{cases}
\end{align}
Then, denoting by  $s_{\max}(\bM)$ the largest singular value of $\bM$, we have $\lim_{n\to\infty}s_{\max}(\bM) =\lambda_*$ in probability.
\end{lemma}
\begin{proof}
By rotational invariance of $\bZ$, we can and will assume $\bu=\be_1$, and will denote by $\tbZ\in\reals^{(n-1)\times d}$ the matrix
containing the last $(n-1)$ rows of $\bZ$. We further let $\bw = \gamma\bv+\alpha_{\|}\bZ^{\sT}\bu$. With these definitions,  
\begin{align}
\bM\bM^{\sT} 
= \left[\begin{matrix}
\|\bw\|_2^2 & \alpha_{\perp}(\tbZ\bw)^{\sT}\\
\alpha_{\perp}(\tbZ\bw)& \alpha_{\perp}^2\tbZ\tbZ^{\sT}
\end{matrix}\right]\, .
\end{align}
Note that, almost surely, $\lim_{n\to\infty}\lambda_{\max}(\tbZ\tbZ^{\sT}) = (1+\sqrt{\delta})^2$ \cite{BaiSilverstein},
and therefore $\lim\inf_{n\to\infty} s_{\max}(\bM)^2\ge  \alpha_{\perp}^2(1+\sqrt{\delta})^2$ almost surely.

Recall that, as long as $s_n^2$ is not an eigenvalue of $\alpha_{\perp}^2\tbZ\tbZ^{\sT}$, we have
\begin{align}
\det(s_n^2\id -\bM\bM^{\sT}) = \det(s_n^2\id -\alpha_{\perp}^2\tbZ\tbZ^{\sT})\, \Big\{s_n^2 - \|\bw\|_2^2-
\alpha_{\perp}^2\<\bw,\tbZ^{\sT}(s_n^2\id-\alpha_{\perp}^2\tbZ\tbZ^{\sT})^{-1}\tbZ\bw\>\Big\}\
\end{align}
It is immediate to see that (unless $\alpha_{\perp}=0$ or $\bv=0$), $s_n^2> \lambda_{\max}(\alpha_{\perp}^2\tbZ\tbZ^{\sT})$
almost surely, and therefore $s_n$ is given by the largest solution of  the equation
\begin{align}
s_n^2 = \|\bw\|_2^2+\alpha_{\perp}^2\<\bw,\tbZ^{\sT}(s_n^2\id-\alpha_{\perp}^2\tbZ\tbZ^{\sT})^{-1}\tbZ\bw\>\, .\label{eq:FiniteN}
\end{align}
Note that, almost surely, $\lim_{n\to\infty}\|\bw\|_2^2=\gamma^2+\alpha_{\|}^2 \equiv \tgamma^2$.
Further, $\bw$ is independent of $\tbZ$. 
Hence, by a standard random matrix theory argument \cite{Guionnet,BaiSilverstein}, for any $s^2>\alpha_{\perp}^2(1+\sqrt{\delta})^2$,
the following limits hold almost surely
\begin{align}
\lim_{n\to\infty}\frac{\alpha_{\perp}^2}{\|\bw\|_2^2}\<\bw,\tbZ^{\sT}(s^2\id-\alpha_{\perp}^2\tbZ\tbZ^{\sT})^{-1}\tbZ\bw\> &=
\lim_{n\to\infty}\frac{1}{d}\Tr\Big[\tbZ^{\sT}\big((s^2/\alpha_{\perp}^2)\id-\tbZ\tbZ^{\sT}\big)^{-1}\tbZ\Big] \\
& = -\delta - \frac{s^2\delta}{\alpha_{\perp}^2}\lim_{n\to\infty}\frac{1}{n}\Tr\Big[\big(\tbZ\tbZ^{\sT}-(s^2/\alpha_{\perp}^2)\id\big)^{-1}\Big] \\
& = -\delta -\frac{s^2\delta}{\alpha_{\perp}^2}  R\Big(\frac{s^2}{\alpha_{\perp}^2}\Big)\, ,
\end{align}
where $R(t)$ is the Stieltjes transform of the limit eigenvalues distribution of a Wishart matrix, which is given by the
Marcenko-Pastur law \cite{BaiSilverstein}
\begin{align}
R(z) & = \frac{-z-\delta+1+\sqrt{(z+\delta-1)-4\delta z}}{2\delta z}\, . 
\end{align}
Recall that $z\mapsto R(z)$ is increasing on $[z_v,\infty)$, $z_c\equiv (1+\sqrt{\delta})^2$, with $R(z_c+u) = R(z_c)-c\sqrt{u}+O(u)$
(for a constant $c>0$)
as $u\downarrow 0$, and $R(z) = -1/z+O(1/z^2)$ as $z\to \infty$. We therefore can consider the following 
asymptotic version of Eq.~(\ref{eq:FiniteN}):
\begin{align}
&\frac{s^2}{\tgamma^2}  =\hR\left(\frac{s^2}{\alpha_{\perp}^2}\right)\, ,\;\;\;\;\;\;\;\;\;\;
\hR(z)  = 1-\delta -\delta z\, R(z)\, .\label{eq:FixedPointEigenvalue}
\end{align}
Note that $\hR(z)$ is monotone decreasing on $[z_c,\infty)$ with $\hR(z_c) = (1+\sqrt{\delta})$, $\hR(z_c+u) =\hR(z_c)-c\sqrt{u}+O(u)$,
and $\hR(z) = 1+O(1/z)$ as $z\to\infty$.
For $\tgamma^2>(1+\sqrt{\delta})\alpha_{\perp}^2$, this equation has a unique solution $s_*^2$ with 
$s^2/\tgamma^2 <\hR(s^2\alpha_{\perp}^2)$ for $s^2\in [\alpha_{\perp}^2(1+\sqrt{\delta})^2,s_*^2)$ and 
$s^2/\tgamma^2 >\hR(s^2\alpha_{\perp}^2)$ for $s^2>s_*^2$. Hence, the largest solution $s_n^2$ of (\ref{eq:FiniteN}) converges 
almost surely to  $s_*^2$ as $n\to\infty$. 

For $\tgamma^2>(1+\sqrt{\delta})\alpha_{\perp}^2$, we have  $s^2/\tgamma^2 >\hR(s^2\alpha_{\perp}^2)$ for all
$s^2>\alpha_{\perp}^2(1+\sqrt{\delta})^2$ and therefore $\lim\sup_{n\to\infty} s^2_n\le \alpha_{\perp}^2(1+\sqrt{\delta})^2$
almost surely. Since we have a matching lower bound, we conclude that $\lim_{n\to\infty} s^2_n\le \alpha_{\perp}^2(1+\sqrt{\delta})^2$
in this case.

Finally, the expression  (\ref{eq:LimitSingValue})  follows by solving rewriting Eq.~(\ref{eq:FixedPointEigenvalue}) 
as $\hR^{-1}(s^2/\tgamma^2)= s^2/\alpha_{\perp}^2$, whereby the inverse of $\hR$ in $(1,1+\sqrt{\delta})$ is given by
\begin{align}
\hR^{-1}(x) = \frac{x(x+\delta-1)}{x-1}\, .
\end{align}
\end{proof}

We next state a general lemma that can be used to check whether a matrix of the form (\ref{eq:HessianFinalFormula})
is positive semidefinite.
\begin{lemma}\label{lemma:GeneralBlock}
Let $\bZ\in\reals^{n\times d}$ be random matrices with $(Z_{ij})_{i\le n, j\le d}\sim\normal(0,1/d)$, and $\bu\in\reals^n$, $\bv\in\reals^d$ be unit vectors,
with $n/d\to\delta$ as $n\to\infty$.
Define the projectors $\bPu = \bu\bu^{\sT}$ and $\bPup =\id-\bu\bu^{\sT}$.
For $a,b r,s,\beta,\xi\in\reals$ with $\beta\ge 0$ and $r> s$, let
\begin{align}
\obX & = \xi \,\bu\bv^{\sT} + \bZ\, ,\label{eq:RankOneLemma}\\
\oHess & =\left[\begin{matrix}
a\,\id_d & -\sqrt{\beta}\, \obX^{\sT}(\id_n-b\bPu)\\
-\sqrt{\beta}\, (\id_n-b\bPu) \obX & (r\id_n-s\bPu)
\end{matrix}\right]\, .\label{eq:HGeneral}
\end{align}
Assume that one of the following two conditions holds:
\begin{enumerate}
\item $(1-b)^2(1+\xi^2)/(r-s)\ge (1+\sqrt{\delta})/r$ and
\begin{align}
a (r-s) >\beta \frac{(1-b)^2(1+\xi^2)\big[(1-b)^2(1+\xi^2)r-(1-\delta)(r-s)\big]}{(1-b)^2(1+\xi^2)r-(r-s)}\, .
\label{eq:aCond_1}
\end{align}
\item $(1-b)^2(1+\xi^2)/(r-s)< (1+\sqrt{\delta})/r$ and
\begin{align}
a> \frac{\beta}{r}(1+\sqrt{\delta})^2\, . \label{eq:aCond_2}
\end{align}
\end{enumerate}
Then, there exists a constant $\eps>0$ such that, almost surely, $\Hess\succeq \eps\id$ for all $n$ large enough.
\end{lemma}
\begin{proof}
Let us first prove that, under the stated conditions, $\Hess\succeq\bzero$.
Since $r\id_n-s\bPu\succ \bzero$, we have $\Hess\succ \bzero$ if and only if
\begin{align}
a\id_d\succ \beta\obX^{\sT}(\id-b\bPu) (r-s\bPu)^{-1}(\id-b\bPu)\obX\, . \label{eq:ConditionSchur}
\end{align}
Notice that
\begin{align}
 (\id-b\bPu) (r-s\bPu)^{-1}(\id-b\bPu) = \frac{1}{r}\, \bPup+\frac{(1-b)^2}{r-s}\, \bPu\, .
\end{align}
Hence, condition (\ref{eq:ConditionSchur}) is equivalent to $a>\lambda_{\max}(\bM^{\sT}\bM) = s_{\max}(\bM)^2$,
where
\begin{align}
\bM = \sqrt{\beta}\left[\frac{1-b}{\sqrt{r-s}}\bPu+\frac{1}{\sqrt{r}}\bPup\right]\obX\, .
\end{align}
Note that $\bM$ is of the form of Lemma \ref{lemma:RMT_lemma}, with 
\begin{align}
\gamma = \sqrt{\frac{\beta\xi^2(1-b)^2}{r-s}}\, ,\;\;\;\;\;\; \alpha_{\|} = \sqrt{\frac{\beta(1-b)^2}{r-s}}\, ,
\;\;\;\;\;\; \alpha_{\perp} = \sqrt{\frac{\beta}{r}}\, .
\end{align}
The claim that $\Hess\succ\bzero$ then follows by using the asymptotic characterization of $s_{\max}(\bM)$ in 
Lemma \ref{lemma:RMT_lemma}.

We next  prove that in fact $\Hess\succeq \eps\id$. If the stated conditions hold, there exists $\eps$ small enough such that they hold  also after replacing 
$a$ with $a'=a-\eps$ and $r$ with $r'=r-\eps$. Let us write $\Hess(a,r)$ for the matrix of Eq.~(\ref{eq:HGeneral}), where we emphasized the dependence on
the parameters $a,r$. We have $\Hess(a,r)= \Hess(a',r')+\eps\id$, and hence the thesis follows since $\Hess(a',b')\succeq \bzero$.
\end{proof}

In order to apply the last lemma, we will show that, for $\beta<\beta_{\sp}$, the LDA model of Eq.~(\ref{eq:LDAModel}) is equivalent
for our purposes to a simpler model.
\begin{lemma}\label{lemma:Contiguity}
Let $\bX\in\reals^{n\times d}$ be distributed according to the LDA model (\ref{eq:LDAModel}) and
let  $\bR_1\in \reals^{n\times n}$, $\bR_2\in \reals^{d\times d}$ be uniformly random (Haar distributed) orthogonal matrices conditional to $\bR_1\bfone = \bfone$,
with $\{\bX,\bR_1,\bR_2\}$ mutually independent. Denote by $\prob_{1,n}$ the law of $\bX_{R}\equiv\bR_1\bX\bR_2$.

Define $\obX = \xi \,\bu\bv^{\sT} + \bZ$ as per Eq.~(\ref{eq:RankOneLemma}), with $\bu =\bfone_n/\sqrt{n}$, $\bv$ be a vector with i.i.d. entries 
$v_i\sim\normal(0,1/d)$, independent of $\bZ$, and $\xi = \sqrt{\beta\delta/k}$,
and denote by $\prob_{0,n}$ the law of $\obX$.

If $\beta<\beta_{\sp}(k,\nu,\delta)$, then $\prob_{1,n}$ is contiguous to $\prob_{0,n}$.
\end{lemma}
\begin{proof}
Recalling that $\bP = \bfone_k\bfone_k^{\sT}/k$, $\bPp=\id_k\bP$, and letting $\bv_0 = \bH\bfone_k/\sqrt{d k}$, we have
\begin{align} 
\bX =\xi\, \bu\bv_0^{\sT} + \frac{\sqrt{\beta}}{d}\bW_{\perp}\bH^{\sT}_\perp+\bZ \equiv \xi\, \bu\bv_0^{\sT} +\tbZ\, ,
\end{align}
where $\bW_{\perp} =\bW\bPp$ and $\bH_{\perp} =\bH\bPp$. Since $\bv_0$ is distributes as $\bv$, and independent of
$\tbZ$, it is sufficient to prove that the law of $\tbZ_R= \bR_1\tbZ\bR_2$ is contiguous to the law of $\bZ$. 

Note that by the law of large numbers, almost surely (see Eq.~(\ref{eq:S_Formula}))
\begin{align}
\lim_{n\to\infty}\frac{1}{n}\|\bW_{\perp}\|_{\op}^2 & = \lim_{n\to\infty}\frac{1}{n}\|\bW_{\perp}^{\sT}\bW_{\perp}\|_{\op} =  
\left\|\int (\bPp\bw)^{\otimes 2}\tq_0(\de\bw)\right\|_{\op} = \frac{1}{k(k\nu+1)}\, ,\\
\lim_{d\to\infty}\frac{1}{d}\|\bH_{\perp}\|_{\op}^2 & = \lim_{d\to\infty}\frac{1}{d}\|\bH_{\perp}^{\sT}\bH_{\perp}\|_{\op} =1\, .
\end{align}
Hence 
\begin{align}
\lim\sup_{n\to\infty}\left\|\frac{\sqrt{\beta}}{d}\bW_{\perp}\bH^{\sT}_\perp\right\|_{\op}\le \sqrt{\frac{\beta\delta}{k(k\nu+1)}} \equiv \sqrt{\beta_{\perp}}\, .
\end{align}
For $\beta<\beta_{\sp}$, we have $\beta_{\perp}<\sqrt{\delta}$, and therefore the rank-$k$ perturbation in $\tbZ$ does not produce 
an outlier eigenvalue \cite{benaych2012singular}. 

In order to prove that the law of $\tbZ_R= \bR_1\tbZ\bR_2$ is contiguous to the law of $\bZ$, note that 
$\tbZ_R \stackrel{{\rm d}}{=} (\sqrt{\beta}/d)\bR_1\bW_{\perp}\bH_{\perp}\bR_2+\bZ$. Let $\bbQ_{1,n}$ be the law of $\bW_1 = \bR_1\bW_{\perp}$ 
and $\bbQ_{2,n}$ the law of $\bW_2=\tilde{\bR_1}\bW_{\perp}$, where $\tilde{\bR_1}$ is a uniformly random orthogonal matrix (not Haar distributed).
We claim that $\lim_{n\to\infty}\|\bbQ_{1,n}-\bbQ_2\|_{\sTV}=0$. Indeed both $\bbQ_1$ and $\bbQ_1$ are uniform
conditional on $\bW^{\sT}\bW/\sqrt{n} = \bQ$ and $\bW^{\sT}\bfone/\sqrt{n} = \bb$. 
However, the joint laws of  $(\bQ,\bb)$ converge in total variation to the same Gaussian limit by the local central limit theorem.

It is therefore sufficient to show that the law of $\tbZ_{RR}= \tilde{\bR_1}\tbZ\bR_2$ is contiguous to the law of $\bZ$.
This follows by second moment method and follows exactly  as in \cite{montanari2017limitation}.
\end{proof}

\begin{lemma}\label{lemma:EigContig}
Let $\obX$ as per Eq.~(\ref{eq:RankOneLemma}), with $\bu =\bfone_n/\sqrt{n}$, $\bv$ be a vector with i.i.d. entries 
$v_i\sim\normal(0,1/d)$, independent of $\bZ$, and $\xi = \sqrt{\beta\delta/k}$, and define 
\begin{align}
\oHess =  \left[\begin{matrix}
\Big(1+\frac{\delta\beta}{k(k\nu+1)}\Big)\id_d & -\sqrt{\beta}\obX^{\sT}\Big(\id_n-\frac{\beta}{d(k+\delta \beta)}\allone_n\Big)\\
-\sqrt{\beta}\Big(\id_n-\frac{\beta}{d(k+\delta \beta)}\allone_n\Big)\obX & 
\big(\beta+k(k\nu+1)\big)\id_n-\frac{\beta^2}{d(k+\delta\beta)}\allone_n
\end{matrix}\right]\, .\label{eq:SimplifiedHessian}
\end{align}
If $\beta<\beta_{\sp}(k,\nu,\delta)$, then the law of the eigenvalues of the Hessian $\Hess$ defined in 
Eq.~(\ref{eq:HessianFinalFormula}) is contiguous to the law of the eigenvalues of $\oHess$.
\end{lemma}
\begin{proof}
Consider the random orthogonal matrix $\bR\in\reals^{(n+d)\times (n+d)}$
\begin{align}
\bR = \left[
\begin{matrix}
\bR_2^{\sT} & \bzero\\
\bzero& \bR_1
\end{matrix}
\right]
\end{align}
where $\bR_1\in \reals^{n\times n}$, $\bR_2\in \reals^{d\times d}$ be uniformly random (Haar distributed) orthogonal matrices conditional to $\bR_1\bfone = \bfone$.
Notice that the eigenvalues of $\Hess$ are the same as the ones of $\bR\Hess\bR^{\sT}$.
Further, we have
\begin{align}
\bR\Hess\bR^{\sT} =  \left[\begin{matrix}
\Big(1+\frac{\delta\beta}{k(k\nu+1)}\Big)\id_d & -\sqrt{\beta}\bX_R^{\sT}\Big(\id_n-\frac{\beta}{d(k+\delta \beta)}\allone_n\Big)\\
-\sqrt{\beta}\Big(\id_n-\frac{\beta}{d(k+\delta \beta)}\allone_n\Big)\bX_R & 
\big(\beta+k(k\nu+1)\big)\id_n-\frac{\beta^2}{d(k+\delta\beta)}\allone_n
\end{matrix}\right] \, ,
\end{align}
where $\bX_R = \bR_{1}\bX\bR_2$ is defined as in the statement of Lemma \ref{lemma:Contiguity}. Applying that
lemma, we obtain that the law of $\bR\Hess\bR^{\sT}$ is contiguous to the one of $\oHess$, and therefore
we obtain the desired contiguity for the laws of eigenvalues.
\end{proof}

The next lemma establishes that the simplified Hessian $\oHess$ is positive semidefinite.
\begin{lemma}
Let $\oHess$ be defined as per Eq.~(\ref{eq:SimplifiedHessian}) where $\obX = \xi \,\bu\bv^{\sT} + \bZ$
with $\bu =\bfone_n/\sqrt{n}$, $\bv$ be a vector with i.i.d. entries 
$v_i\sim\normal(0,1/d)$, independent of $(Z_{ij})_{i\le n,j\le d}\sim_{i.i.d.}\normal(0,1/d)$, and $\xi = \sqrt{\beta\delta/k}$.

If $\beta<\beta_{\sp}(k,\delta,\nu)$, then there exists $\eps>0$ such that, almost surely, $\oHess\succeq \eps\,\id$ 
for all $n$ large enough.
\end{lemma}
\begin{proof}
The matrix $\obX$ fits the setting of Lemma \ref{lemma:GeneralBlock} with 
\begin{align}
a &= 1+\frac{\delta \beta}{k(k\nu+1)}\, ,\;\;\;\;\;\;\;\;\;\; b = \frac{\beta\delta}{k+\delta\beta}\, ,\\
r &= \beta+k(k\nu+1)\, ,\;\;\;\;\;\;\;\;\;\; s = \frac{\beta^2\delta}{k+\delta\beta}\, .
\end{align}
The claim follows by checking that condition 2 in Lemma \ref{lemma:GeneralBlock} holds. Indeed we
have
\begin{align}
A \equiv \frac{(1-b)^2(1+\xi^2)}{r-s} = \frac{1}{\beta+(k\nu+1)(k+\beta\delta)}\, .
\end{align}
Hence $A<(1+\sqrt{\delta}/r)$. Further, setting $q=k(k\nu+1)$, we have 
\begin{align}
a-\frac{\beta}{r}(1+\sqrt{\delta})^2 &= 1+\frac{\delta\beta}{q} - \frac{\beta(1+\sqrt{\delta})^2}{\beta+q}\\
&=\frac{1}{\beta+q}\Big(\frac{\delta\beta^2}{q}-2\sqrt{\delta} \beta+q\Big)\\
& = \frac{\delta}{q(\beta+q)}\Big(\beta-\frac{q}{\sqrt{\delta}}\Big)>0\, .
\end{align}
(The last inequality follows since $\beta_{\sp} = q/\sqrt{\delta}$.) This completes the proof.
\end{proof}

The proof of Theorem \ref{thm:StabilityTAP} follows immediately from the above lemmas. Since the law of the
eigenvalues of $\Hess$ is contiguous to the law of the eigenvalues of $\oHess$ (by Lemma \ref{lemma:EigContig}),
and $\oHess\succeq \eps\id$ with high probability, we have 
\begin{align}
\lim_{n\to\infty}\prob(\lambda_{\min}(\Hess)<\eps/2) = 0\, .
\end{align}

\section{TAP free energy: Numerical results}

\subsection{Damped AMP}
\label{app:Damped}

AMP turns out to converge poorly near the spectral threshold, i.e. for $\beta\approx \beta_{\sp}$.
Note that this appears to be an algorithmic problem, rather than a problem related to the free energy approximation.
 To alleviate this issue, we used damped AMP for our numerical simulations. Damped AMP iterations are as follows
\begin{eqnarray}
\bm^{t+1} &=&  (1-\gamma) \bm^{t} +\gamma \bX^{\sT}\,\tsF(\tbm^t;\tbQ^t) -\gamma^ 2 \sF(\bm^t;\bQ^t) \bK^t_W\,, \\
\tbm^t &=&  (1 - \gamma) \tbm^{t-1} +\gamma \bX\,\sF(\bm^t;\bQ^t) - \gamma^2 \tsF(\tbm^{t-1};\tbQ_{t-1}) \bK_H^t\,, \\
\bQ^{t+1} &=& \frac{1}{d}\sum_{a=1}^n \tsF(\tbm_a^t;\tbQ^t)^{\otimes  2}\, ,\\
\tbQ^t &=& \frac{1}{d} \sum_{i=1}^d \sF(\bm_i^t;\bQ^t)^{\otimes  2}\, .
\end{eqnarray}
The matrices $\bK_H^t$ and $\bK_W^t$ are smoothed sum of Jacobian matrices and are computed as
\begin{eqnarray}
\bK_H^{t+1}&=& \sum_{i = 1}^{t+1} (1 - \gamma)^{t - i + 1}\sB_t\,,  \\
\bK_W^t &=& \sum_{i = 1}^t (1-\gamma)^{t - i}\sC_t
\end{eqnarray}
where 
\begin{eqnarray}
(\sB_t)_{rs} &=& \frac{1}{d}\sum_{i=1}^d\frac{\partial\sF_s}{\partial (\bm^t_i)_{r}}(\bm_i^t;\bQ^t)\,, \\
(\sC_t)_{rs} &=& \frac{1}{d}\sum_{a=1}^n\frac{\partial\tsF_s}{\partial (\tbm^t_i)_{r}}(\tbm_a^t;\tbQ^t)\,. 
\end{eqnarray}
In these calculations, $\gamma$ is the smoothing parameter that throughout our simulations is fixed to $\gamma = 0.8$.

The specific choice of this damping scheme (and --in particular-- the construction of matrices $\bK_H^{t+1}$, $\bK_W^{t+1}$)
is dictated by the fact that this specific choice admits a state evolution analysis, analogous to the one holding on the undamped case.

\section{Approximate Message Passing: Numerical results for $k=3$}
\label{app:AMPk3}

\begin{figure}[h!]
\phantom{A}
\vspace{-1cm}

\centering
\includegraphics[height=5.5in]{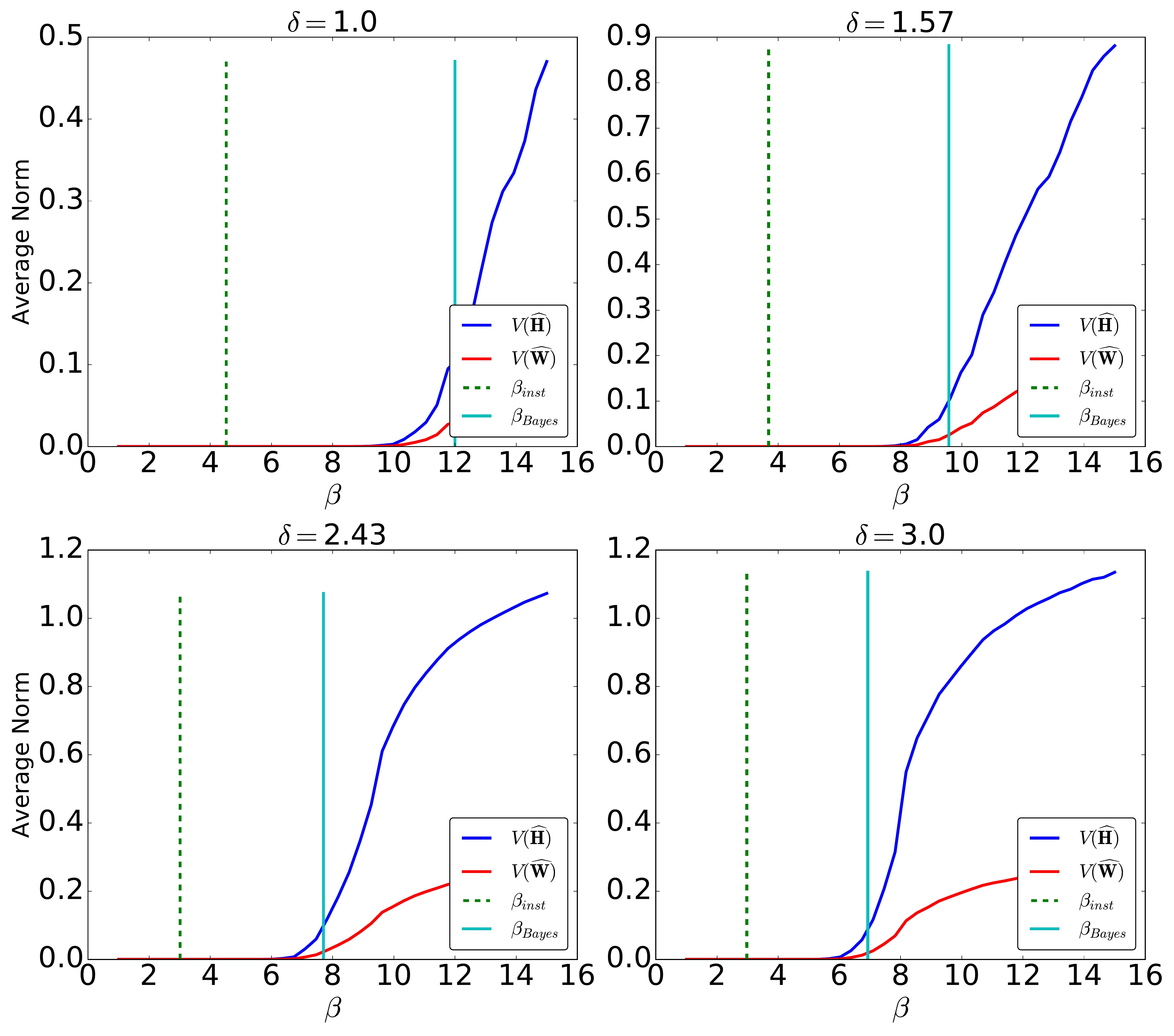}
\caption{Normalized distances $\Norm(\hbH)$, $\Norm(\hbW)$  of the AMP estimates from the uninformative fixed point.  Here
 $k=3$, $d = 1000$ and $n= d\delta$: each data point corresponds to an average over $400$ random realizations.} 
\label{fig:AMP_norm_k_3}
\end{figure}

\begin{figure}[h!]
\phantom{A}\hspace{-1.85cm}\includegraphics[height=2.66in]{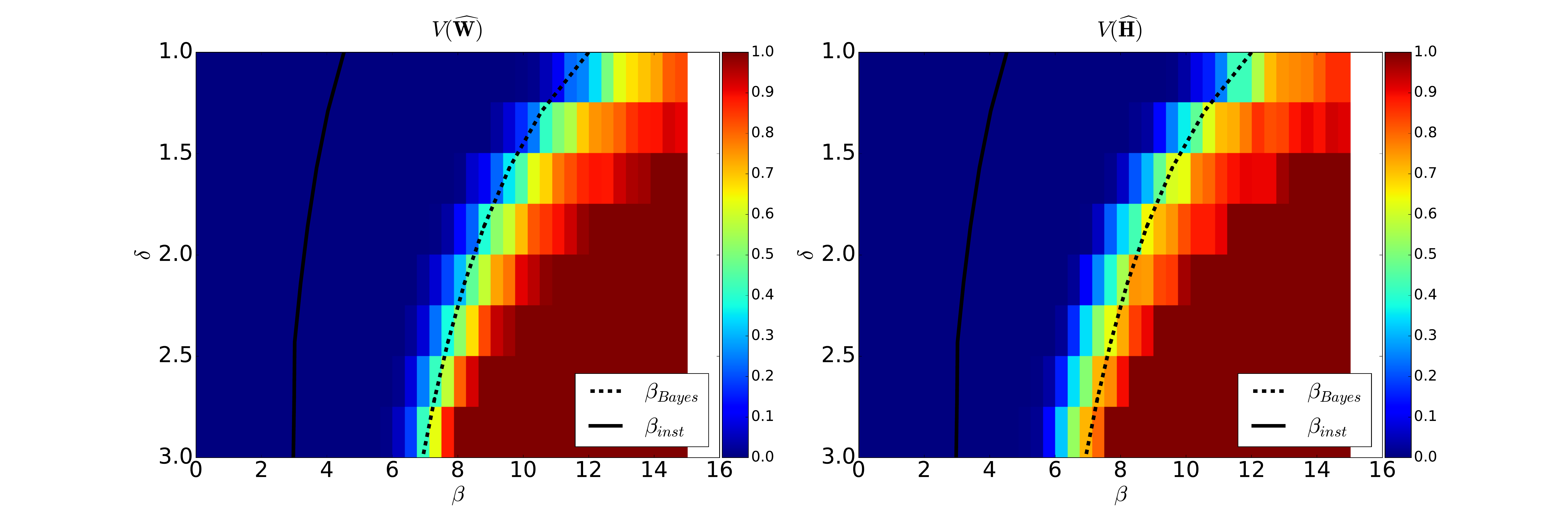}
\caption{Empirical fraction of instances such that  $\Norm(\hbW)\ge \eps_0=5\cdot 10^{-3}$ (left) or $\Norm(\hbH)\ge \eps_0$ (right), where $\hbW, \hbH$ are the AMP estimates. Here 
$k=3$, $d=1000$,    and for each $(\delta,\beta)$ point on a grid we ran AMP on  $400$ random realizations.}
\label{fig:AMP_norm_k_3_HM}
\end{figure}

\begin{figure}[h!]
\phantom{A}
\vspace{-1cm}

\centering
\includegraphics[height=5.5in]{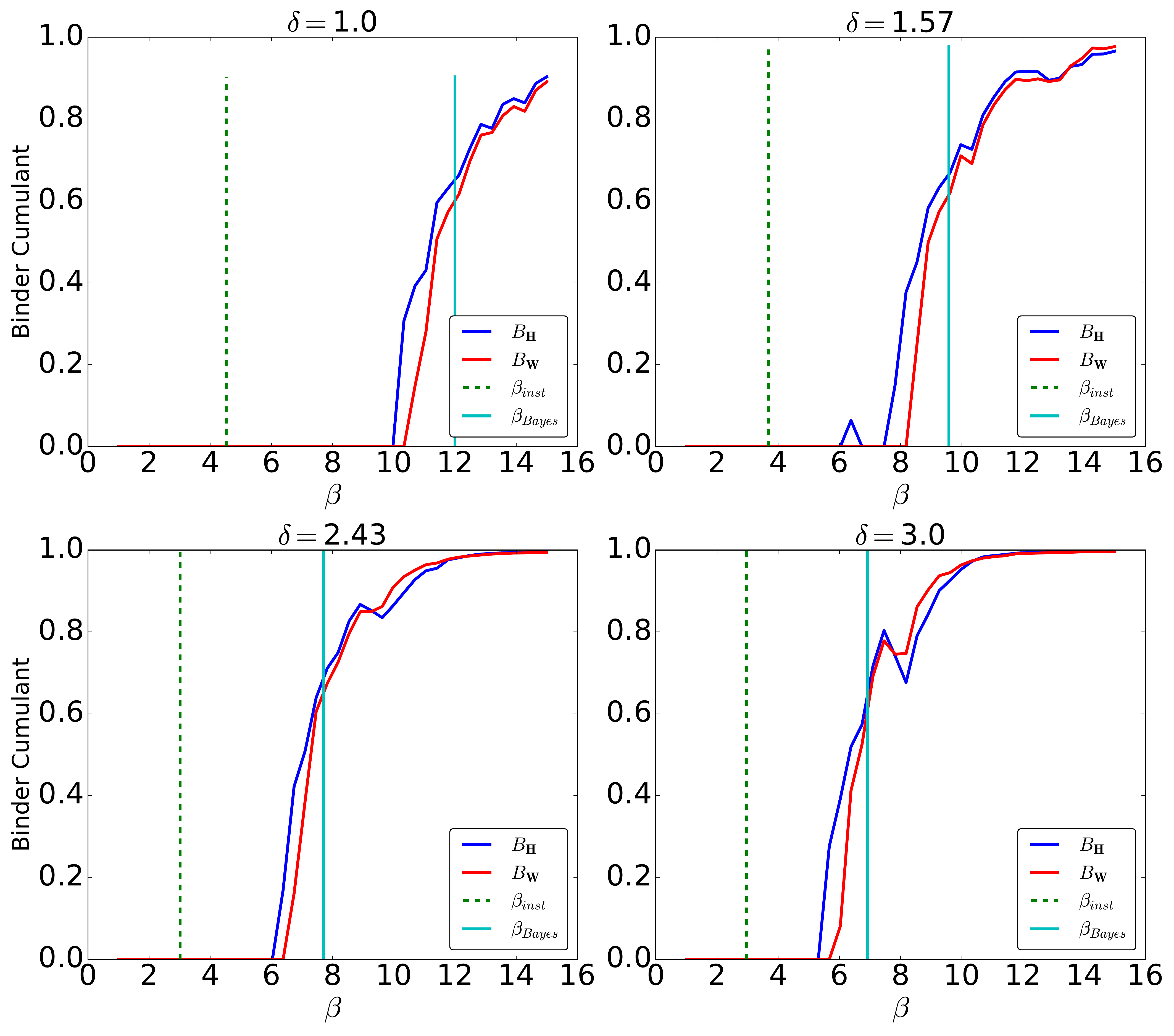}
\caption{Binder cumulant for the correlation between AMP estimates  $\hbW,\hbH$ and the true weights and topics $\bW, \bH$. Here $k=3$, $d=1000$,  $n=d\delta$ and estimates are obtained by averaging over $400$ realizations.} 
\label{fig:AMP_corrs_k_3}
\end{figure}

\begin{figure}[h!]
\phantom{A}\hspace{-1.85cm}\includegraphics[height=2.66in]{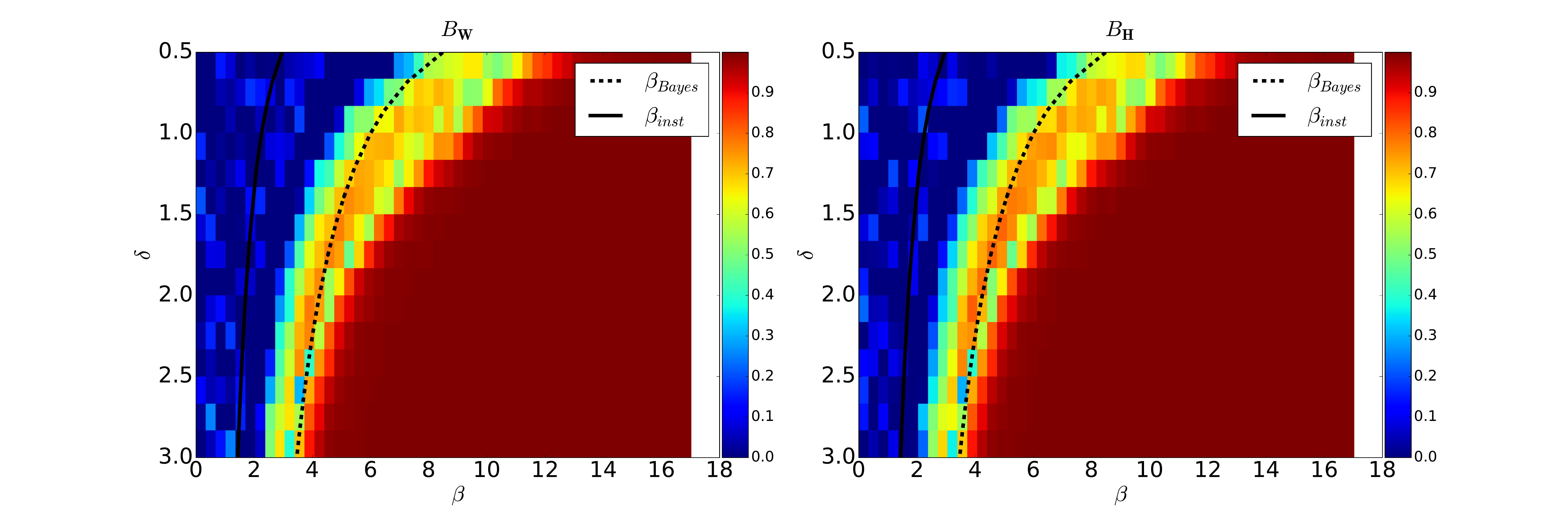}
\caption{Binder cumulant for the correlation between AMP estimates  $\hbW$, $\hbH$ and the true weights and topics $\bW, \bH$.
Here $k=3$, $d=1000$ and estimates are obtained by averaging over $400$ realizations.}
\label{fig:AMP_corrs_k_3_HM}
\end{figure}

In Figures \ref{fig:AMP_norm_k_3} to \ref{fig:AMP_corrs_k_3_HM} we report our numerical results using damped AMP 
for the case of $k=3$ topics. These simulations are analogous to the one presented in the main text for $k=2$, 
cf. Section \ref{sec:TAP_numerical}.

Figures \ref{fig:AMP_norm_k_3} and \ref{fig:AMP_norm_k_3_HM} report results on the normalized distance from the uninformative
subspace $\Norm(\hbH)$, $\Norm(\hbW)$. These are consistent with the claim that AMP converges to a fixed point that is
significantly distant from this subspace only if $\beta >\beta_{\sBayes}(k,\nu,\delta)=\beta_{\sp}(k,\nu,\delta)$.
In Figures \ref{fig:AMP_corrs_k_3} and \ref{fig:AMP_corrs_k_3_HM} we present our results on the correlation between the AMP estimates
$\hbH$, $\hbW$ and the true factors $\bH$, $\bW$. We measure this correlation through the same Binder parameter 
introduced in Section \ref{sec:NMF_k3}.

\section{Uniqueness of the solution to \eqref{eq:qs_1_main}}
\label{app:Uniqueness}

In this appendix, we  prove that the solution to \eqref{eq:qs_1_main} is unique under the following conjecture 
\begin{conjecture}
\label{conj:normsquared}

Let $q>0$ and $\bw\in \reals^k$ be a random variable with density $p(\bw) \propto \exp\left\{-q\left\|\bw\right\|_2^2\right\}\tq_0(\bw)$. Then
\begin{align}
\sigma(q)\gamma(q) \leq \frac{2}{q}
\end{align}
where $\sigma(q)$ and $\gamma(q)$ are the standard deviation and skewness of $\left\|\bw\right\|_2^2$.
\end{conjecture}
\begin{remark}
For a Gaussian random vector $\bz \sim \mathcal N(0, (2q)^{-1}\id_k)$ so that $p(\bz) \propto \exp\left\{-q\left\|\bz\right\|_2^2\right\}$,
\begin{align}
\tilde \sigma(q) \tilde\gamma(q) = \frac{2}{q}
\end{align}
where $\tilde \sigma(q), \tilde\gamma_1(q)$ are the standard deviation and the skewness of $\left\|\bz\right\|_2^2$.
\end{remark}
Using the above conjecture, it can be shown that the solution to \eqref{eq:qs_1_main} is unique.

Let $V(q)$ be the variance of $X = \|\bw\|_2^2$, when $\bw$ is distributed with density $p(\bw) \propto \exp\left\{-q\left\|\bw\right\|_2^2\right\}\tq_0(\bw)$. Define
\begin{align}
f(q) = \frac{k\beta\delta}{k-1}\, \left\{\sE\left(\frac{\beta}{1+q};\nu\right) - \frac{1}{k^2}\right\}\, .
\end{align}
Note that using the proof of Lemma \eqref{lemma:Uninf}, $f(q)$ is non-negative, continuous and monotone increasing
for $q>0$. Further,
\begin{align}
f^\prime(q) = \frac{\beta^2\delta}{(k-1)(1+q)^2}V\left(\frac{\beta}{1+q}\right).
\end{align}
Since $f(0) > 0$,
if we show that $f^\prime(q)$ is decreasing, then for $q>q^*$ where $q^*$ is the smallest solution to $f(q)=q$,
$f^\prime(q) < 1$. This will imply that $f(q) < q$ for $q > q^*$ that proves the uniqueness. We have
\begin{align}
f^{\prime\prime} (q) = \frac{\beta^2\delta}{(k-1)(1+q)^4} \left[-\frac{\beta}{(1+q)^2}V^\prime\left(\frac{\beta}{1+q}\right)(1+q)^2 - 2(1+q)V\left(\frac{\beta}{1+q}\right)\right]
\end{align}
Hence, $f^\prime(q)$ is decreasing if and only if
\begin{align}
- V^\prime\left(\frac{\beta}{1+q}\right) \leq 2\left(\frac{1+q}{\beta}\right)V\left(\frac{\beta}{1+q}\right).
\end{align}
Therefore, it is sufficient to show that for $q>0$, 
\begin{align}
\frac{-V^\prime(q)}{V(q)} \leq \frac{2}{q}.
\end{align}
Note that if we let $X=\|\bw\|_2^2$ where $\bw$ is as in Conjecture \ref{conj:normsquared}, we have
\begin{align}
V(q) = \E (X^2) - (\E X)^2.
\end{align}
Further,
\begin{align}
V^\prime(q) &= -\E X^3 + (\E X)(\E X^2) - 2(\E X)\left[- \E X^2 + (\E X)^2\right]\\
&= -\E X^3 + 3(\E X^2)(\E X) - 2 (\E X)^3.
\end{align}

Hence, 
\begin{align}
\frac{-V^\prime(q)}{V(q)} =  \frac{ \E (X^3) - 3(\E X^2)(\E X) + 2 (\E X)^3}{- \E X^2 + (\E X)^2} = \sigma(q)\gamma(q) \leq \frac{2}{q}
\end{align}
using Conjecture \ref{conj:normsquared}. Therefore, $f(q)$ is concave and \eqref{eq:qs_1_main} has a unique solution in $q \in (0, \infty)$. 

\end{document}